\documentclass[preprint,times,10pt]{elsarticle}
\usepackage[top=4.3cm,right=4.8cm,bottom=4.3cm,left=4.8cm]{geometry}
\usepackage{setspace}
\usepackage{adjustbox}
\usepackage{array}

\usepackage[utf8]{inputenc} 
\usepackage[T1]{fontenc}    
\usepackage{hyperref}       
\usepackage{url}            
\usepackage{booktabs}       
\usepackage{amsfonts}       
\usepackage{nicefrac}       
\usepackage{microtype}      
\usepackage{xcolor}         
\usepackage{amsmath}
\usepackage{pgfplots}
\pgfplotsset{compat=1.16}
\usepgfplotslibrary{groupplots}
\usetikzlibrary{positioning}
\usetikzlibrary{arrows.meta}
\usepackage{graphicx}
\usepackage{rotating}        
\usepackage{float}       
\usepackage{longtable}   
\usepackage{placeins}    
\usepackage{multirow}
\usepackage{subcaption}
\usepackage{tikz}
\usepackage{enumitem}\setlist{nosep,topsep=2pt,partopsep=0pt,parsep=0pt,itemsep=2pt}
\AtBeginDocument{\setlength{\abovedisplayskip}{4pt}\setlength{\belowdisplayskip}{4pt}%
  \setlength{\abovedisplayshortskip}{2pt}\setlength{\belowdisplayshortskip}{2pt}}
\newtheorem{theorem}{Theorem}
\newtheorem{lemma}{Lemma}
\newtheorem{corollary}{Corollary}
\newtheorem{proposition}{Proposition}
\newproof{pf}{Proof}

\begin{document}

\begin{frontmatter}

\title{Differential rank preservation manifold learning}

\author[bay]{\footnotesize Henry Han\corref{cor1}}
\ead{Henry_Han@baylor.edu}
\author[bay]{\footnotesize Luyu Niu}
\author[umes]{\footnotesize Diane Li}
\author[mon]{\footnotesize Jiacun Wang}
\cortext[cor1]{Corresponding author.}
\affiliation[bay]{organization={Data Science and AI Innovation Laboratory, Department of Computer Science, Baylor University},
                  city={Waco},
                  state={TX},
                  postcode={76798},
                  country={USA}}
\affiliation[umes]{organization={Department of Business, Management and Accounting, University of Maryland Eastern Shore},
                  city={Princess Anne},
                  state={MD},
                  postcode={21853},
                  country={USA}}
\affiliation[mon]{organization={Department of Computer Science and Software Engineering, Monmouth University},
                  city={West Long Branch},
                  state={NJ},
                  postcode={07764},
                  country={USA}}

\begin{abstract}
Manifold learning methods like t-SNE and UMAP preserve local neighborhoods,
but they do not preserve the ranking of the neighbors inside them. Two
embeddings can return the same neighbor set while disagreeing on which point
is closest. This loss of order degrades downstream applications that value the
order in the embedding. We propose Smooth Rank t-SNE (SR-tSNE),
which adds a differentiable rank term to the standard t-SNE loss. To enable efficient training, this term uses a sigmoid rank approximation to explicitly force a point's true high-dimensional neighbors to outrank unrelated \emph{intruders} and retain their original order. We also introduce Ordinal
DLP (O-DLP), an extension of the Degree of Locality Preservation (DLP) metric that
credits a neighbor only when it returns at its original rank. Across eight
benchmarks spanning image, text, speech, remote-sensing and single-cell data,
SR-tSNE outperforms 12 baselines in rank preservation and neighborhood
overlap. Ablation studies demonstrate our rank term alone boosts O-DLP by 5\% to 28\%, a gain not achievable using standard pairwise loss. This advantage transfers to downstream
applications. SR-tSNE achieves the best label-transfer performance across seven vision datasets. For cell-type discovery across six single-cell datasets, it ranks first in 17 of 18 dataset–metric combinations, whether clustering resolution is selected with or without cell-type labels. In intraday trading, SR-tSNE achieves higher returns, lower risk, and lower turnover than the baselines. To our knowledge, SR-tSNE is the first t-SNE-
family method with explicit order preservation.
\end{abstract}


\end{frontmatter}

\section{Introduction}
Manifold learning preserves which points are neighbors in the embedding, not the order
among them. Methods such as t-SNE~\cite{van2008visualizing},
UMAP~\cite{mcinnes2018umap}, PHATE~\cite{moon2019visualizing} and
TriMap~\cite{amid2019trimap} embed data from a high-dimensional input space while seeking to preserve neighborhood relationships. However, what they preserve is the membership of that
neighborhood; the order ranking inside it is left free. An embedding can recover
exactly the right neighbors and still shuffle them, putting the nearest neighbor
at rank~5 and the second-nearest at rank~1. This can make downstream analyses that rely on neighbor order in the embedding less effective.

\textit{The challenge.} Many downstream tasks rely on neighbor order as well as set membership. In vision-based retrieval and label transfer using the nearest neighbor, the top-ranked neighbor directly determines the predicted label; scrambling positions can introduce labeling errors. In single-cell clustering, community detection builds graphs from top-$k$ lists; rank errors that change these lists' membership can alter graph connectivity and split homogeneous cell populations into artificial clusters. In algorithmic trading, marker selection algorithms compute rank-correlation metrics between input and embedding spaces to flag trade signals; ordinal shifts can alter entry timing. Preserving both neighborhood membership and neighbor order allows the embedding to more faithfully represent the local structure of the input data.


The loss of neighbor order in manifold learning stems directly from how standard losses are formulated and how embedding performance is evaluated. Both current optimization objectives and downstream metrics are entirely order-agnostic. For example, t-SNE minimizes the  KL divergence based purely on pairwise similarities, never tracking where a neighbor lands. Consequently, two embeddings can achieve nearly identical KL loss while ordering their neighbors completely differently. Similarly, the Degree of Locality Preservation (DLP)~\cite{HAN2022877}, a manifold explainability metric, treats each neighborhood as an unordered set, assigning identical credit whether a true neighbor sits at rank 1 or rank $k$.

Furthermore, many manifold learning methods focus exclusively on neighborhood membership, scalability, or global structure~\cite{mcinnes2018umap,moon2019visualizing,10.1162/089976603321780317,wang2025localmap,kury2026dreams}, deliberately sacrificing ordinal fidelity. The crowding problem further bounds how much neighborhood structure two dimensions can hold~\cite{OLSZEWSKI2025130619}, and order is precisely what these methods give up. While methods like UMAP, PHATE, and LocalMAP improve local preservation or scalability, none require a retained neighbor to maintain its exact rank in the embedding. Practitioners are thus forced to interpret visualizations for structural relations that the underlying embedding principles never guarantee~\cite{jeon2026stop}. Although NeRV~\cite{venna2006,venna2010information} incorporates retrieval-oriented trade-offs into manifold learning, its precision--recall criteria still evaluate set membership rather than intra-neighborhood order.

\textit{SR-tSNE.} To address this challenge, we introduce Smooth Rank t-SNE (SR-tSNE), a method designed to preserve both neighborhood membership and exact neighbor order. SR-tSNE augments the standard t-SNE loss with a single differentiable rank term. For each data point, this term optimizes two criteria at once. It encourages true input-space neighbors to rank ahead of embedding-space \emph{intruders}, the points outside the corresponding input-space neighborhood. It also penalizes any reshuffling among the true neighbors themselves. Using a sigmoid formulation, discrete ranks are approximated by differentiable values, allowing direct optimization via standard gradient descent while a logarithmic discount prioritizes the nearest ranks. A single regularization weight, $\lambda$, governs the balance between neighborhood preservation and rank preservation.

\textit{O-DLP.} To evaluate how well SR-tSNE preserves neighbor order, we propose Ordinal Degree of Locality Preservation (O-DLP). While standard DLP checks only whether input-space neighbors remain neighbors in the embedding, regardless of their order, O-DLP credits a retained neighbor only when it appears at its exact original rank. This enables O-DLP to distinguish an embedding that faithfully preserves the original neighbor order from one that merely recovers the same neighbors in a shuffled order.

Empirical evaluations demonstrate that SR-tSNE more effectively preserves neighbor order and improves downstream performance. Benchmarked against 12 baseline methods across eight diverse datasets: spanning computer vision, text, speech, remote sensing, and single-cell transcriptomics, SR-tSNE consistently achieves the highest O-DLP and neighborhood overlap scores on every dataset. These gains also extend to neighborhood preservation across scales, as measured by the area under the $R_{NX}$ curve. SR-tSNE achieves the highest scores under both local and global evaluation regimes on MNIST and CIFAR-10. In downstream applications, SR-tSNE outperforms competing methods in image retrieval, label transfer, and single-cell cell-type discovery, while also delivering higher risk-adjusted returns in a rank-sensitive algorithmic trading pipeline.

\textbf{Our contributions.}
\textit{(1) Method:} SR-tSNE augments t-SNE with end-to-end differentiable rank regularization to penalize intruders and preserve exact neighbor order. \textit{(2) Theory:} We prove zero loss at perfect order, exponential soft-to-exact rank convergence, strict order preference where t-SNE/UMAP lack it, and an upper bound on order deficit. \textit{(3) Metric:} O-DLP extends DLP from neighborhood membership to exact-rank fidelity. \textit{(4) Benchmarks:} Across eight datasets, SR-tSNE outperforms 12 baselines in O-DLP and neighborhood overlap, with higher $R_{NX}$ AUC. \textit{(5) Utility:} We demonstrate gains in label transfer, single-cell cell-type discovery, and algorithmic trading; the regularizer extends to any differentiable neighborhood-embedding objective.
\section{Related Work}
Neighbor-preserving manifold learning methods keep local affinities
or neighborhood membership, not order.
t-SNE~\cite{van2008visualizing}, UMAP~\cite{mcinnes2018umap} and
PHATE~\cite{moon2019visualizing} optimize different similarity objectives.
Laplacian Eigenmaps~\cite{10.1162/089976603321780317},
LLE~\cite{doi:10.1126/science.290.5500.2323} and
PaCMAP~\cite{wang2021understanding} instead preserve local graph or distance
relationships.
NeRV~\cite{venna2010information} incorporates a retrieval-oriented trade-off
into a manifold learning objective and is the closest such objective to ours.
All of them preserve local relationships. None of them asks a neighbor to come back to
its own place in the list. SR-tSNE keeps t-SNE as it is and adds one term for that place.

Recent refinements modify neighborhood selection rather than internal ordering.
LocalMAP~\cite{wang2025localmap} dynamically adapts neighbor graphs,
FUnc-SNE~\cite{lambert2025funcsne} drops the fixed neighborhood size,
and preprocessing techniques denoise the manifold prior to embedding~\cite{CONTRERAS2026133199}.
While these methods adjust which candidate points are embedded, they leave the order within the neighborhood unconstrained.
DREAMS~\cite{kury2026dreams} regularizes t-SNE by trading local fidelity for global PCA alignment;
in contrast, SR-tSNE regularizes t-SNE specifically to enforce fine-grained local neighbor order.

Rank- and comparison-based methods are the closest conceptual relatives, but they address order incompletely. Triplet-based approaches such as TriMap~\cite{amid2019trimap} and t-STE~\cite{vandermaaten2012stochastic} sample isolated relative comparisons, leaving the exact rank order of an entire neighborhood free. Conversely, purely ordinal methods like RankVisu~\cite{lespinats2009rankvisu} and local ordinal embedding~\cite{terada2014local} rely solely on discrete rank relations, entirely discarding the continuous distance densities that preserve cluster geometry. SR-tSNE bridges this gap: it retains t-SNE's continuous similarity modeling while adding a smooth, listwise penalty that jointly aligns the full input-space neighbor ranking and expels embedding-space intruders.

Existing evaluation metrics assess rank preservation at varying granularities, but stop short of requiring exact positional fidelity. The co-ranking matrix~\cite{lee2009quality} provides the overarching framework, while summary metrics such as $Q_{NX}$ and $R_{NX}$~\cite{lee2009quality,lee2015multi} evaluate overall neighborhood shifts and serve as standard embedding benchmarks~\cite{debodt2025lowdim}. At the local level, DLP~\cite{HAN2022877} measures unordered neighborhood overlap, and nDCG provides a graded score that tolerates position shifts. O-DLP complements these metrics by enforcing the strictest criterion: a retained neighbor earns credit if and only if it recovers its exact original rank.
\section{Methodology}
\label{sec:method}
\textit{SR-tSNE Overview.} SR-tSNE augments the standard t-SNE objective with a differentiable regularizer that explicitly enforces neighbor order. Let $X \in \mathbb{R}^{n \times d}$ denote the input-space data and $Y \in \mathbb{R}^{n \times 2}$ the target low-dimensional embedding coordinates. Standard t-SNE~\cite{van2008visualizing} minimizes the Kullback--Leibler divergence $\mathrm{KL}(P \| Q)$. Here $P = [p_{ij}]$ holds symmetrized Gaussian affinities computed from input-space distances, with per-point bandwidths $\sigma_i$ calibrated to a target perplexity. $Q = [q_{ij}]$ holds Student-$t$ probabilities computed from embedding-space distances, with $q_{ij} \propto (1+\|y_i-y_j\|^2)^{-1}$. The overall objective of SR-tSNE is formulated as:
\begin{equation}\label{eq:total_loss}
    \mathcal{L}_{\text{total}}(Y) \;=\;
    \underbrace{\mathrm{KL}(P \,\|\, Q)}_{\text{neighbor similarity}}
    \;+\; \lambda \cdot
    \underbrace{\mathcal{L}_{\text{rank}}(Y)}_{\text{rank loss}},
\end{equation}
where $\mathcal{L}_{\text{rank}}$ is the smooth rank-regularization term detailed in Section~\ref{sec:smooth_rank_loss}, and $\lambda \ge 0$ balances neighborhood affinity against exact ordinal preservation. When $\lambda = 0$, the objective reduces to standard t-SNE. By optimizing $\mathcal{L}_{\text{total}}$, SR-tSNE encourages embeddings that preserve both local neighborhood membership and the fine-grained ordinal ranking of true neighbors.

\begin{figure}[!htpb]
\centering
\includegraphics[width=0.78\linewidth]{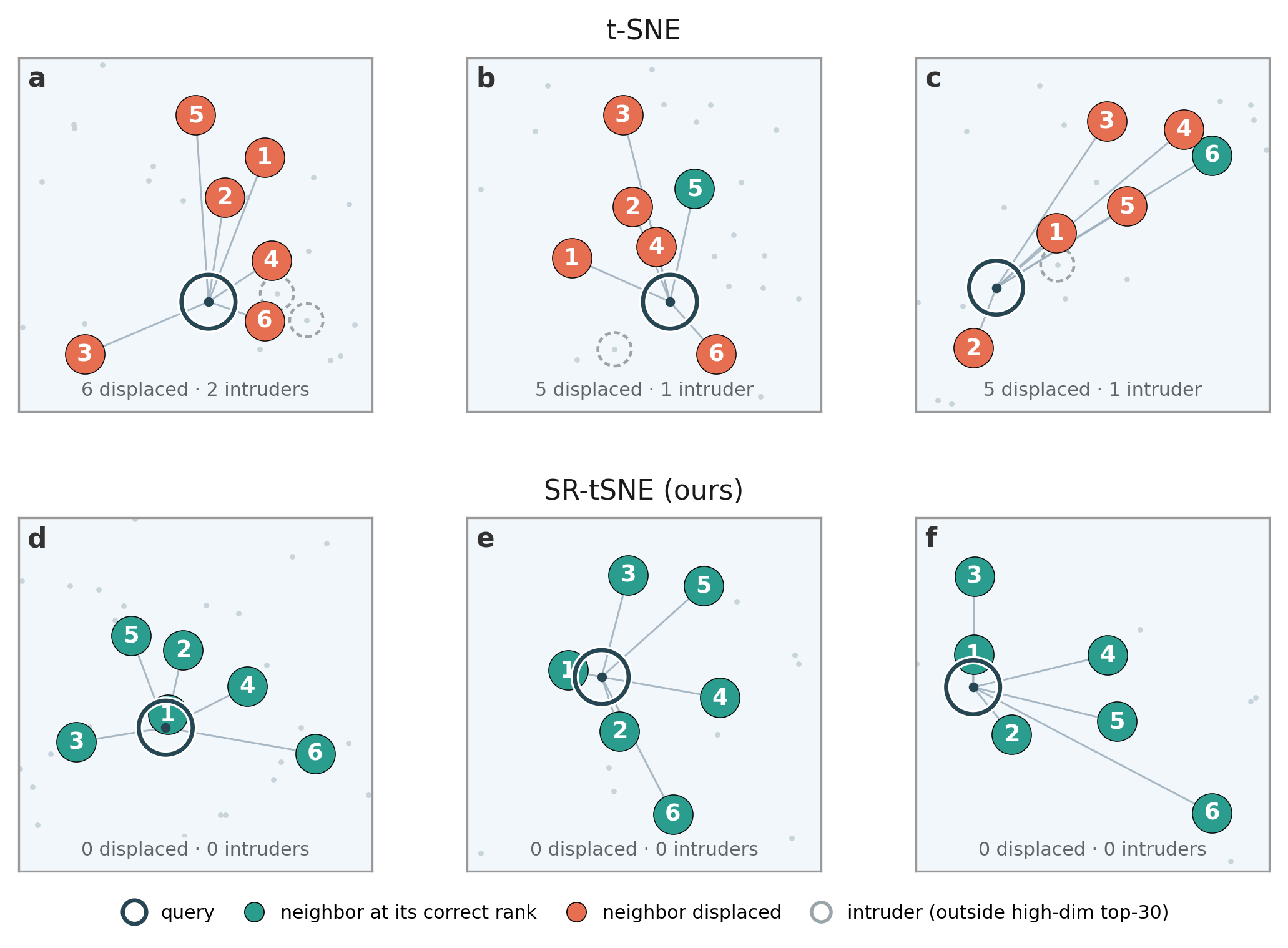}
\caption{\textit{Mechanism of the smooth rank regularization on three MNIST queries ($k=6$).}
Solid discs show true input-space neighbors positioned in the embedding space and labeled by original rank: \textit{teal} preserves exact rank, while \textit{coral} is displaced.
\textit{Dashed circles} mark embedding-space \emph{intruders} (from outside the top-$30$).
$\mathcal{L}_{\text{rank}}$ simultaneously repels intruders and restores neighbor order, improving exact-rank preservation (O-DLP) by $19\%$ over t-SNE on the full $70{,}000$-sample MNIST dataset (Table~\ref{tab:odlp_main}); the gain over the equal-budget control is $+17\%$ (Table~\ref{tab:control_main}).}
\label{fig:rankillus}
\end{figure}

\subsection{Smooth Rank Loss}
\label{sec:smooth_rank_loss}

Standard t-SNE minimizes KL divergence between pairwise similarity distributions, leaving the internal rank order of neighbors unconstrained. To resolve this, SR-tSNE introduces a \textit{differentiable rank loss} $\mathcal{L}_{\text{rank}}$ that enforces two explicit constraints.
First, true input-space neighbors must outrank any false intruders drifting into the embedding neighborhood.
Second, true neighbors must retain their exact original order, with nearest neighbors prioritized via a logarithmic discount (Figure~\ref{fig:rankillus}).

\textit{1. Candidate Pool Construction.}
For each point $i$, we define its candidate pool by combining two disjoint groups:
\begin{itemize}
    \item \textit{True Neighbors ($\mathcal{N}_i$):} The ordered sequence of $k$ nearest neighbors in the input space, $(\mathcal{N}_i(1), \ldots, \mathcal{N}_i(k))$. This list is computed once on $X$ and fixed. The true input-space rank of candidate $j \le k$ is simply $r_{i,j}^{\text{in}} = j$.
    \item \textit{Intruders ($\mathcal{I}_i$):} The $m$ closest embedding points to $y_i$ that do not belong to $\mathcal{N}_i \cup \{i\}$. These points represent embedding-space artifacts (impostors~\cite{weinberger2009distance}) and are updated every $T_r$ iterations.
\end{itemize}
Concatenating both sets forms a candidate pool $\mathcal{C}_i = \mathcal{N}_i \oplus \mathcal{I}_i$ of size $k + m$. For each candidate $j \in \{1, \ldots, k+m\}$, we compute its squared Euclidean distance to $y_i$ in the embedding space: $d_{i,j}^2 = \| y_i - y_{C_i(j)} \|_2^2$.

\textit{2. Differentiable Soft Rank.}
Sorting distances directly yields discrete, non-differentiable ranks with zero gradient almost everywhere. To enable backpropagation, we approximate the discrete indicator $\mathbb{1}[d_{i,l}^2 < d_{i,j}^2]$ using a continuous logistic sigmoid $\sigma(z) = (1 + e^{-z})^{-1}$ with temperature $\tau > 0$~\cite{qin2010general}. The soft embedding rank of candidate $j$ is computed by summing the smoothed pairwise comparisons:
\begin{equation}\label{eq:r_low}
    r_{i,j}^{\text{emb}} \;=\; 1 \,+\, \sum_{\substack{l=1 \\ l \neq j}}^{k+m} \sigma\!\left(\frac{d_{i,j}^2 - d_{i,l}^2}{\tau}\right).
\end{equation}
Intuitively, $r_{i,j}^{\text{emb}}$ represents $1$ plus the effective number of candidates placed closer to $y_i$ than candidate $j$. As $\tau \to 0^+$, $r_{i,j}^{\text{emb}}$ converges to the exact integer rank.

\textit{3. Loss via Discounted Gain.}
Only the $k$ true neighbors earn positive relevance; intruders carry zero gain and act solely as rank obstructions. We assign an input-space relevance score of $k - j + 1$ to the $j$-th true neighbor. The actual discounted gain $G_i$ in the embedding space is:
\begin{equation}\label{eq:gain}
    G_i \;=\; \sum_{j=1}^{k} \frac{k - j + 1}{\log_2\!\bigl(1 + r_{i,j}^{\text{emb}}\bigr)}.
\end{equation}
The maximum possible gain $G^{\star}$ occurs under perfect preservation, where every true neighbor recovers its exact original rank ($r_{i,j}^{\text{emb}} = j$) ahead of all intruders:
\begin{equation}\label{eq:ideal_gain}
    G^{\star} \;=\; \sum_{j=1}^{k} \frac{k - j + 1}{\log_2(1 + j)}.
\end{equation}
The total rank loss is the average normalized deficit across all $n$ data points:
\begin{equation}\label{eq:rank_gain}
    \mathcal{L}_{\text{rank}}(Y) \;=\; 1 \;-\; \frac{1}{n}\sum_{i=1}^{n} \frac{G_i}{G^{\star}}.
\end{equation}

Here, $\mathcal{L}_{\text{rank}}(Y) \in [0, 1]$. It reaches zero if and only if all $k$ true neighbors occupy ranks $1, \ldots, k$ in their exact input order without any intruder interference. When $m = 0$, this formulation smoothly reduces to a differentiable Spearman rank correlation (Appendix~\ref{app:m0}).
\subsection{Ordinal Degree of Locality Preservation (O-DLP)}
\label{sec:metrics}

We refer to $\mathcal{L}_{\text{rank}}(Y)$ in Eq.~(\ref{eq:rank_gain}) as the continuous (soft) rank loss $\mathcal{L}_{\text{rank}}^{\text{soft}}(Y)$ when evaluated with the soft ranks $r_{i,j}^{\text{emb}}$, and denote its per-point contribution as $\mathcal{L}_{\text{rank}, i}^{\text{soft}} = 1 - G_i/G^{\star}$.

Standard Degree of Locality Preservation (DLP)~\cite{HAN2022877} evaluates unordered neighborhood overlap. Specifically, let $A_{x_i,k}$ denote the set of $k$-nearest neighbors of $x_i$ in the input space, and let $B_{y_i,k}$ denote the corresponding $k$-nearest neighbors of $y_i$ in the embedding space. DLP measures the average set retention across all $n$ points:
\begin{equation}\label{eq:dlp}
  \text{DLP}(k) \;=\; \frac{1}{nk} \sum_{i=1}^{n} \bigl| A_{x_i,k} \cap B_{y_i,k} \bigr|.
\end{equation}
Because DLP treats neighborhoods strictly as unordered sets, a neighbor drifting from rank~$1$ to rank~$k$ incurs zero penalty as long as it remains within $B_{y_i,k}$.

\textit{O-DLP.} O-DLP extends DLP to evaluate exact positional alignment. Let $\pi_i^{\text{in}}(r)$ and $\pi_i^{\text{emb}}(r)$ denote the index of the $r$-th nearest neighbor of point $i$ in the input space and embedding space, respectively. O-DLP measures the proportion of exact rank matches across all points:
\begin{equation}\label{eq:odlp}
  \text{O-DLP}(k) \;=\; \frac{1}{nk} \sum_{i=1}^{n}\sum_{r=1}^{k} \mathbf{1}\!\left[\pi_i^{\text{in}}(r) = \pi_i^{\text{emb}}(r)\right],
\end{equation}
where $\mathbf{1}[\cdot]$ is the indicator function. A score of $1$ signifies that all $k$-neighborhoods preserve identical rank order across both spaces, whereas $0$ indicates that no neighbor reappears at its original position. Because rank-preserved neighbors necessarily belong to the intersection $A_{x_i,k} \cap B_{y_i,k}$, O-DLP forms a strict lower bound on DLP:
\begin{equation}\label{eq:odlp_le_dlp}
  \text{O-DLP}(k) \;\le\; \text{DLP}(k),
\end{equation}
with equality holding if and only if every retained neighbor maintains its exact rank.

\textit{Graded evaluation via nDCG.} While O-DLP establishes the strictest standard by requiring exact positional alignment, we also report Normalized Discounted Cumulative Gain (nDCG)~\cite{qin2010general}, the graded form of the quantity the rank term optimizes, which credits a neighbor by the position it reaches rather than only at its exact rank. Crucially, while soft metrics like nDCG can remain deceptively high under local rank inversions, O-DLP exposes whether individual positions are truly preserved. Evaluating both metrics ensures that SR-tSNE improves both coarse neighborhood placement (captured by nDCG) and fine-grained ordinal sequences (captured by O-DLP).
Downstream steps such as clustering rest on both parts of a neighborhood. They need a high degree of locality
preservation, and they need the order inside it. Eq.~(\ref{eq:odlp_le_dlp}) bounds one by the other, and
Theorem~\ref{thm:blind}(i) separates O-DLP into membership times order. SR-tSNE leads on both.

\subsection{SR-tSNE Optimization Details}
\label{sec:implementation}

SR-tSNE optimizes the total loss $\mathcal{L}_{\text{total}}(Y) = \mathrm{KL}(P \| Q) + \lambda \mathcal{L}_{\text{rank}}(Y)$ in two stages: a deterministic warm start followed by joint divergence and refinement.

\textit{1. Two-stage optimization.}
A standard Barnes--Hut t-SNE~\cite{10.5555/2627435.2697068} warm start (e.g., PCA initialization, perplexity $30$, learning rate $\max(n/12, 50)$, early exaggeration $12$) runs for $1{,}000$ iterations to establish a coarse global layout. Joint refinement then runs for another $T = 1{,}000$ iterations without early exaggeration. Throughout refinement, we set $k = 10$ neighbors and $m = 10$ intruders. Every $T_r = 50$ iterations, intruder sets $\mathcal{I}_i$ are refreshed by identifying the $m$ points nearest to $y_i$ in the embedding outside $\mathcal{N}_i \cup \{i\}$.

\textit{2. Temperature annealing.}
To ensure broad gradient exploration early and strict rank agreement late, the smoothing temperature $\tau$ decreases linearly over iterations $t \in [1, T]$:
\begin{equation}\label{eq:tau_anneal}
    \tau^{(t)} \;=\; \tau_{\text{start}} \,-\, \frac{t - 1}{T - 1}\bigl(\tau_{\text{start}} - \tau_{\text{end}}\bigr),
\end{equation}
falling from $\tau_{\text{start}} = 5.0$ to $\tau_{\text{end}} = 0.05$. Per Theorem~\ref{thm:soft}, $\tau_{\text{end}} = 0.05$ guarantees that the soft rank matches the discrete rank to
within $\mathcal{O}(e^{-\gamma_i / 0.05})$, where $\gamma_i > 0$ is the smallest gap between the squared embedding
distances of two distinct candidates of point $i$.

\textit{3. Gradient update and re-centering.}
Coordinates are updated using SGD with Polyak momentum ($\mu = 0.8$, $\eta = 10$):
\begin{align}
    V^{(t+1)} &\;=\; \mu \, V^{(t)} \,+\, \nabla_Y \mathcal{L}_{\text{total}}\bigl(Y^{(t)}\bigr), \label{eq:sgd_velocity} \\
    \widetilde{Y}^{(t+1)} &\;=\; Y^{(t)} \,-\, \eta \, V^{(t+1)}. \label{eq:sgd_coord}
\end{align}
Because neither term penalizes translation, we re-center coordinates after each step to remove translational drift:
\begin{equation}\label{eq:centering}
    Y^{(t+1)} \;=\; \widetilde{Y}^{(t+1)} \,-\, \frac{1}{n} \sum_{i=1}^{n} \tilde{y}_i^{(t+1)}.
\end{equation}

\textit{4. Memory complexity and design rationale.}
For large datasets ($n \ge 50{,}000$), we store the sparse matrix $P$ with $k_P = 90$ affinities per row. The partition function $Z = \sum_{a \ne b}(1+\|y_a-y_b\|^2)^{-1}$ is computed in row chunks of $c = 10{,}000$ using gradient checkpointing, lowering peak memory from $\mathcal{O}(n^2)$ to $\mathcal{O}(c\,n)$. Together, this two-stage schedule ensures that the warm start preserves coarse neighborhood membership, while the annealed rank refinement directly recovers fine-grained neighbor order and expels intruders.


\subsection{Theoretical Guarantees}
\label{sec:theory}

Standard objectives like t-SNE and UMAP preserve neighbor order only when high-dimensional affinities strictly differ; when affinities tie, they become indifferent to the ordering. SR-tSNE resolves this by explicitly penalizing rank inversions through its rank loss. This section formalizes this distinction through four guarantees. The rank loss
controls O-DLP and vanishes exactly at perfect order (Theorem~\ref{thm:control}), and the soft
rank used in optimization converges exponentially to the discrete rank
(Theorem~\ref{thm:soft}). Every inverted neighbor pair is strictly penalized, where t-SNE
and UMAP are indifferent (Theorem~\ref{thm:pref}), and the loss bounds the order deficit
among retained neighbors relative to an order-blind baseline (Theorem~\ref{thm:blind}).

\textbf{Preliminaries.} For each point $i$, the candidate pool $\mathcal{C}_i = \mathcal{N}_i \oplus \mathcal{I}_i$ of size $K = k + m$ combines its $k$ true input-space neighbors $\mathcal{N}_i$ with $m$ embedding intruders $\mathcal{I}_i$ (the $m$ non-neighbors closest to $y_i$). We sort their embedding distances $\|y_i - y_{C_i(j)}\|_2$ from nearest to farthest to obtain the discrete rank $\pi_i^{-1}(j) \in \{1, \dots, K\}$. Evaluating the discounted gain with these integer ranks rather than the soft values $r_{i,j}^{\text{emb}}$ yields the discrete gain $G_i^{\text{hard}} = \sum_{j=1}^{k} \frac{k - j + 1}{\log_2(1 + \pi_i^{-1}(j))}$ and the corresponding discrete rank loss:
\begin{equation}\label{eq:l_hard}
    \mathcal{L}^{\text{hard}}_{\text{rank}}(Y) \;=\; 1 \;-\; \frac{1}{n}\sum_{i=1}^{n} \frac{G_i^{\text{hard}}}{G^{\star}} \;=\; 1 \;-\; \frac{1}{n\,G^{\star}}\sum_{i=1}^{n}\sum_{j=1}^{k} \frac{k - j + 1}{\log_2\!\bigl(1 + \pi_i^{-1}(j)\bigr)},
\end{equation}
where $G^{\star} = \sum_{j=1}^k \frac{k - j + 1}{\log_2(1 + j)}$ is the ideal gain, and $\Delta_K = \mathrm{disc}(K-1) - \mathrm{disc}(K) > 0$ represents the smallest discrete step of the position discount $\mathrm{disc}(r) = 1/\log_2(1+r)$. All proofs are in \ref{app:proofs}.

\begin{theorem}[The rank loss controls O-DLP]
\label{thm:control}
Let $m \ge 1$ and let $k'$, the neighborhood size at which the metrics are read, satisfy $k' \le \min(k, m)$. Then
\begin{equation}
1 - \text{O-DLP}(k') \;\le\; \frac{2\,G^{\star}}{k'\,\Delta_K}\;
\mathcal{L}^{\text{hard}}_{\text{rank}}(Y).
\label{eq:control}
\end{equation}
Moreover $\mathcal{L}^{\text{hard}}_{\text{rank}}(Y) = 0$ if and only if
$\text{O-DLP}(k) = 1$, in which case $\text{O-DLP}(k') = 1$ for every
$k' \le k$.
\end{theorem}

Theorem 1 establishes two key results. First, it provides a quantitative surrogate error bound: when convergence is incomplete ($\mathcal{L}_{\text{rank}}^{\text{hard}} > 0$), the discrete ordinal error $1 - \text{O-DLP}(k')$ is strictly upper-bounded by the loss. The prefactor $\frac{2G^\star}{k'\Delta_K}$ represents the worst-case penalty where every inversion occurs at the smallest discount step $\Delta_K$ ($1.6 \times 10^3$ when $k = m = k' = 10$). Crucially, this equivalence requires intruders ($m \ge 1$). Without them ($m = 0$), false points could insert themselves between true neighbors without increasing the loss, since the pairwise order among true neighbors remains unchanged. Assigning target ranks behind all true neighbors to these $m$ intruders guarantees that any intervening intruder is penalized, thereby locking both neighborhood membership and order.

\begin{theorem}[Continuous-to-discrete convergence]
\label{thm:soft}
Let $\gamma_i = \min_{a \ne b \in \mathcal{C}_i} \bigl| \|y_i - y_a\|_2^2 - \|y_i - y_b\|_2^2 \bigr| > 0$ be the minimum squared distance gap between distinct candidates of point $i$. For any candidate $j$, as the sigmoid smoothing temperature $\tau \to 0^+$, the continuous soft rank $r_{i,j}^{\text{emb}}$ converges exponentially to the true discrete rank $\pi_i^{-1}(j)$:
\begin{equation}
\bigl| r_{i,j}^{\text{emb}} - \pi_i^{-1}(j) \bigr| \;\le\; (K - 1) \, e^{-\gamma_i / \tau},
\label{eq:soft_bound}
\end{equation}
and the per-point difference between the continuous soft loss $\mathcal{L}_{\text{rank}, i}^{\text{soft}}$ and the discrete hard loss $\mathcal{L}_{\text{rank}, i}^{\text{hard}}$ satisfies:
\begin{equation}
\bigl| \mathcal{L}_{\text{rank}, i}^{\text{soft}} - \mathcal{L}^{\text{hard}}_{\text{rank}, i} \bigr| \;\le\; \frac{k(k+1)(K-1)}{4 \ln 2 \, G^{\star}} \, e^{-\gamma_i / \tau}.
\label{eq:loss_soft_bound}
\end{equation}
\end{theorem}

The gap closes exponentially in the margin over the temperature. The schedule
of Section~\ref{sec:implementation} ends at $\tau = 0.05$; a margin of $0.5$
in squared distance then leaves a factor $e^{-10} \approx 4.5 \times 10^{-5}$.
This is why the loss the optimizer minimizes and the metric the tables report
agree at the end of training.

The third result isolates what each objective intrinsically demands of neighbor order. By treating each loss purely as a function of squared distances $D = (d_{a,b}^2)$, we examine an idealized row swap $D \mapsto D' := D^{i:j\leftrightarrow l}$ that exchanges $d_{i,j}^2$ and $d_{i,l}^2$ and leaves every other entry unchanged. Whether this configuration is geometrically embeddable in 2D is deliberately set aside; this isolates objective preference from dimensionality constraints (empirically evaluated in Section~\ref{sec:results}). Crucially, t-SNE and UMAP penalize order inversions only in proportion to affinity gaps, vanishing when affinities tie. In contrast, SR-tSNE strictly enforces the correct rank order under all conditions.

\begin{theorem}[Pairwise order preference]
\label{thm:pref}
Fix a point $i$ and two of its true neighbors $j < l$, so that their input-space distances satisfy
$d^H_{i,j} < d^H_{i,l}$, where $d^H_{i,j} = \|x_i - x_{N_i(j)}\|_2$. Suppose the embedding inverts them, $d_{i,l}^2 < d_{i,j}^2$,
and let $D'$ be the row swap that corrects the inversion. Then:
\begin{enumerate}
\item[(a)] For the t-SNE loss with symmetrized affinities $p$,
\begin{equation}
\mathrm{KL}(D) - \mathrm{KL}(D') \;=\; (p_{ij} - p_{il})
\Bigl[\ln(1+d_{i,j}^2) - \ln(1+d_{i,l}^2)\Bigr].
\label{eq:kl_pref}
\end{equation}
The bracket is positive. The loss therefore prefers the correct order when
$p_{ij} > p_{il}$, is indifferent when $p_{ij} = p_{il}$, and prefers the
inversion when $p_{ij} < p_{il}$.
\item[(b)] For the UMAP cross-entropy with fuzzy-union weights $v$ and
low-dimensional kernel $w(x) = (1 + a x^{b})^{-1}$, in which $a > 0$ and $b > 0$ are UMAP's fitted kernel constants,
\begin{equation}
\mathrm{CE}(D) - \mathrm{CE}(D') \;=\; (v_{ij} - v_{il})\,
b\,\Bigl[\ln d_{i,j}^2 - \ln d_{i,l}^2\Bigr],
\label{eq:ce_pref}
\end{equation}
with the same three cases in the sign of $v_{ij} - v_{il}$.
\item[(c)] For the rank loss, with $r < r'$ the two positions the pair
occupies,
\begin{equation}
n\bigl[\mathcal{L}^{\text{hard}}_{\text{rank}}(D)
       - \mathcal{L}^{\text{hard}}_{\text{rank}}(D')\bigr]
\;\ge\; \frac{(l - j)\,[\mathrm{disc}(r) - \mathrm{disc}(r')]}{G^{\star}}
\;\ge\; \frac{\Delta_K}{G^{\star}} \;>\; 0 .
\label{eq:rank_pref}
\end{equation}
The preference is strict, and it does not depend on $P$, on $v$, or on the
embedding distances.
\end{enumerate}
\end{theorem}

Parts (a) and (b) show that t-SNE and UMAP enforce neighbor order strictly through pairwise affinity differences. Consequently, this ordering pressure vanishes, or backfires, under two common scenarios (Lemma~\ref{lem:affinity}): UMAP's fuzzy-union affinities saturate to one for close neighbors, while t-SNE's symmetrization frequently contradicts true input distances. The latter is remarkably prevalent: on MNIST, $32.9\%$ of adjacent pairs $(j, j{+}1)$ in a point's top-$30$ list exhibit $p_{i,j+1} \ge p_{i,j}$, affecting $98.5\%$ of points ($19.7\%$ and $91.3\%$ on CIFAR-10). For every such pair, the standard KL divergence is indifferent to the true order or actively favors an inversion. Part (c) highlights the core theoretical guarantee of SR-tSNE: regardless of input affinities, every neighbor inversion incurs a strict penalty of at least $\Delta_K / (n G^{\star})$. Corollary~\ref{cor:lambda} ensures that above a finite weight $\lambda^{\star}$ that does not grow with $n$, the total loss strictly penalizes all inversions.

Our final theoretical result links the rank loss back to the evaluation metric by decomposing exact-rank preservation into neighborhood membership and internal ordering. As introduced in Section~\ref{sec:metrics}, let $\text{O-DLP}_i = \text{DLP}_i \cdot A_i$, where $\text{DLP}_i$ measures membership retention and $A_i$ represents the fraction of retained neighbors preserving their exact ranks. The theorem below establishes an analytical order-blind baseline for shuffled neighborhoods and proves that the SR-tSNE rank loss directly bounds the ordering degradation among retained neighbors ($\text{DLP} - \text{O-DLP}$).

\begin{theorem}[Order beyond membership]
\label{thm:blind}
Fix an embedding and a neighborhood size $k'$, and let $A_i(k')$ denote the fraction of retained neighbors of point $i$ that preserve their exact original ranks.
\begin{enumerate}
\item[(i)] \emph{Decomposition.} O-DLP decomposes into membership times order:
\begin{equation}
\text{O-DLP}(k') = \frac{1}{n}\sum_{i=1}^n \text{DLP}_i(k')\,A_i(k'),
\end{equation}
and its normalized ratio $\omega(k') := k'\,\text{O-DLP}(k')/\text{DLP}(k') \in [0, k']$ represents the $\text{DLP}_i$-weighted mean of $k' A_i(k')$.
\item[(ii)] \emph{Order-blind baseline.} If the $k'$ embedded neighbors of each point are randomly shuffled, then $\mathbb{E}[A_i] = 1/k'$ and $\mathrm{Var}(k'\,\text{O-DLP}_i) \le \text{DLP}_i$. Assuming independent arrangements across points, $\text{O-DLP}(k') = \text{DLP}(k')/k'$ within a standard deviation of $k'^{-1}\sqrt{\text{DLP}(k')/n}$, which corresponds to $\omega(k') = 1$ within a standard deviation of:
\begin{equation}
\sigma_\omega = \frac{1}{\sqrt{n\,\text{DLP}(k')}} .
\label{eq:omega_res}
\end{equation}
\item[(iii)] \emph{Order deficit bound.} For $k' \le \min(k, m)$, the gap between membership retention and exact-rank preservation is strictly bounded by the discrete rank loss:
\begin{equation}
\text{DLP}(k') - \text{O-DLP}(k') \le \frac{2\,G^{\star}}{k'\,\Delta_K}\,
\mathcal{L}^{\text{hard}}_{\text{rank}}(Y) .
\label{eq:deficit}
\end{equation}
\end{enumerate}
\end{theorem}

Part (ii) is the \emph{order-blind baseline}: an embedding that keeps the
right neighbors and shuffles them scores $\omega = 1$, and at the benchmark
sizes of Section~\ref{sec:setup} Eq.~(\ref{eq:omega_res}) places it within $0.03$
for every method that keeps at least a fiftieth of its neighbors, and within $0.07$ for all of them. Perfect order is $\omega = k'$. Part (iii) refines
Theorem~\ref{thm:control}. The rank loss bounds not only the distance of O-DLP
from one but also the gap between O-DLP and DLP. Driving the loss down
therefore tightens the bound on the order of the neighbors that are kept, not only on how many
are kept.

\textbf{Remark.} The guarantees above characterize the intrinsic properties of the objectives. Specifically, Theorem~\ref{thm:control} establishes the theoretical trajectory of O-DLP as the rank loss is minimized. In practice, Section~\ref{sec:results} empirically evaluates how closely gradient-based optimization approaches these theoretical limits, while the equal-budget $\lambda{=}0$ control in \ref{app:seeds} confirms that the observed order gains stem directly from the rank regularization term itself.

\section{Experiments}\label{sec:setup}
\label{sec:setup_data}

\begin{table}[!htbp]
\centering
\caption{Overview of the eight benchmark datasets. Sources and data types are
listed in Table~\ref{tab:datasets_full} of \ref{app:figs}.}
\label{tab:datasets}
\begin{tabular}{llrrr}
\toprule
Dataset & Domain & $n$ & $p$ & Classes \\
\midrule
MNIST & Vision & 70\,000 & 784 & 10 \\
CIFAR-10 & Vision & 60\,000 & 512 & 10 \\
Fashion-MNIST & Vision & 70\,000 & 784 & 10 \\
Speech Commands & Speech & 59\,387 & 768 & 35 \\
Renal-Cortex & Single-cell & 50\,027 & 36\,030 & 35 \\
Brain-Cortex & Single-cell & 50\,212 & 58\,232 & 35 \\
Covertype & Remote sense & 59\,999 & 54 & 7 \\
AG News & NLP & 60\,000 & 768 & 4 \\
\bottomrule
\end{tabular}
\end{table}

\textit{Data and preprocessing.}
The core benchmark spans eight datasets across five distinct domains: computer vision, audio, natural language, tabular remote sensing, and single-cell transcriptomics (Table~\ref{tab:datasets}). For feature representations, MNIST and Fashion-MNIST are evaluated on raw pixels and Covertype retains its $54$ native attributes. CIFAR-10 uses $512$-dimensional pooled features from an ImageNet-pretrained ResNet-18, AG News mean-pooled bert-base-uncased embeddings ($768$-d), and Speech Commands mean-pooled wav2vec 2.0 features ($768$-d).

For the two single-cell transcriptomics datasets (Renal-Cortex and Brain-Cortex), raw counts first undergo standard library-size normalization and $\log(1{+}x)$ transformation, followed by the selection of the $2{,}000$ most highly variable genes. These are then reduced via PCA to $50$ principal components, providing an identical $50$-dimensional input to all methods. Across our benchmark tables, $p$ denotes the nominal input dimensionality (or the raw gene count for single-cell data prior to PCA).

\textit{Peer methods in comparison.} We benchmark SR-tSNE against twelve peer methods. Seven are established: t-SNE, UMAP, PHATE, TriMap, PaCMAP, ISOMAP~\cite{tenenbaum2000global} and a 2-layer autoencoder. Four are recent: Param\-PaCMAP~\cite{huang2024navigating}, SUDE~\cite{peng2025sude}, LocalMAP~\cite{wang2025localmap} and DREAMS~\cite{kury2026dreams}. The twelfth is the retrieval-oriented objective NeRV~\cite{venna2010information}. All methods receive identical inputs and precomputed neighbor graphs, running at their library defaults or official implementations (exact ISOMAP uses $k=10$ without landmark approximation; NeRV runs through the original dredviz implementation). Eleven of the twelve run on all eight benchmarks at full size. NeRV, whose implementation is built for sizes below $50{,}000$, is compared at matched scale on $5{,}000$-point subsamples of all eight benchmarks (Table~\ref{tab:nerv5k} of \ref{app:figs}). Every paired test therefore runs over all eight datasets.

\textit{SR-tSNE setting.} SR-tSNE
uses the smooth rank loss with $k{=}10$ and $m{=}10$ throughout. The weight $\lambda$
follows dataset size. It is $\lambda{=}50$ for $n \ge 50{,}000$ and $\lambda{=}5$
below that threshold. The same rule carries over to the seven vision
datasets of Section~\ref{sec:retrieval}.

\label{sec:setup_metrics}

\textit{Evaluation metrics.} The evaluation combines order, membership and co-ranking criteria.
O-DLP is the main criterion for order, DLP scores neighborhood membership, and
nDCG grades the returned list by position, all at $k{=}10$. We also report the
area under the $R_{NX}$ curve on a $20$k subsample of MNIST and CIFAR-10, read
off the co-ranking matrix and independent of the training objective. All four
take the high-dimensional neighbor ranking as reference; the clustering
criteria of Section~\ref{sec:clustering} are read against a curated annotation.

The paired tests certify the sign of the lead, and the tables give its size.
Every test is a two-sided Wilcoxon signed-rank over datasets, one observation per
dataset, so the smallest attainable $p$ is $0.0078$ with eight datasets, $0.0156$
with seven and $0.031$ with six. A $p$ at that floor certifies that SR-tSNE leads on
every dataset of the pair. Table~\ref{tab:effects} of the appendix gives the median
and range of the per-dataset differences for every baseline, and we read the two
together; the sign counts carry the conclusion and do not depend on a multiplicity
correction.

\begin{table}[!ht]
\centering
\caption{O-DLP at $k{=}10$ for t-SNE, the strongest published baseline
(DREAMS), the equal-budget $\lambda{=}0$ control and SR-tSNE, with the mean
over datasets and the gain of the rank term over the control. All twelve
methods are in Table~\ref{tab:odlp_main} below.}
\label{tab:control_main}
\begin{tabular}{l rrrr r}
\toprule
Dataset & t-SNE & DREAMS & $\lambda{=}0$ control & SR-tSNE & Gain (\%) \\
\midrule
MNIST            & .0781 & .0409 & .0793 & \textbf{.0927} & +17 \\
CIFAR-10         & .0545 & .0254 & .0548 & \textbf{.0597} & +9 \\
Fashion-MNIST    & .0672 & .0497 & .0672 & \textbf{.0792} & +18 \\
Speech Commands  & .0809 & .0533 & .0818 & \textbf{.0858} & +5 \\
Renal-Cortex     & .0578 & .0336 & .0578 & \textbf{.0741} & +28 \\
Brain-Cortex     & .0615 & .0241 & .0616 & \textbf{.0767} & +25 \\
Covertype        & .1496 & .1114 & .1497 & \textbf{.1913} & +28 \\
AG News          & .0810 & .0516 & .0811 & \textbf{.0905} & +12 \\
\midrule
Mean             & .0788 & .0488 & .0792 & \textbf{.0938} & +18 \\
\bottomrule
\end{tabular}
\end{table}

\begin{figure}[!htbp]
\centering
\includegraphics[width=\linewidth]{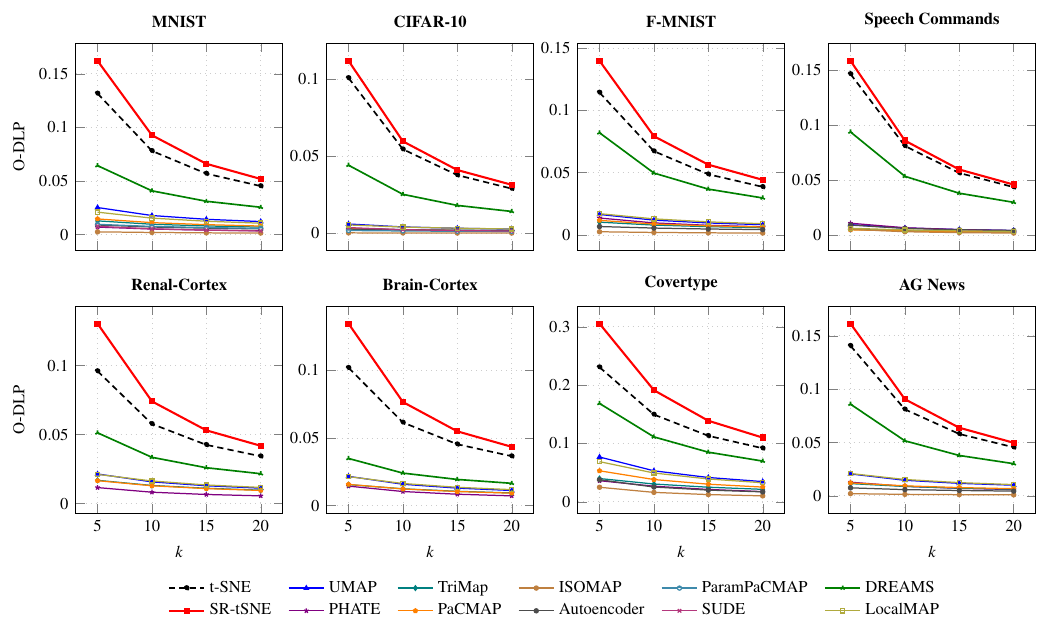}
\caption{O-DLP versus neighborhood size $k$ on the eight benchmarks of
Table~\ref{tab:datasets}. Each panel carries every method run on that dataset,
so the curves match the entries of Appendix Table~\ref{tab:odlp_main}.}
\label{fig:sweep}
\end{figure}

\subsection{Neighbor Order Preservation}\label{sec:results}

\textit{SR-tSNE achieves best exact-rank fidelity.}
Across all eight benchmarks, SR-tSNE records the highest O-DLP (Table~\ref{tab:control_main}, Figure~\ref{fig:sweep}), outperforming t-SNE by $+6\%$ to $+28\%$ (median $+18\%$, Table~\ref{tab:odlp_main}). Nonparametric Wilcoxon signed-rank tests confirm this advantage across all benchmarks ($p = 0.0078$ against each of the eleven full-scale baselines). On $5{,}000$-point subsamples, SR-tSNE surpasses NeRV by factors of $2.9\times$ to $9.2\times$ (Table~\ref{tab:nerv5k}). Absolute O-DLP values run between $0.06$ and $0.19$, the scale set by a criterion that credits a neighbor only at its exact position. Theorem~\ref{thm:control} bounds the ordinal error that remains by the rank loss. This validates Theorems~\ref{thm:control} and~\ref{thm:soft}: driving the smooth surrogate down under temperature annealing turns continuous gradient descent into discrete rank optimization.

\textit{Rank optimization simultaneously reinforces neighborhood membership.}
Sorting true neighbors ahead of mined intruders eliminates false low-dimensional neighbors. Consequently, SR-tSNE achieves the highest DLP on all eight benchmarks ($p = 0.0078$, Figure~\ref{fig:sweep_dlp}). Of the twelve methods, SR-tSNE is the one that holds the highest DLP and the highest O-DLP together, on every benchmark (Tables~\ref{tab:odlp_main} and~\ref{tab:dlp_main}). A downstream step reads both, so leading both is what carries the gain forward. In neighborhood sweeps (Figure~\ref{fig:sweep}), SR-tSNE retains top rank across all 32 dataset--$k$ settings for O-DLP, and leads DLP at $k \in \{5, 10\}$. This fidelity generalizes across multi-scale metrics. SR-tSNE secures the highest nDCG on all eight datasets (Table~\ref{tab:ndcg_main}). It also takes the top $R_{NX}$ AUC on MNIST and CIFAR-10, in both the local regime ($K \le 100$) and the full-range one (Table~\ref{tab:rnx}); that criterion measures neighbor agreement across scales, normalized against random chance.

\textit{Ablation isolates the rank term as the sole active driver.}
An equal-budget control ($\lambda = 0$), sharing the identical warm start, SGD momentum, and iteration budget, matches baseline t-SNE to within $2\%$ across all benchmarks (Table~\ref{tab:control_main}). Activating $\mathcal{L}_{\text{rank}}$ delivers the entire performance margin, boosting O-DLP by $+5\%$ to $+28\%$ and DLP by $+1\%$ to $+4\%$ ($p = 0.0078$). This improvement is strictly positive across all 30 metric differences over ten paired runs, two datasets by five perturbed starts (\ref{app:seeds}) and all 35 trading names against the control ($p = 5.8 \times 10^{-11}$). Initializing from a $2{,}000$-iteration t-SNE run further amplifies these gains (Table~\ref{tab:tsne2000}). This confirms Theorem 1 empirically: directly optimizing the rank objective governs the discrete ordinal error, driving O-DLP upward while standard affinity loss leaves it unconstrained.

\textit{Order is preserved beyond pure membership.}
Table~\ref{tab:omega} reads order against the order-blind baseline of Theorem~\ref{thm:blind}(ii), where $\omega = 1$ is
a neighborhood kept and shuffled and $\omega = 10$ is perfect order. Nine of the eleven peer methods run at full size sit between $1.07$
and $1.32$, so the order inside their neighborhoods stays close to that baseline. This is the room the rank term is
built to recover, and SR-tSNE holds the highest $\omega$ on all eight benchmarks. The gain is in the ordering rather
than in the membership: from t-SNE to SR-tSNE, O-DLP rises $3.2$ to $21.4$ times faster than DLP. Read as order
recovered above the baseline, $\omega - 1$, SR-tSNE stands $27\%$ above t-SNE and $25\%$ above the equal-budget
control, on every benchmark. The control shares the warm start, the optimizer, the budget and the seed, so that margin
is the rank term's. Position-wise hit rates show this gain concentrates where downstream retrieval reads (Figure~\ref{fig:perpos}, Table~\ref{tab:perpos}). At rank position one, SR-tSNE retrieves the true nearest neighbor $1.06\times$ to $1.41\times$ more often than t-SNE. Examples are $0.366 \to 0.443$ on MNIST and $0.268 \to 0.377$ on Renal-Cortex. Finally, on a held-out donor split of the Human Lung Cell Atlas~\cite{sikkema2023integrated}, SR-tSNE preserves its lead, exceeding t-SNE by $+30\%$ in O-DLP (\ref{app:heldout}). These results empirically validate Theorems~\ref{thm:pref} and \ref{thm:blind}: the objective suppresses pairwise inversions at the head of the list and strictly minimizes the order deficit beyond membership alone.

\begin{table}[!htbp]
\centering
\caption{Order beyond membership: $\omega(10) = 10\cdot\text{O-DLP}(10)/\text{DLP}(10)$, computed from
Tables~\ref{tab:odlp_main} and~\ref{tab:dlp_main}. $\omega = 1$ is the order-blind baseline of
Theorem~\ref{thm:blind}(ii), a neighborhood kept and shuffled; $\omega = 10$ is perfect order. At these sizes the
baseline is exact to within $0.03$ by Eq.~(\ref{eq:omega_res}) for every method with $\text{DLP} \ge 0.02$,
and to within $0.07$ for all of them. The eleven peer methods that run at full size are scored on all
eight benchmarks; NeRV is compared at matched scale in Table~\ref{tab:nerv5k}. Rows are sorted by the mean. Columns: MN MNIST, CF CIFAR-10, FM Fashion-MNIST, SC Speech Commands,
RC Renal-Cortex, BC Brain-Cortex, CT Covertype, AG AG News.}
\label{tab:omega}
\setlength{\tabcolsep}{2pt}
\begin{tabular}{lrrrrrrrrr}
\toprule
Method & MN & CF & FM & SC & RC & BC & CT & AG & Mean \\
\midrule
TriMap & 1.04 & 1.00 & 1.06 & 1.16 & 1.07 & 1.05 & 1.10 & 1.09 & 1.07 \\
PaCMAP & 1.10 & 1.02 & 1.08 & 1.13 & 1.07 & 1.05 & 1.25 & 1.08 & 1.10 \\
LocalMAP & 1.10 & 1.12 & 1.09 & 1.09 & 1.11 & 1.08 & 1.30 & 1.11 & 1.13 \\
Autoencoder & 1.07 & 1.05 & 1.07 & 1.46 & 1.03 & 1.04 & 1.32 & 1.10 & 1.14 \\
ParamPaCMAP & 1.02 & 1.01 & 1.02 & 1.64 & 1.03 & 1.04 & 1.32 & 1.10 & 1.15 \\
UMAP & 1.17 & 1.15 & 1.11 & 1.19 & 1.13 & 1.13 & 1.38 & 1.14 & 1.17 \\
PHATE & 1.08 & 1.10 & 1.12 & 1.56 & 1.13 & 1.11 & 1.16 & 1.17 & 1.18 \\
SUDE & 1.18 & 1.16 & 1.15 & 1.71 & 1.04 & 1.04 & 1.83 & 1.37 & 1.31 \\
ISOMAP & 1.12 & 0.98 & 1.08 & 2.66 & 1.03 & 0.99 & 1.37 & 1.37 & 1.32 \\
DREAMS & 1.44 & 1.80 & 1.67 & 2.02 & 1.36 & 1.19 & 1.91 & 1.77 & 1.64 \\
t-SNE & 2.01 & 2.61 & 2.00 & 2.63 & 1.78 & 1.79 & 2.35 & 2.30 & 2.18 \\
$\lambda{=}0$ control & 2.05 & 2.63 & 2.00 & 2.66 & 1.78 & 1.79 & 2.35 & 2.30 & 2.19 \\
SR-tSNE & \textbf{2.31} & \textbf{2.78} & \textbf{2.27} & \textbf{2.76} & \textbf{2.20} & \textbf{2.16} & \textbf{2.96} & \textbf{2.54} & \textbf{2.50} \\
\bottomrule
\end{tabular}
\end{table}

\begin{figure}[!htbp]
\centering
\includegraphics[width=\linewidth]{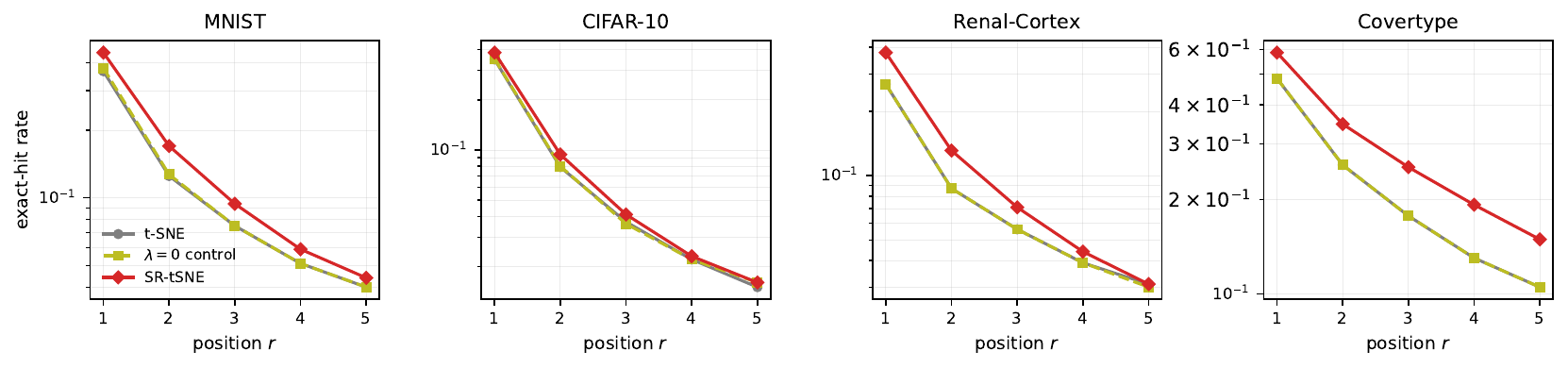}
\caption{\textit{Where the recovered order sits.} Exact-hit rate at the first five positions: the fraction of points whose
$r$-th input-space neighbor returns at position $r$ in the embedding, on four benchmarks (log scale). This is the head
of the list, where the position discount places its weight and where retrieval, label transfer and the
$k$-nearest-neighbor graph read. SR-tSNE lifts every one of these positions on each of the four benchmarks shown. The equal-budget
$\lambda{=}0$ control tracks t-SNE, so the lift is the rank term's. All ten positions on all eight benchmarks are in
Table~\ref{tab:perpos}.}
\label{fig:perpos}
\end{figure}

\begin{table}[!htbp]
\centering
\caption{O-DLP at $k{=}10$, the main criterion for order, with $\lambda{=}0$ the
equal-budget control: the same warm start and refinement schedule with the rank
term switched off. Bold marks the best method per dataset. The last column is the two-sided
Wilcoxon signed-rank test of SR-tSNE against that method, paired over the datasets
where both are available, bold where $p < 0.05$ and ``--'' where no test is
available. $^{\dagger}$NeRV runs on a $5{,}000$-point subsample of each benchmark,
the largest size its implementation reaches (Table~\ref{tab:nerv5k}); its row is scored on
that subsample and its test is paired with SR-tSNE on the same subsamples.}
\label{tab:odlp_main}
\setlength{\tabcolsep}{3.1pt}
\adjustbox{max width=\textwidth}{\begin{tabular}{@{}lcccccccccc r@{}}
\toprule
Method & MNIST & CIFAR & FMNIST & Speech & Renal & Brain & Cover & AGNews & $p$  \\
\midrule
t-SNE & .0781 & .0545 & .0672 & .0809 & .0578 & .0615 & .1496 & .0810 & .0078 \\
UMAP & .0178 & .0043 & .0120 & .0063 & .0158 & .0158 & .0533 & .0146 & .0078 \\
PHATE & .0052 & .0024 & .0098 & .0065 & .0083 & .0104 & .0270 & .0089 & .0078 \\
TriMap & .0097 & .0025 & .0080 & .0038 & .0132 & .0124 & .0306 & .0088 & .0078 \\
PaCMAP & .0112 & .0024 & .0089 & .0037 & .0129 & .0123 & .0383 & .0094 & .0078 \\
ISOMAP & .0021 & .0004 & .0021 & .0029 & .0026 & .0027 & .0163 & .0014 & .0078 \\
Autoencoder & .0071 & .0024 & .0056 & .0059 & .0067 & .0070 & .0255 & .0059 & .0078 \\
ParamPaCMAP & .0077 & .0017 & .0061 & .0047 & .0098 & .0094 & .0325 & .0053 & .0078 \\
SUDE & .0054 & .0025 & .0070 & .0061 & .0076 & .0065 & .0512 & .0080 & .0078 \\
LocalMAP & .0154 & .0042 & .0130 & .0045 & .0164 & .0164 & .0495 & .0153 & .0078 \\
DREAMS & .0409 & .0254 & .0497 & .0533 & .0336 & .0241 & .1114 & .0516 & .0078 \\
NeRV$^{\dagger}$ & .0164 & .0078 & .0385 & .0243 & .0394 & .0374 & .0463 & .0194 & .0078 \\
\midrule
SR-tSNE ($\lambda{=}0$) & .0793 & .0548 & .0672 & .0818 & .0578 & .0616 & .1497 & .0811 & .0078 \\
SR-tSNE (ours) & \textbf{.0927} & \textbf{.0597} & \textbf{.0792} & \textbf{.0858} & \textbf{.0741} & \textbf{.0767} & \textbf{.1913} & \textbf{.0905} & -- \\
\bottomrule
\end{tabular}}
\end{table}
\begin{table}[!htbp]
\centering
\caption{DLP at $k{=}10$ across all eight datasets, scoring neighborhood membership rather
than order, with $\lambda{=}0$ the equal-budget control. Bold marks the best method per dataset. The last column is the two-sided Wilcoxon
signed-rank test of SR-tSNE against that method, paired over all eight datasets,
bold where $p < 0.05$ and ``--'' for SR-tSNE itself.
$^{\dagger}$NeRV as in Table~\ref{tab:odlp_main}: scored and paired on the $5{,}000$-point subsamples of Table~\ref{tab:nerv5k}.}
\label{tab:dlp_main}
\setlength{\tabcolsep}{3.1pt}
\adjustbox{max width=\textwidth}{\begin{tabular}{@{}lcccccccccc r@{}}
\toprule
Method & MNIST & CIFAR & FMNIST & Speech & Renal & Brain & Cover & AGNews & $p$  \\
\midrule
t-SNE & .3880 & .2087 & .3368 & .3078 & .3243 & .3434 & .6369 & .3529 & .0078 \\
UMAP & .1526 & .0375 & .1085 & .0528 & .1398 & .1401 & .3874 & .1286 & .0078 \\
PHATE & .0481 & .0218 & .0872 & .0418 & .0734 & .0938 & .2331 & .0760 & .0078 \\
TriMap & .0930 & .0251 & .0756 & .0328 & .1233 & .1181 & .2770 & .0810 & .0078 \\
PaCMAP & .1018 & .0235 & .0827 & .0326 & .1205 & .1176 & .3063 & .0872 & .0078 \\
ISOMAP & .0187 & .0041 & .0195 & .0109 & .0252 & .0274 & .1188 & .0102 & .0078 \\
Autoencoder & .0664 & .0228 & .0525 & .0403 & .0651 & .0673 & .1925 & .0535 & .0078 \\
ParamPaCMAP & .0757 & .0168 & .0596 & .0287 & .0949 & .0908 & .2455 & .0481 & .0078 \\
SUDE & .0457 & .0215 & .0608 & .0357 & .0732 & .0623 & .2802 & .0584 & .0078 \\
LocalMAP & .1401 & .0376 & .1193 & .0412 & .1474 & .1513 & .3800 & .1376 & .0078 \\
DREAMS & .2839 & .1412 & .2983 & .2634 & .2477 & .2028 & .5832 & .2915 & .0078 \\
NeRV$^{\dagger}$ & .1354 & .0761 & .3153 & .2023 & .3284 & .3103 & .3466 & .1669 & .0078 \\
\midrule
SR-tSNE ($\lambda{=}0$) & .3873 & .2085 & .3368 & .3073 & .3243 & .3434 & .6368 & .3529 & .0078 \\
SR-tSNE (ours) & \textbf{.4009} & \textbf{.2149} & \textbf{.3493} & \textbf{.3110} & \textbf{.3375} & \textbf{.3545} & \textbf{.6452} & \textbf{.3568} & -- \\
\bottomrule
\end{tabular}}
\end{table}
\FloatBarrier
\subsection{Applications}
\label{sec:applications}
We evaluate three downstream tasks that read the embedding neighbor list from the top down: marker selection in algorithmic trading, retrieval and label transfer in vision, and cell-type discovery in single-cell transcriptomics. Because every embedding is routed through the identical downstream pipeline, reported metrics reflect purely paired differences between dimensional reduction methods.

\subsubsection{Algorithmic Trading}
\label{sec:trading}

\textit{Experimental setup.} We benchmark SR-tSNE within the trading-map pipeline of \cite{han2024explainable}. Each 5-minute bar contains eight market features (open, high, low, close, volume, average price, notional value, and time differential), with a binary direction flag indicating the trade side ($\mathrm{sign}(\text{close} - \text{open})$). Bars are embedded into 2D, and DBSCAN+K (\ref{app:scorer}) selects trade markers using fixed clustering parameters ($\mathrm{minPts} = 30$, $\varepsilon$ set to the 25th percentile of the $\mathrm{minPts}$-distance distribution). We evaluate 35 liquid US equities across seven sectors (5 stocks per group, Table~\ref{tab:pool35}) from 2026-04-14 to 2026-07-09 ($\approx 4{,}680$ bars per equity; $\lambda = 5, k = m = 10$). Both SR-tSNE and the equal-budget $\lambda = 0$ control refine the identical GPU t-SNE warm start (RAPIDS cuML), isolating the embedding objective as the sole experimental variable. Returns are simulated under realistic friction: $0$, $0.5$, $1$, and $2$ basis points (bps) per round trip.
Figure~\ref{fig:trading_case} follows one equity end to end, from the markers each embedding produces to the
positions they open and close.

\textit{Theoretical alignment.} The marker score reads $1 - \rho_i$, the Spearman rank disagreement between a bar's input-space and embedded neighbor order. Proposition~\ref{prop:scorer} proves that the per-bar rank loss $\mathcal{L}_{\text{rank}, i}$ strictly upper-bounds this quantity. Consequently, SR-tSNE directly minimizes the scoring metric, whereas standard t-SNE remains indifferent whenever affinities tie (Theorem~\ref{thm:pref}(a)).

\begin{figure*}[!t]
\centering
\includegraphics[width=\linewidth]{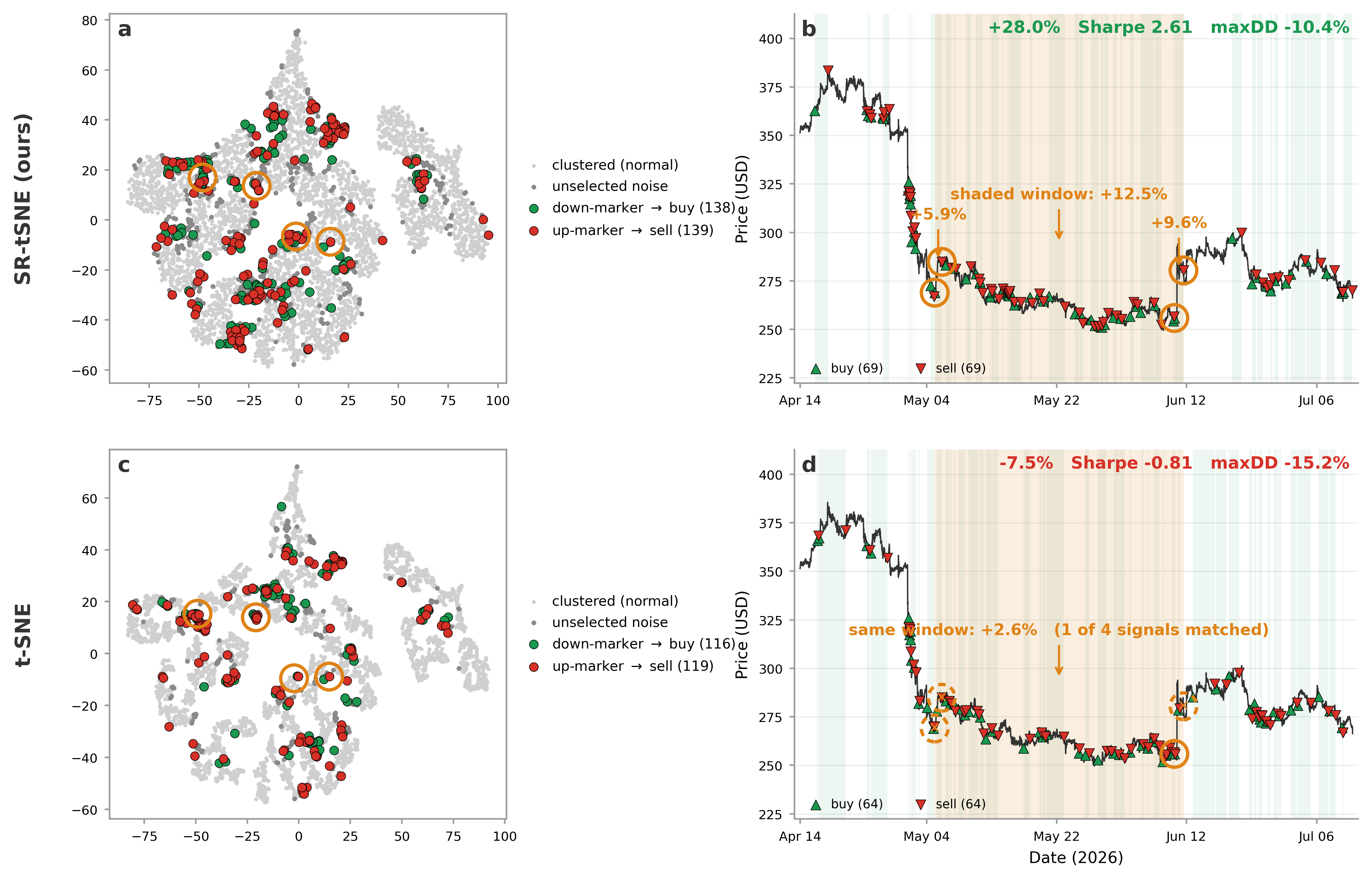}
\caption{End-to-end trading on IDCC at $2$\,bps (5-minute bars; identical DBSCAN$+$K hyperparameters). \textbf{(a, c)} Generated markers and \textbf{(b, d)} executed trades. The gold band highlights SR-tSNE's two strongest round trips ($+12.5\%$ vs. $+2.6\%$ for t-SNE); gold circles mark trade executions (dashed where t-SNE places no matching side). Over the quarter, SR-tSNE delivers $+28.0\%$ return at Sharpe $2.61$ versus $-7.5\%$ at Sharpe $-0.81$ for t-SNE.}
\label{fig:trading_case}
\end{figure*}

\textit{Superior net returns and risk-adjusted alpha.} Under $2$\,bps transaction costs, SR-tSNE achieves the highest net return on 21 of the 35 equities, with a median return of $+13.8\%$ versus $+3.6\%$ for t-SNE (Figure~\ref{fig:pool30_equity}). Against a cost-adjusted buy-and-hold baseline, SR-tSNE outperforms on 25 of 35 names ($p = 0.0017$, Wilcoxon signed-rank test). Beyond raw return, SR-tSNE significantly enhances capital efficiency and drawdown control (Table~\ref{tab:pool35_risk}): median annualized Sharpe ratio rises from $0.63$ to $2.32$, median maximum drawdown moderates from $-11.26\%$ to $-10.22\%$, and total executions decrease from $337$ to $266$ round trips.

\begin{figure}[!htbp]
\centering
\includegraphics[width=0.64\linewidth]{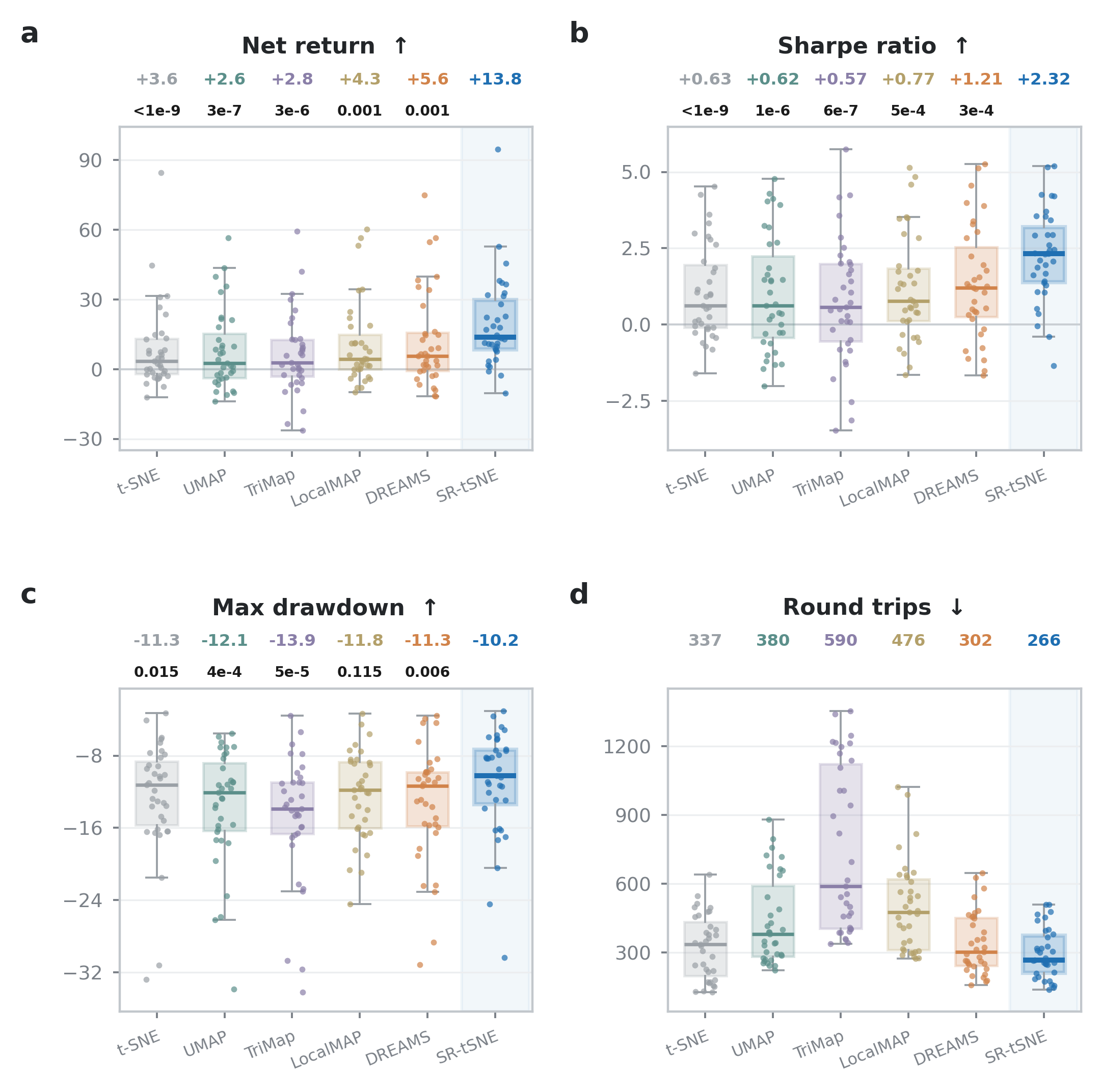}
\caption{Cross-method comparison across 35 equities at $2$\,bps under identical DBSCAN$+$K scoring: \textbf{(a)} net return (\%), \textbf{(b)} annualized Sharpe ratio, \textbf{(c)} maximum drawdown (\%), and \textbf{(d)} round trips. Boxes indicate interquartile ranges; dots represent individual equities; top numbers denote medians; bottom numbers in \textbf{(a)}--\textbf{(c)} report two-sided paired Wilcoxon $p$-values versus SR-tSNE.}
\label{fig:pool30_five}
\end{figure}

\textit{Consistent edge over modern manifold baselines.} At $2$\,bps the median net return of SR-tSNE is $+13.8\%$, against $+2.6\%$ for UMAP, $+2.8\%$ for TriMap, $+4.3\%$ for LocalMAP and $+5.6\%$ for DREAMS (Figure~\ref{fig:pool30_five}). Every one of those four paired comparisons favors SR-tSNE ($p = 3.3 \times 10^{-7}$, $2.9 \times 10^{-6}$, $0.0014$ and $0.0010$). SR-tSNE takes the top return on 31, 30, 28 and 27 of the 35 names, respectively, and doubles the median of the strongest competitor.

\textit{Robustness to trading frictions.} Because preserving rank reduces false-positive cluster transitions, SR-tSNE
trades sparingly and keeps most of its gross return once costs are charged. Raising the cost from $0$ to $2$\,bps,
SR-tSNE retains $+13.8\%$ of its $+18.2\%$, giving up $4.4$ points with round trips steady at $269 \to 266$;
TriMap retains $+2.8\%$ of $+13.0\%$ with round trips at $1{,}008 \to 590$, and UMAP $+2.6\%$ of $+9.1\%$ with
$542 \to 380$. A method that trades less pays the spread less often, so the ranking of the arms at the highest cost
level follows turnover as well as gross return.

\textit{Higher post-marker drift.} Evaluating unconditional forward returns confirms that SR-tSNE flags structurally superior inflection points (Figure~\ref{fig:marker_edge}). It leads t-SNE on 25 of 35 names at 30 minutes ($p = 0.0003$) and on 21 of 35 at 1 hour ($p = 0.007$). At 2 hours it leads on 25 of 35, with a median forward edge of $+0.063\%$ against $+0.027\%$ ($p = 0.005$).
\subsubsection{Retrieval and Label Transfer in Vision}
\label{sec:retrieval}

\textit{Exact neighbor alignment at the head of the list.}
Because downstream retrieval is inspected from the top down, whether the item in position one is truly the nearest neighbor determines downstream utility. Figure~\ref{fig:retrieval} inspects the first three retrieved positions across three representative queries each on MNIST, CIFAR-10, and Fashion-MNIST. Three queries on each dataset give nine inspected positions per dataset. SR-tSNE returns all nine at their exact input-space rank on each of the three ($9/9$, $9/9$, $9/9$). The equal-budget $\lambda = 0$ control returns $7/9$ on MNIST, $2/9$ on CIFAR-10, and $1/9$ on Fashion-MNIST (baseline t-SNE scores $6/9$, $2/9$, and $1/9$, while DREAMS attains $4/9$, $1/9$, and $0/9$). Crucially, ordinal misplacements in the baselines frequently trigger category errors: on the CIFAR-10 dog query, all seven comparative baselines retrieve incorrect classes (predominantly cats), and on Fashion-MNIST Shirt, all seven return a T-shirt. SR-tSNE eliminates these top-rank misclassifications.

\begin{sidewaysfigure}
\centering
\includegraphics[width=\linewidth]{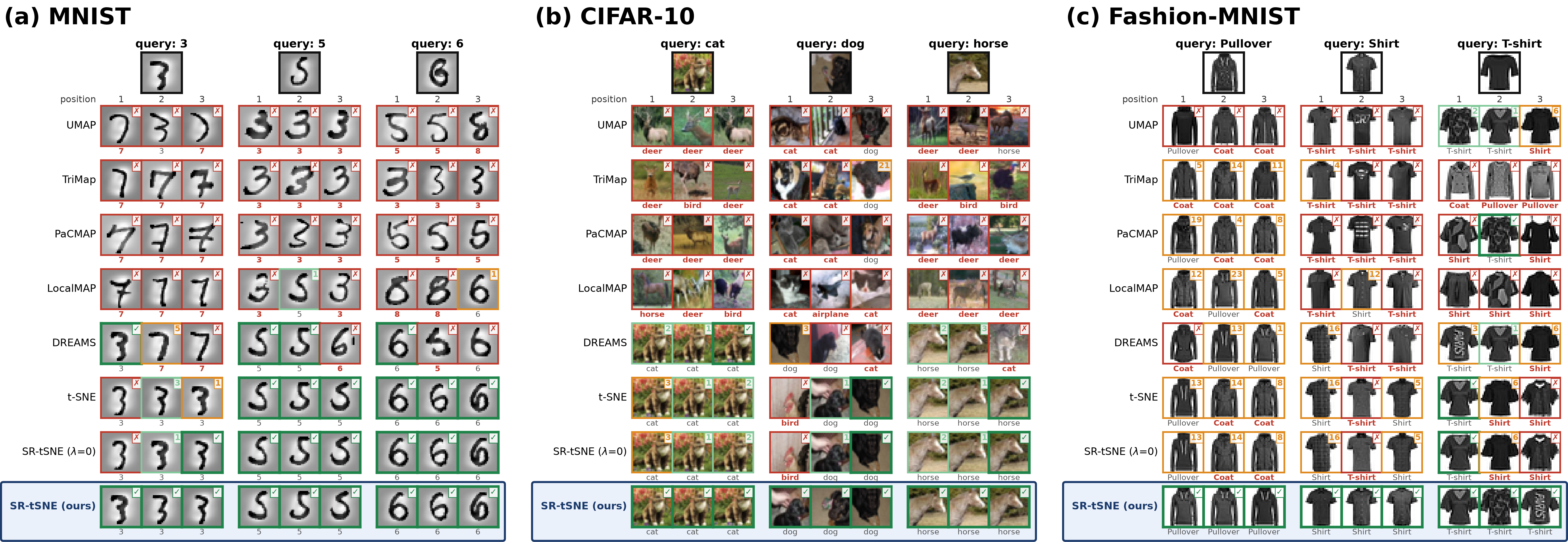}
\caption{Retrieval comparison on (a) MNIST, (b) CIFAR-10, and (c) Fashion-MNIST across the top three embedded positions. Thumbnail badges indicate high-dimensional ground-truth ranks within the query's top-$30$; frame colors encode rank displacement (dark green: exact match; pale green: offset by one; orange: offset by $\ge 2$; red: outside top-$30$). Class labels turn red upon class mismatch. Queries illustrate exact-rank preservation where comparative methods introduce ordinal and semantic errors.}
\label{fig:retrieval}
\end{sidewaysfigure}

\textit{Population-level exact-hit counts.}
This top-of-list advantage generalizes across the entire query distribution. Evaluating all $70{,}000$ MNIST queries, we measure the mean number of the first three embedded neighbors returned at their exact input-space rank. SR-tSNE averages $0.71$ exact hits out of three, substantially exceeding the $\lambda = 0$ control ($0.58$), t-SNE ($0.57$), and DREAMS ($0.26$), while UMAP, LocalMAP, PaCMAP, and TriMap remain below $0.10$.

\textit{Transductive label transfer on held-out splits.}
We next test whether geometric order preservation improves semantic generalization. Transductive label transfer runs across seven vision benchmarks, spanning $9{,}298$ to $70{,}000$ samples and $10$ to $200$ classes. Four are raw-pixel sets and three are ResNet-18 feature sets (Table~\ref{tab:cvdatasets}). Embeddings are computed on the full dataset, with a random $20\%$ held out exclusively from the retrieval index as queries. We evaluate precision at rank one ($\text{P}_1$) and Mean Reciprocal Rank (MRR) based on the first occurrence of the query class among the ten nearest reference points.

\textit{Consistent gains in semantic precision.}
SR-tSNE achieves strictly positive improvements in both $\text{P}_1$ and MRR across all seven datasets against all five baselines ($p = 0.0156$, two-sided Wilcoxon signed-rank test; Table~\ref{tab:crossdata}). Relative to t-SNE (the strongest baseline, with median $\text{P}_1 = 0.8035$ and $\text{MRR} = 0.8537$), SR-tSNE wins $7/7$ evaluations, while the $\lambda = 0$ control reproduces t-SNE within $0.0004$. The performance margin over the control is most pronounced on highly confusable datasets, adding $+0.0016$ to $+0.0164$ in $\text{P}_1$, led by EMNIST-Balanced (47 mixed alphanumeric classes) and Fashion-MNIST (Table~\ref{tab:crossdata_summary}).

\subsubsection{Single-Cell Type Discovery}
\label{sec:clustering}

\textit{Cell-type resolution hinges directly on embedded neighbor fidelity.}
Graph-based clustering reads the embedding's top-$k$ list, so exact order governs how cell types resolve.
In single-cell transcriptomics, unsupervised cell-type discovery forms the standard endpoint: a community-detection algorithm partitions the $k$-nearest-neighbor graph built on 2D coordinates into biological subpopulations. Graph edges exist solely where one cell falls within another's $k$ nearest embedded neighbors. Consequently, preserving the ground-truth top-$k$ ordering prevents intra-type fragmentation, directly evaluating the topological fidelity optimized by our rank loss.

\textit{Experimental protocol and oracle-envelope evaluation.}
We benchmark across six annotated single-cell datasets spanning $5{,}002$ to $50{,}212$ cells and $14$ to $35$ curated cell types (Table~\ref{tab:scdatasets}), all preprocessed to 50 principal components. A $k$-nearest-neighbor graph ($k = 15$) is constructed from the 2D coordinates and clustered via Leiden across eight resolutions in $[0.1, 2.0]$. To isolate geometric quality from resolution choice, each embedding is evaluated at its empirical upper-envelope resolution that maximizes the target criterion against annotations (identical oracle protocol across all methods). 

Embeddings are evaluated against curated cell-type annotations using three complementary external criteria. The Adjusted Rand Index (ARI) scores chance-corrected pairwise accuracy and the Adjusted Mutual Information (AMI) information retention across imbalanced populations. The Fowlkes--Mallows Index (FMI) scores geometric precision and recall.

\begin{table}[!ht]
\centering
\caption{Clustering performance across six annotated single-cell datasets under the Leiden upper-envelope protocol. Bold indicates the top method per cell. Sign counts and two-sided Wilcoxon signed-rank $p$-values evaluate SR-tSNE against each baseline across all datasets ($p < 0.05$ bolded).}
\label{tab:umapwins}
\setlength{\tabcolsep}{2.2pt}
\begin{tabular}{l ccccccc}
\toprule
Dataset & t-SNE & PHATE & UMAP & TriMap & LocalMAP & DREAMS & SR-tSNE \\
\midrule
\multicolumn{8}{l}{\textit{ARI} $\uparrow$} \\
\midrule
Trachea            & .7841 & .5082 & .6771 & .5988 & .5982 & .6001 & \textbf{.8221} \\
Liver-MacParland   & .5776 & .5347 & .5815 & .5842 & .5066 & .5629 & \textbf{.6012} \\
Kidney             & .7179 & .4276 & .6688 & .6874 & .6855 & .6943 & \textbf{.7499} \\
Kidney-Park        & .2069 & .2247 & .2143 & .2214 & .1792 & .1685 & \textbf{.2961} \\
Brain-Cortex       & .1141 & .1618 & .1640 & .1607 & .1444 & .1515 & \textbf{.2070} \\
Renal-Cortex       & .3063 & .1528 & .3752 & .3652 & .3443 & .3360 & \textbf{.3848} \\
\cmidrule(lr){1-8}
Wins & 6/6 & 6/6 & 6/6 & 6/6 & 6/6 & 6/6 &  \\
$p$ & \textbf{.031} & \textbf{.031} & \textbf{.031} & \textbf{.031} & \textbf{.031} & \textbf{.031} &  \\
\midrule
\multicolumn{8}{l}{\textit{AMI} $\uparrow$} \\
\midrule
Trachea            & .8212 & .7019 & .7854 & .7654 & .7683 & .7581 & \textbf{.8296} \\
Liver-MacParland   & .8045 & .7341 & .8073 & .7803 & .7694 & .7963 & \textbf{.8169} \\
Kidney             & .7698 & .6181 & .7666 & .7535 & .7624 & .7486 & \textbf{.7734} \\
Kidney-Park        & .6004 & .5593 & .6037 & .5983 & .5930 & .5753 & \textbf{.6360} \\
Brain-Cortex       & .5683 & .5794 & .6163 & .6148 & .5977 & .6037 & \textbf{.6477} \\
Renal-Cortex       & .6797 & .4094 & \textbf{.6939} & .6723 & .6851 & .6789 & .6887 \\
\cmidrule(lr){1-8}
Wins & 6/6 & 6/6 & 5/6 & 6/6 & 6/6 & 6/6 &  \\
$p$ & \textbf{.031} & \textbf{.031} & .062 & \textbf{.031} & \textbf{.031} & \textbf{.031} &  \\
\midrule
\multicolumn{8}{l}{\textit{FMI} $\uparrow$} \\
\midrule
Trachea            & .8153 & .5781 & .7231 & .6569 & .6611 & .6640 & \textbf{.8463} \\
Liver-MacParland   & .6087 & .5722 & .6124 & .6148 & .5437 & .5956 & \textbf{.6374} \\
Kidney             & .7459 & .4830 & .7020 & .7186 & .7188 & .7251 & \textbf{.7757} \\
Kidney-Park        & .4228 & .4396 & .4307 & .4373 & .3937 & .3806 & \textbf{.5094} \\
Brain-Cortex       & .2997 & .3568 & .3615 & .3574 & .3386 & .3471 & \textbf{.4081} \\
Renal-Cortex       & .3929 & .2206 & .4454 & .4308 & .4182 & .4200 & \textbf{.4496} \\
\cmidrule(lr){1-8}
Wins & 6/6 & 6/6 & 6/6 & 6/6 & 6/6 & 6/6 &  \\
$p$ & \textbf{.031} & \textbf{.031} & \textbf{.031} & \textbf{.031} & \textbf{.031} & \textbf{.031} &  \\
\bottomrule
\end{tabular}
\end{table}

\textit{Systematic gains across biological benchmarks.}
SR-tSNE achieves the highest score in 17 of 18 dataset--criterion comparisons (Table~\ref{tab:umapwins}). Against t-SNE, PHATE, TriMap, LocalMAP, and DREAMS, SR-tSNE wins on all six datasets across all three metrics ($6/6$, $p = 0.031$, two-sided Wilcoxon signed-rank test). Against UMAP, it wins $6/6$ in ARI and FMI ($p = 0.031$) and $5/6$ in AMI ($p = 0.062$). The margin peaks on Trachea, exceeding UMAP by $+0.145$ in ARI and TriMap by $+0.223$. These gains do not depend on oracle selection: fixing the Leiden resolution globally at $1.0$ or $0.5$ preserves SR-tSNE's top ARI across all six datasets ($p = 0.031$; median gains over t-SNE: $+23\%$ and $+19\%$, Table~\ref{tab:fixedres}).

\begin{figure}[!htbp]
\centering
\includegraphics[width=0.72\linewidth]{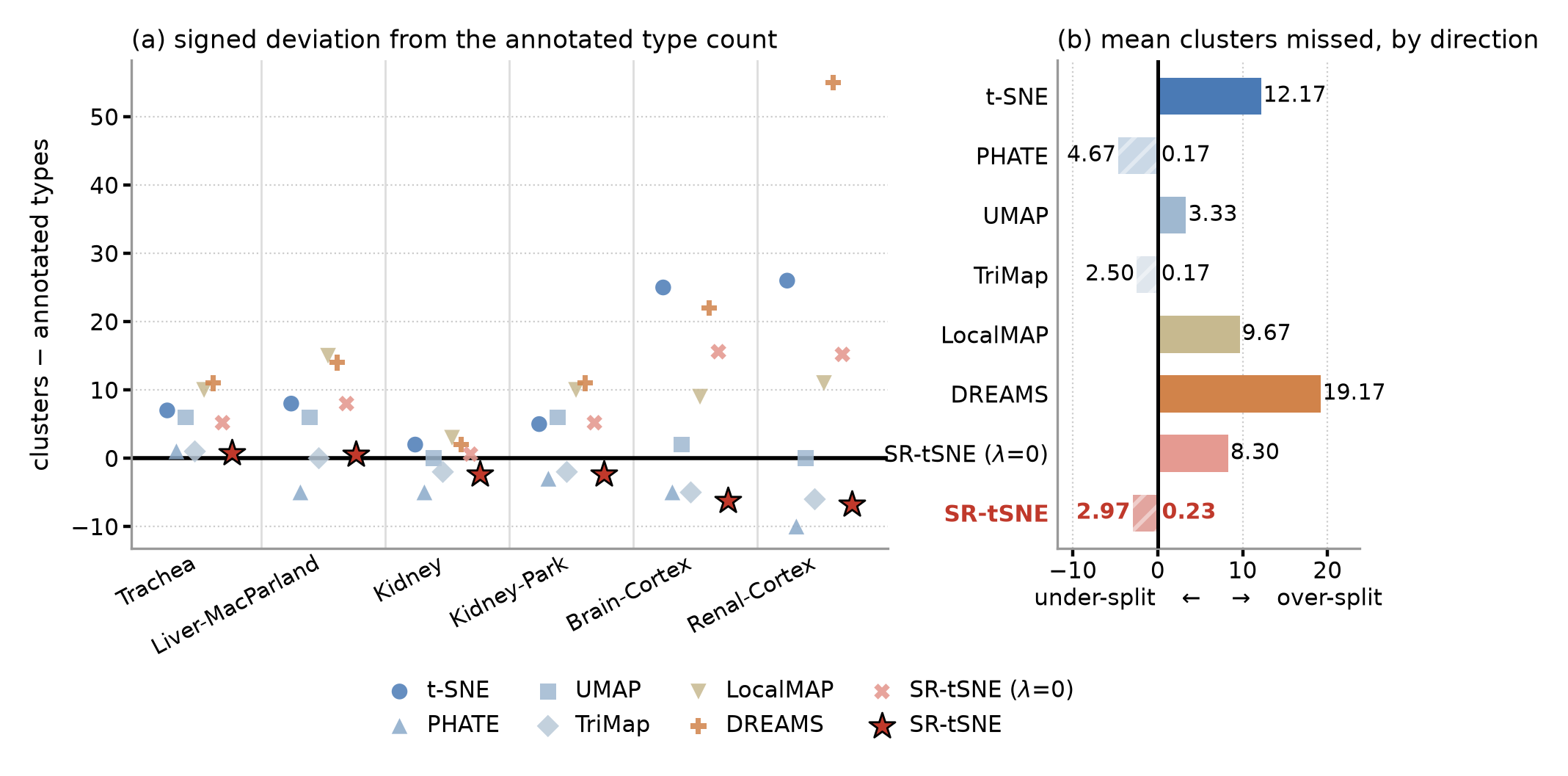}
\caption{Cluster count discrepancies relative to curated ground-truth annotations across six single-cell datasets. \textbf{(a)} Signed deviation per dataset (zero denotes the ground truth). \textbf{(b)} Mean absolute error decomposed into over-clustering (excess clusters, solid bars) and under-clustering (omitted clusters, hatched bars). SR-tSNE minimizes spurious cluster fragmentation.}
\label{fig:sccount}
\end{figure}

\textit{Suppression of artificial cluster fragmentation.}
Standard embeddings over-split homogeneous cell types into spurious sub-clusters (Figure~\ref{fig:sccount}). Baseline t-SNE produces $12.2$ excess clusters on average (up to $+26$ on Brain-Cortex), while DREAMS and LocalMAP over-split by $+19.2$ and $+9.7$, respectively. In contrast, SR-tSNE holds total cluster error to $3.2$ ($+0.2$ excess, $-3.0$ short), the range occupied by the closest peer methods, and leads them on 17 of the 18 dataset--criterion cells of biological alignment (Table~\ref{tab:umapwins}). The $\lambda = 0$ control over-splits by $+8.3$, proving that the rank loss drives this consolidation. By enforcing correct internal rank order, SR-tSNE maintains cluster connectivity. Each annotated type is partitioned into fewer disconnected components and retains more cells in the primary component ($0.725$ against $0.690$ on Renal-Cortex), while the cross-type edge rate stays within $0.6\%$ of baseline (Table~\ref{tab:scgraph}).

\subsection{Parameter Sensitivity and Stability}\label{sec:sensitivity}

\textit{Robust hyperparameter defaults across scales.}
SR-tSNE relies on three intuitive hyperparameters: neighbor count $k$, intruder pool size $m$, and regularization weight $\lambda$. We fix $k = m = 10$ across all experiments to align candidate pools with standard retrieval and scoring windows. The loss weight follows a simple dataset size rule: $\lambda = 50$ for large datasets ($\ge 50{,}000$ points) and $\lambda = 5$ for smaller regimes, calibrated once on MNIST and a market-equity pool and frozen thereafter. Performance is robust across scales: SR-tSNE consistently outperforms t-SNE across neighborhood sweeps ($k \in \{5, 10, 15, 20\}$, Figure~\ref{fig:sweep}) and across all scales in $R_{NX}$ AUC (Table~\ref{tab:rnx}). Perturbing initial coordinates across five seeds with $1\%$ Gaussian jitter yields strictly positive gains in all 30 pairwise trials (Table~\ref{tab:seeds}), verifying stability to initialization.

\textit{Listwise ranking vs.\ pairwise alternatives.}
The listwise form of the term is what delivers the gain. Under identical warm starts, budget and tuning, the listwise term holds the highest O-DLP of the three arms on all eight benchmarks, above both the warm start and a pairwise logistic surrogate over the same candidate pool (Table~\ref{tab:pairwise}). It provides $+5\%$ to $+28\%$ relative gains from that same start, so the discounted listwise penalty is what aligns the whole top-$k$ list.

\textit{Negligible refinement overhead.}
Evaluating wall-clock time on a dedicated NVIDIA A6000 demonstrates that computing the soft rank loss on candidate pools ($\mathcal{C}_i$, $K = 20$) introduces only $1.8\%$--$2.1\%$ overhead relative to an equal-budget $\lambda = 0$ control ($26.16$ vs.\ $25.69$ minutes on full MNIST, Table~\ref{tab:runtime}). The embedding cost remains dominated by the global t-SNE gradient computation, which admits standard $\mathcal{O}(n \log n)$ Barnes--Hut or FFT acceleration as detailed in Section~\ref{sec:discussion}.

Refining a longer warm start raises the gain further: from a $2{,}000$-iteration t-SNE run, SR-tSNE leads that run on
all eight benchmarks and exceeds its own $1{,}000$-iteration-start result (Table~\ref{tab:tsne2000}).

\section{Discussion}\label{sec:discussion}

\textit{Unified mechanism across downstream tasks.}
Our three applications evaluate the identical topological property from complementary perspectives. The trading scorer measures Spearman rank agreement. Vision retrieval evaluates the exact position of retrieved items. Leiden clustering reads the persistent edges within local top-$k$ graphs. By directly penalizing ordinal violations, the smooth rank loss acts specifically on the exact mechanism that governs downstream performance across these distinct domains.

\textit{Decoupling order fidelity from neighborhood membership.}
The rank objective acts on two complementary aspects of a neighborhood: the permutation order among true neighbors and the exclusion of intruders. Table~\ref{tab:omega} separates the two through the normalized order index $\omega$, and Theorem~\ref{thm:pref}(b) predicts where the affinity-based baselines sit on it: their ordering pressure follows an affinity gap that can close. The larger share of the O-DLP margin comes from restored internal order, governed by the discounted listwise gain, and a further share from membership, enforced by the intruder penalty (Figure~\ref{fig:sweep_dlp}). In single-cell discovery this dual control keeps the Leiden cluster count close to the curated one (Figure~\ref{fig:sccount}).

\textit{Robustness across initialization schemes.}
The benefits of rank refinement are invariant to the choice of warm start. Refining CPU-based Barnes--Hut t-SNE yields $+5\%$ to $+28\%$ gains in O-DLP across image and tabular benchmarks. Concurrently, refining GPU-accelerated cuML t-SNE in the algorithmic trading pipeline translates directly into higher net returns, superior Sharpe ratios, and lower turnover, an advantage that reproduces against t-SNE, the equal-budget control, UMAP and TriMap when initialized from CPU t-SNE (Table~S2 of the Supplementary Material). Furthermore, on $5{,}000$-point subsamples, SR-tSNE outperforms NeRV, the closest neighbor-retrieval objective, by factors of $2.9\times$ to $9.2\times$ in O-DLP (Table~\ref{tab:nerv5k}), establishing that explicit listwise rank bounds surpass soft information-theoretic divergences.

\textit{Extensions and future directions.}
\label{sec:limitations}
Three extensions follow directly from this framework:
(1)~\textit{Sub-quadratic scaling:} The smooth rank loss scales strictly linearly as $\mathcal{O}(n(k+m))$, evaluating only the local candidate pool $\mathcal{C}_i$. The computational bottleneck resides entirely in the exact t-SNE partition function ($\mathcal{O}(n^2)$) and exhaustive intruder search. Replacing the exact denominator with FFT-accelerated interpolation (FIt-SNE) and adopting approximate nearest-neighbor search brings overall complexity to $\mathcal{O}(n \log n)$ without modifying the rank term.
(2)~\textit{Continuous regularization schedules:} While our piecewise size rule ($\lambda \in \{5, 50\}$) proves robust, dynamically scheduling $\lambda(n)$ via continuous validation-based tuning could further optimize convergence trajectories.
(3)~\textit{Parametric extensions:} Formulating the soft rank loss as an objective for training deep neural encoders will enable instantaneous out-of-sample inference via a single forward pass, bypassing iterative gradient refinement for large-scale streaming pipelines.
\section{Conclusion}\label{sec:conclusion}

\textit{Summary of contributions.}
We introduced SR-tSNE, which adds a smooth listwise rank loss to t-SNE, along with Ordinal Degree of Locality Preservation (O-DLP). Our core contributions are two-fold:
(1)~\textit{Metric formulation and theoretical control:} We formulated O-DLP to credit a neighbor only when it returns at its exact high-dimensional rank under a position discount, establishing that the rank loss provably controls the discrete metric deficit (Theorem~\ref{thm:control}) with exponential surrogate convergence (Theorem~\ref{thm:soft}).
(2)~\textit{A rank-preserving embedding objective:} By ranking each point's high-dimensional neighbors against mined intruders under a position discount, one simple term forces the embedding to place the right neighbors in the right order with negligible computational overhead ($\sim 2\%$).

\textit{Empirical validation across downstream tasks.}
On eight benchmarks against twelve methods, SR-tSNE attains the highest exact-rank preservation (O-DLP), the highest neighborhood overlap (DLP) and the highest nDCG, on all eight. An equal-budget $\lambda{=}0$ control confirms the rank term as the sole active driver. This gain directly reaches tasks that read neighbor lists from the top down:
(1)~\textit{Algorithmic trading:} It delivers the best return of six embeddings on 21 of 35 names with higher Sharpe ratios, shallower drawdowns, and fewer trades;
(2)~\textit{Computer vision:} In label transfer across seven vision datasets, it leads on every one in precision at rank one and in mean reciprocal rank;
(3)~\textit{Single-cell discovery:} In clustering across six single-cell datasets, it leads on 17 of 18 dataset--criterion cells and prevents spurious cluster over-splitting in Leiden community detection.

\textit{Strengths and overall impact.}
The method's primary strengths are its economy and its exact match to the criterion: one term and one weight turn t-SNE into an embedding that keeps neighbor order, and O-DLP measures the exact quantity the term is built to control. The empirical evidence spans eight benchmarks, twelve methods, three applications, one held-out corpus, and an equal-budget control. Neighbor order is a fundamental property an embedding can and should be asked for, O-DLP measures whether it was delivered, and delivering it pays wherever a downstream step reads the list from the top. Beyond non-linear manifold learning, establishing that smooth listwise surrogates can provably govern discrete ordinal metrics provides a generalizable blueprint for top-$k$ representation learning across information retrieval, bioinformatics, and recommendation systems.

\textbf{Data and code availability.} All benchmark datasets are publicly available from their original sources. The
implementation of SR-tSNE and its criteria is released at
\url{https://github.com/hank08819/SR-tSNE}, where embeddings and preprocessing
scripts are being added.

\begingroup\setstretch{1.0}
\vspace{-6pt}\section*{Funding}
This work is partially supported by NSF 2229138 and the McCollum Endowed Chair
startup fund.

\vspace{-6pt}\section*{Declaration of competing interest}
The authors declare that they have no known competing interests.

\vspace{-6pt}\section*{CRediT authorship contribution statement}
\textbf{Henry Han}: Conceptualization, Methodology, Software, Supervision, Writing. \textbf{Luyu Niu}: Methodology, Software, Data curation. \textbf{Diane Li}: Investigation. \textbf{Jiacun Wang}: Investigation
\vspace{-6pt}\section*{Declaration of generative AI and AI-assisted technologies in the manuscript preparation process} AI is used for english polish.
\endgroup

\bibliographystyle{elsarticle-num}
\setlength{\bibsep}{0pt}
{\scriptsize\setstretch{1.0}
\bibliography{references}
\label{pg:refend}
}

\appendix
\label{pg:appstart}
\section{Supplementary Material}
\renewcommand{\thefigure}{\Alph{section}.\arabic{figure}}
\renewcommand{\thetable}{\Alph{section}.\arabic{table}}
\setcounter{figure}{0}
\setcounter{table}{0}

\FloatBarrier

\subsection{Frequently Asked Questions}
\label{app:faq}

This section collects the questions most often raised about SR-tSNE. Every answer states what the paper establishes and
names the section, table or figure that carries the evidence.

\begingroup
\setstretch{1.15}\small
\begin{enumerate}[label=\textbf{Q\arabic*.}, leftmargin=2.6em, itemsep=5pt]

\item \textbf{What is SR-tSNE?}
SR-tSNE is t-SNE with one added term, the smooth rank loss of Section~\ref{sec:smooth_rank_loss}. For each point the
loss reads its $k$ nearest neighbors in the input space, in order. It then asks the embedding to return the same
neighbors in the same order. The neighbor that is nearest in the input space should be nearest in the embedding, the
second should be second, and so on. The loss also mines \emph{intruders}: points that lie close to a point in the
embedding while belonging to no part of its input-space neighbor list. It pushes every intruder behind every true
neighbor. A single weight $\lambda$ sets how much the term counts against the t-SNE objective, and
Section~\ref{sec:method} states the two together.

\item \textbf{What gap does SR-tSNE close?}
Neighbor-embedding methods preserve who belongs to a neighborhood. The order inside that neighborhood is the part
SR-tSNE adds. Two embeddings can hold the same neighbors and still rank them differently, and SR-tSNE fixes both the
membership and the order.

\item \textbf{Why does neighbor order matter?}
Three common pipelines read the neighbor list from the top down. Retrieval and label transfer act on the item in
position one. Community detection builds its edges from top-$k$ lists. The trading scorer of Section~\ref{sec:trading}
reads a Spearman agreement between the two spaces. Each of them improves when the order improves.

\item \textbf{What is O-DLP?}
O-DLP credits a neighbor when it returns at its exact input-space rank, Eq.~(\ref{eq:odlp}). DLP credits the same
neighbor for being present at any position. O-DLP is therefore the strict reading of one neighborhood, and
Eq.~(\ref{eq:odlp_le_dlp}) relates the two. A third criterion, nDCG, grades the returned list by position and gives
partial credit for a near miss. All three are read at $k = 10$.

\item \textbf{What is novel here?}
Three elements are new. The first is an objective in the t-SNE family that asks for exact neighbor order. The second is
O-DLP, a criterion for that property. The third is a set of guarantees that tie the optimized loss to the reported
metric, Section~\ref{sec:theory}.

\item \textbf{What are the contributions?}
The method adds end-to-end differentiable rank regularization to t-SNE. The theory proves zero loss at perfect order,
exponential soft-to-exact convergence, a strict order preference, and an upper bound on the order deficit. The metric
extends DLP from membership to exact rank. The benchmarks cover eight datasets and twelve peer methods. The
applications cover label transfer, cell-type discovery and intraday trading.

\item \textbf{How does SR-tSNE relate to NeRV, RankVisu and triplet methods?}
NeRV optimizes a retrieval trade-off defined on neighborhood membership. RankVisu and local ordinal embedding work from
discrete rank relations. Triplet methods such as TriMap sample relative comparisons. SR-tSNE keeps t-SNE's continuous
similarities and adds one \emph{listwise} term. Listwise means the term scores a point's whole top-$k$ list at once,
rather than one pair at a time. It therefore combines both sources of information.
Table~\ref{tab:nerv5k} places it beside NeRV at matched scale, where SR-tSNE leads by $2.9\times$ to $9.2\times$ in
O-DLP.

\item \textbf{Why does the candidate pool carry intruders?}
The loss scores one candidate pool per point: the $k$ true neighbors, followed by the $m$ mined intruders. Intruders
give the loss something to rank the true neighbors against. With $m \ge 1$, an outside point that moves
between two true neighbors raises the loss, which is what locks membership and order together in
Theorem~\ref{thm:control}. The $m = 0$ case is worked out in \ref{app:m0}, where the term becomes a differentiable
Spearman correlation.

\item \textbf{Why a listwise penalty?}
The listwise form of Q7 is what delivers the gain. The warm start it improves on is the plain t-SNE run that SR-tSNE
refines. From that warm start the listwise term lifts O-DLP by $+5\%$ to $+28\%$. On the same candidates, start and budget, the
listwise term lifts O-DLP above a pairwise logistic surrogate on all eight benchmarks, Table~\ref{tab:pairwise}. Aligning the whole
top-$k$ list at once is therefore the active ingredient.

\item \textbf{Why a soft rank with an annealed temperature?}
The sigmoid count of Eq.~(\ref{eq:r_low}) supplies a gradient everywhere, so the discrete ordering becomes trainable.
Annealing $\tau$ from $5.0$ to $0.05$ keeps early gradients broad and late ranks sharp. Theorem~\ref{thm:soft} closes
the gap to the discrete rank exponentially in the margin. The loss that is optimized and the metric that is reported
therefore agree at the end of training.

\item \textbf{How are $k$, $m$ and $\lambda$ chosen?}
Every experiment in this paper uses $k = m = 10$, the window that retrieval and graph construction read. The weight
follows one size rule: $\lambda = 50$ at $n \ge 50{,}000$ and $\lambda = 5$ below it. The rule was calibrated once and
then carried unchanged through every benchmark, application and appendix study, \ref{app:hyperparams}.

\item \textbf{Does SR-tSNE use labels?}
SR-tSNE is fully unsupervised. The neighbor list and the intruder pool come from input-space distances alone. Labels
and curated annotations enter at evaluation time, where they give an independent reading of the geometry.

\item \textbf{What do the four guarantees give a reader?}
Theorem~\ref{thm:control} bounds the exact-rank deficit by the loss, so minimizing the loss controls the reported
metric. Theorem~\ref{thm:soft} makes the optimized surrogate agree with the discrete rank exponentially fast.
Theorem~\ref{thm:pref} gives a strict preference for the correct order at every inversion.
Theorem~\ref{thm:blind} bounds the order deficit among the neighbors that are kept. The prefactor of
Theorem~\ref{thm:control} is uniform: it is fixed by $k$, $m$ and the discount, and it holds for every embedding and
every dataset. Proofs are in \ref{app:proofs}.

\item \textbf{What does the rank term add to an affinity objective?}
Theorem~\ref{thm:pref}(a)--(b) shows that the order pressure of t-SNE and UMAP follows an affinity gap.
Theorem~\ref{thm:pref}(c) shows that the rank term charges at least $\Delta_K / (n G^{\star})$ for every inversion,
whatever the affinities are. The room this opens is measurable. On MNIST, $32.9\%$ of adjacent pairs in a top-$30$ list carry
$p_{i,j+1} \ge p_{i,j}$. At least one such pair occurs for $98.5\%$ of points.

\item \textbf{Is the evaluation independent of the loss?}
Yes, on four counts. The $R_{NX}$ AUC is read off the co-ranking matrix, so it is independent of the objective.
SR-tSNE leads it on MNIST and CIFAR-10, in the local regime and in the global one, Table~\ref{tab:rnx}. Cell types are
scored against curated annotations. Label transfer is scored against class labels, and trading in net return. All four
agree with the O-DLP ordering.

\item \textbf{What do the absolute O-DLP values tell us?}
O-DLP counts exact rank matches, so its scale is set by how strict that test is. At $k = 10$, an O-DLP of $0.093$ on
MNIST means that about one neighbor in ten comes back at its own position, averaged over the ten positions. SR-tSNE
holds the highest value on all eight benchmarks. Table~\ref{tab:omega} reads the same quantity against the order-blind
baseline, and that is the share the downstream tasks convert into higher precision at rank one, cleaner Leiden
clusters and higher net return.

\item \textbf{What does the $\lambda = 0$ control establish?}
It shares the warm start, the optimizer, the iteration budget and the seed with SR-tSNE. The rank term is therefore the
single difference between the two arms. The control lands within $2\%$ of t-SNE, and the rank term adds $+5\%$ to $+28\%$ in
O-DLP on top of it, Table~\ref{tab:control_main}. The margin is therefore attributable to the term itself.

\item \textbf{What happens when the warm start is run longer?}
The rank term adds order on top of a t-SNE run of any length. Table~\ref{tab:tsne2000} continues the warm start to
$2{,}000$ iterations on all eight benchmarks. Refined from that longer run, SR-tSNE leads it on all eight benchmarks, by a median of $+10\%$. It also stands
$+6\%$ above SR-tSNE refined from the $1{,}000$-iteration start.

\item \textbf{What does the paired test certify?}
Each test is a two-sided Wilcoxon signed-rank over datasets, with one observation per dataset. Eight datasets place the
smallest attainable value at $0.0078$, and a $p$ at that floor certifies a win on every dataset of the pair.
Table~\ref{tab:effects} gives the size of each win, and the two are read together. The sign counts are descriptive, so
the conclusion they carry stands under any multiplicity correction.

\item \textbf{How does the result behave across warm starts?}
It holds from a $1{,}000$-iteration start on all eight benchmarks and from a $2{,}000$-iteration start on all eight,
where SR-tSNE leads that longer run by a median of $+10\%$, Table~\ref{tab:tsne2000}. It holds for a CPU
Barnes--Hut start in the benchmarks and for a GPU cuML start in the trading pipeline, Section~\ref{sec:discussion}.

\item \textbf{How stable is the result across initializations?}
Five seeds perturb the initial coordinates with $1\%$ Gaussian jitter. All $30$ metric differences come out positive,
\ref{app:seeds}.

\item \textbf{Why is NeRV compared on $5{,}000$ points?}
The comparison is run at matched scale, on the range the original dredviz implementation is built for. Every method is
re-run on the same $5{,}000$-point subsample of all eight benchmarks and scored there, Table~\ref{tab:nerv5k}. SR-tSNE
holds the best O-DLP and DLP on every subsample, and leads NeRV by $2.9\times$ to $9.2\times$ in O-DLP.

\item \textbf{Is every peer method scored on every benchmark?}
Yes. All twelve are scored on all eight benchmarks, so the tables carry no gaps. Eleven of them run at full size.
NeRV is scored at matched scale, on the $5{,}000$-point subsamples of Table~\ref{tab:nerv5k}, the largest size its
implementation reaches. Every paired test therefore runs over all eight datasets.

\item \textbf{What does the trading study deliver?}
It is a paired comparison of embeddings inside one fixed pipeline. The features, the DBSCAN$+$K settings, the execution
rule and the four cost levels are identical for every arm, so the embedding objective is the single variable. At
$2$\,bps SR-tSNE takes the best net return on $21$ of $35$ names, with a median of $+13.8\%$ against $+3.6\%$ for
t-SNE, a median Sharpe ratio of $2.32$ against $0.63$, and fewer round trips, Table~\ref{tab:pool35_risk}. Against a
cost-adjusted buy-and-hold baseline it leads on $25$ of $35$ names, $p = 0.0017$.

\item \textbf{How does the single-cell result behave under the resolution choice?}
The upper-envelope protocol is applied identically to every method. The conclusion also carries over to a fixed
resolution. The Adjusted Rand Index scores chance-corrected agreement with the curated cell-type labels. Setting the
Leiden resolution globally at $1.0$ and at $0.5$ keeps SR-tSNE first on all six datasets. The median gains over t-SNE
are $+23\%$ and $+19\%$, Table~\ref{tab:fixedres}. The cluster counts of Figure~\ref{fig:sccount} agree, with SR-tSNE
holding the total deviation to $3.2$ clusters.

\item \textbf{How is the label-transfer protocol set up?}
Embeddings are computed on the full dataset, and a random $20\%$ is then withheld from the retrieval index to serve as
queries. This transductive protocol is applied identically to every method. The two criteria are precision at rank one ($\text{P}_1$)
and the mean reciprocal rank of the first correct neighbor (MRR). SR-tSNE leads on all seven datasets in
both. A donor-split corpus held out from all tuning is reported in \ref{app:heldout}, where
SR-tSNE leads t-SNE by $+30\%$ in O-DLP.

\item \textbf{Where in the list does the gain sit?}
It sits at the head, where the downstream steps read. At position one SR-tSNE returns the true nearest neighbor
$1.06\times$ to $1.41\times$ more often than t-SNE, for example $0.367 \to 0.443$ on MNIST and $0.268 \to 0.377$ on
Renal-Cortex, Figure~\ref{fig:perpos} and Table~\ref{tab:perpos}. Over all $70{,}000$ MNIST queries it returns $0.71$
of the first three neighbors at their exact rank, against $0.57$ for t-SNE.

\item \textbf{What does the rank term cost?}
It reads a candidate pool of $K = k + m = 20$ per point, so it costs $\mathcal{O}(n(k+m))$. That overhead is $1.8\%$ on full MNIST, $26.16$ against $25.69$ minutes on one NVIDIA
A6000, and $2.1\%$ on CIFAR-10, Table~\ref{tab:runtime}.

\item \textbf{How does SR-tSNE scale?}
The rank term is linear in $n$ by construction. Barnes--Hut or FFT interpolation for the t-SNE partition function,
together with approximate nearest-neighbor search for the intruder pool, brings the whole refinement to
$\mathcal{O}(n \log n)$ without changing the rank term, Section~\ref{sec:discussion}.

\item \textbf{Can the rank term be attached to other embeddings, and where is the code?}
The term reads pairwise distances inside the embedding alone, so any differentiable neighbor-embedding objective can
carry it with the same candidate pool and the same discount. The implementation and the criteria are released at
\url{https://github.com/hank08819/SR-tSNE}. Every setting is listed in \ref{app:hyperparams}, and the notation is
collected in Table~\ref{tab:notation}.

\end{enumerate}
\endgroup
\FloatBarrier

\subsection{Additional figures and tables}\label{app:figs}

\begin{table}[!ht]
\centering
\caption{Effect sizes behind the paired tests: the difference SR-tSNE minus
baseline in O-DLP and DLP at $k{=}10$, median and range over the datasets the
baseline reaches ($n$). Computed from Tables~\ref{tab:odlp_main}
and~\ref{tab:dlp_main}.}
\label{tab:effects}
\adjustbox{max width=\textwidth}{\begin{tabular}{l c rl rl}
\toprule
& & \multicolumn{2}{c}{$\Delta$O-DLP} & \multicolumn{2}{c}{$\Delta$DLP} \\
\cmidrule(lr){3-4}\cmidrule(lr){5-6}
Baseline & $n$ & Median & Range & Median & Range \\
\midrule
t-SNE & 8 & +0.0133 & [+0.0049, +0.0417] & +0.0097 & [+0.0032, +0.0132] \\
UMAP & 8 & +0.0711 & [+0.0554, +0.1380] & +0.2345 & [+0.1774, +0.2582] \\
PHATE & 8 & +0.0743 & [+0.0573, +0.1643] & +0.2666 & [+0.1931, +0.4121] \\
TriMap & 8 & +0.0765 & [+0.0572, +0.1607] & +0.2747 & [+0.1898, +0.3682] \\
PaCMAP & 8 & +0.0757 & [+0.0573, +0.1530] & +0.2681 & [+0.1914, +0.3389] \\
ISOMAP & 8 & +0.0800 & [+0.0593, +0.1750] & +0.3285 & [+0.2108, +0.5264] \\
Autoencoder & 8 & +0.0768 & [+0.0573, +0.1658] & +0.2920 & [+0.1921, +0.4527] \\
ParamPaCMAP & 8 & +0.0771 & [+0.0580, +0.1588] & +0.2860 & [+0.1981, +0.3997] \\
SUDE & 8 & +0.0760 & [+0.0572, +0.1401] & +0.2903 & [+0.1934, +0.3650] \\
LocalMAP & 8 & +0.0707 & [+0.0555, +0.1418] & +0.2246 & [+0.1773, +0.2698] \\
DREAMS & 8 & +0.0397 & [+0.0295, +0.0799] & +0.0695 & [+0.0476, +0.1517] \\
SR-tSNE ($\lambda{=}0$) & 8 & +0.0127 & [+0.0040, +0.0416] & +0.0097 & [+0.0037, +0.0136] \\
\bottomrule
\end{tabular}}
\end{table}

Table~\ref{tab:perpos} reads the gain position by position. SR-tSNE leads at the head of every one of the eight lists, and the mean over the ten positions is O-DLP$(10)$.

\begin{table}[!ht]
\centering
\caption{Exact-hit rate by position: the fraction of points whose $r$-th
high-dimensional neighbor returns at position $r$ in the embedding,
$r = 1,\dots,10$, for t-SNE, the equal-budget $\lambda{=}0$ control and SR-tSNE
on the eight benchmarks. The mean over the ten positions is O-DLP$(10)$.}
\label{tab:perpos}
\adjustbox{max width=\textwidth}{\begin{tabular}{ll rrrrrrrrrr}
\toprule
Dataset & Method & 1 & 2 & 3 & 4 & 5 & 6 & 7 & 8 & 9 & 10 \\
\midrule
MNIST & t-SNE & .366 & .125 & .075 & .051 & .040 & .032 & .028 & .024 & .020 & .019 \\
 & $\lambda{=}0$ & .378 & .127 & .075 & .051 & .040 & .032 & .027 & .024 & .020 & .019 \\
 & SR-tSNE & .443 & .170 & .094 & .059 & .044 & .032 & .026 & .022 & .019 & .017 \\
\addlinespace[2pt]
CIFAR & t-SNE & .350 & .079 & .037 & .022 & .015 & .012 & .009 & .008 & .007 & .006 \\
 & $\lambda{=}0$ & .353 & .080 & .036 & .022 & .016 & .012 & .009 & .008 & .007 & .006 \\
 & SR-tSNE & .385 & .094 & .041 & .023 & .016 & .012 & .008 & .007 & .006 & .005 \\
\addlinespace[2pt]
F-MNIST & t-SNE & .326 & .106 & .064 & .042 & .033 & .027 & .022 & .019 & .017 & .016 \\
 & $\lambda{=}0$ & .325 & .106 & .064 & .042 & .033 & .027 & .022 & .020 & .017 & .016 \\
 & SR-tSNE & .394 & .142 & .076 & .051 & .036 & .026 & .021 & .018 & .016 & .013 \\
\addlinespace[2pt]
Speech & t-SNE & .455 & .135 & .066 & .042 & .030 & .022 & .018 & .014 & .012 & .011 \\
 & $\lambda{=}0$ & .464 & .136 & .066 & .042 & .029 & .022 & .018 & .014 & .012 & .011 \\
 & SR-tSNE & .482 & .159 & .078 & .045 & .028 & .020 & .015 & .012 & .011 & .008 \\
\addlinespace[2pt]
Renal & t-SNE & .268 & .087 & .056 & .039 & .031 & .023 & .021 & .019 & .017 & .017 \\
 & $\lambda{=}0$ & .268 & .087 & .056 & .039 & .030 & .023 & .021 & .019 & .018 & .016 \\
 & SR-tSNE & .377 & .131 & .071 & .044 & .031 & .023 & .021 & .015 & .015 & .014 \\
\addlinespace[2pt]
Brain & t-SNE & .280 & .098 & .058 & .042 & .033 & .028 & .024 & .020 & .017 & .016 \\
 & $\lambda{=}0$ & .280 & .097 & .058 & .041 & .033 & .028 & .024 & .020 & .018 & .016 \\
 & SR-tSNE & .382 & .136 & .075 & .047 & .034 & .025 & .022 & .018 & .015 & .014 \\
\addlinespace[2pt]
Cover. & t-SNE & .485 & .257 & .177 & .130 & .105 & .089 & .076 & .066 & .059 & .053 \\
 & $\lambda{=}0$ & .485 & .257 & .177 & .130 & .105 & .089 & .076 & .066 & .059 & .053 \\
 & SR-tSNE & .586 & .347 & .253 & .192 & .149 & .116 & .090 & .071 & .059 & .049 \\
\addlinespace[2pt]
AG News & t-SNE & .406 & .134 & .074 & .051 & .037 & .030 & .025 & .020 & .017 & .015 \\
 & $\lambda{=}0$ & .406 & .135 & .074 & .051 & .037 & .030 & .025 & .020 & .018 & .015 \\
 & SR-tSNE & .447 & .170 & .092 & .058 & .039 & .029 & .022 & .018 & .016 & .013 \\
\bottomrule
\end{tabular}}
\end{table}

Table~\ref{tab:nerv5k} places every method on one common footing, a $5{,}000$-point subsample. SR-tSNE holds the best O-DLP and the best DLP on all eight subsamples.

\begin{table}[!ht]
\centering
\caption{Matched-scale comparison on a $5{,}000$-point subsample (seed 42) of
each benchmark: every method is re-run on the same subsample and scored on its
own high-dimensional neighbor lists at $k{=}10$. NeRV runs through the original
dredviz implementation at its defaults ($\lambda_{\text{NeRV}} = 0.1$, seed 42).
SR-tSNE uses the recipe of Section~\ref{sec:implementation} at the size-rule
weight $\lambda{=}5$; $\lambda{=}0$ is the equal-budget control. Bold marks the
best method per row.}
\label{tab:nerv5k}
\adjustbox{max width=\textwidth}{\begin{tabular}{ll rrrrr r r}
\toprule
& & t-SNE & UMAP & TriMap & PaCMAP & NeRV & $\lambda{=}0$ & SR-tSNE \\
\midrule
MNIST & O-DLP & .1228 & .0454 & .0300 & .0309 & .0164 & .1233 & \textbf{.1294} \\
 & DLP & .4601 & .3400 & .2585 & .2657 & .1354 & .4566 & \textbf{.4665} \\
\addlinespace[1pt]
CIFAR-10 & O-DLP & .0694 & .0150 & .0105 & .0095 & .0078 & .0695 & \textbf{.0718} \\
 & DLP & .2512 & .1262 & .1025 & .0869 & .0761 & .2482 & \textbf{.2549} \\
\addlinespace[1pt]
Fashion-MNIST & O-DLP & .1031 & .0370 & .0297 & .0290 & .0385 & .1036 & \textbf{.1137} \\
 & DLP & .4394 & .3045 & .2592 & .2542 & .3153 & .4387 & \textbf{.4456} \\
\addlinespace[1pt]
Speech Commands & O-DLP & .0997 & .0224 & .0151 & .0132 & .0243 & .0997 & \textbf{.0998} \\
 & DLP & .3511 & .1833 & .1351 & .1284 & .2023 & .3495 & \textbf{.3562} \\
\addlinespace[1pt]
Renal-Cortex & O-DLP & .0956 & .0465 & .0384 & .0378 & .0394 & .0959 & \textbf{.1175} \\
 & DLP & .5028 & .3797 & .3407 & .3383 & .3284 & .5013 & \textbf{.5121} \\
\addlinespace[1pt]
Brain-Cortex & O-DLP & .1043 & .0487 & .0383 & .0367 & .0374 & .1051 & \textbf{.1227} \\
 & DLP & .5133 & .3730 & .3373 & .3297 & .3103 & .5126 & \textbf{.5230} \\
\addlinespace[1pt]
Covertype & O-DLP & .1750 & .0729 & .0522 & .0462 & .0463 & .1769 & \textbf{.2124} \\
 & DLP & .6764 & .5176 & .4532 & .3984 & .3466 & .6772 & \textbf{.6872} \\
\addlinespace[1pt]
AG News & O-DLP & .1022 & .0328 & .0224 & .0206 & .0194 & .1035 & \textbf{.1089} \\
 & DLP & .3861 & .2568 & .2055 & .1921 & .1669 & .3831 & \textbf{.3885} \\
\bottomrule
\end{tabular}}
\end{table}

At this size SR-tSNE holds the best O-DLP and DLP on every subsample, and NeRV reaches $11$--$34\%$ of SR-tSNE's
O-DLP. SR-tSNE leads t-SNE on every
subsample, by $+0\%$ to $+23\%$ in O-DLP and $+1\%$ to $+2\%$ in
DLP, so the gain of the rank term does not depend on the size of the
benchmark.
\FloatBarrier

\begin{table*}[!t]
\centering
\caption{nDCG$_{10}$ across all methods. The ranked list is the embedding's own
ten nearest neighbors. Relevance is graded by high-dimensional rank and
intruders score zero. Bold marks the best per dataset. $^{\dagger}$NeRV is scored on the $5{,}000$-point
subsamples of Table~\ref{tab:nerv5k}, the largest size its implementation reaches.}
\label{tab:ndcg_main}
\setlength{\tabcolsep}{4.8pt}
\adjustbox{max width=\textwidth}{\begin{tabular}{@{}lcccccccccc@{}}
\toprule
Method & MNIST & CIFAR & FMNIST & Speech & Renal & Brain & Cover & AGNews  \\
\midrule
t-SNE & .5079 & .3333 & .4448 & .4652 & .4147 & .4380 & .7471 & .4821 \\
UMAP & .1625 & .0404 & .1137 & .0580 & .1449 & .1448 & .4248 & .1349 \\
PHATE & .0495 & .0232 & .0909 & .0499 & .0772 & .0976 & .2379 & .0809 \\
TriMap & .0928 & .0247 & .0745 & .0335 & .1196 & .1141 & .2776 & .0806 \\
PaCMAP & .1025 & .0230 & .0825 & .0330 & .1182 & .1143 & .3237 & .0869 \\
ISOMAP & .0184 & .0039 & .0190 & .0158 & .0245 & .0267 & .1382 & .0112 \\
Autoencoder & .0653 & .0218 & .0502 & .0439 & .0609 & .0632 & .2120 & .0527 \\
ParamPaCMAP & .0731 & .0156 & .0567 & .0323 & .0901 & .0855 & .2696 & .0464 \\
SUDE & .0495 & .0241 & .0651 & .0449 & .0741 & .0631 & .3459 & .0686 \\
LocalMAP & .1428 & .0379 & .1207 & .0420 & .1460 & .1494 & .4037 & .1409 \\
DREAMS & .3360 & .1947 & .3731 & .3670 & .2834 & .2195 & .6703 & .3710 \\
NeRV$^{\dagger}$ & .1387 & .0748 & .3295 & .2123 & .3308 & .3143 & .3728 & .1670 \\
\midrule
SR-tSNE ($\lambda{=}0$) & .5091 & .3335 & .4447 & .4655 & .4146 & .4379 & .7471 & .4821 \\
SR-tSNE (ours) & \textbf{.5689} & \textbf{.3682} & \textbf{.5025} & \textbf{.4993} & \textbf{.4890} & \textbf{.5121} & \textbf{.8013} & \textbf{.5260} \\
\bottomrule
\end{tabular}}
\end{table*}


Table~\ref{tab:cvdatasets} lists the seven vision datasets of the label-transfer protocol and the weight the size rule selects for each.

\begin{table}[!ht]
\centering
\caption{The seven vision datasets of the label-transfer protocol. \emph{Queries}
is the held-out fifth used to probe the retrieval index; the remaining points
form the index. $d$ is the input dimensionality, \emph{Classes} the number of classes, and
$\lambda$ follows the size rule of Section~\ref{sec:setup}. Sets larger than $70{,}000$ points are stratified-subsampled to that
size.}
\label{tab:cvdatasets}
\setlength{\tabcolsep}{2.8pt}
\adjustbox{max width=\textwidth}{\begin{tabular}{l>{\raggedright\arraybackslash}p{4.3cm}rrrrr}
\toprule
Dataset & Data type & $n$ & $d$ & Classes & Queries & $\lambda$ \\
\midrule
MNIST & Grayscale digits & 70\,000 & 784 & 10 & 14\,000 & 50 \\
Fashion-MNIST & Grayscale clothing & 70\,000 & 784 & 10 & 14\,000 & 50 \\
CIFAR-10 & Color, ResNet-18 & 60\,000 & 512 & 10 & 12\,000 & 50 \\
Tiny-ImageNet~\cite{le2015tiny} & Color, ResNet-18 & 70\,000 & 512 & 200 & 14\,000 & 50 \\
EMNIST-Balanced~\cite{cohen2017emnist} & Grayscale characters & 69\,983 & 784 & 47 & 13\,996 & 50 \\
CIFAR-100~\cite{krizhevsky2009learning} & Color, ResNet-18 & 60\,000 & 512 & 100 & 12\,000 & 50 \\
USPS~\cite{hull1994database} & Grayscale digits & 9\,298 & 256 & 10 & 1\,859 & 5 \\
\bottomrule
\end{tabular}}
\end{table}

Table~\ref{tab:crossdata_summary} summarizes label transfer over the seven datasets. SR-tSNE leads every baseline on both criteria, at the sign-count floor of the test.

\begin{table}[!htbp]
\centering
\setlength{\tabcolsep}{2.6pt}
\caption{Label transfer summarized over the seven vision datasets. Each row is one
embedding, \emph{Median} is its median score over the seven, and \emph{SR-tSNE wins on}
gives the sign count and the two-sided paired Wilcoxon $p$ of SR-tSNE against that
embedding. Bold marks the best median in a column and $p < 0.05$. Per-dataset scores are in
Table~\ref{tab:crossdata}.}
\label{tab:crossdata_summary}
\adjustbox{max width=\textwidth}{\begin{tabular}{l cc cc}
\toprule
& \multicolumn{2}{c}{P$_1$ $\uparrow$} & \multicolumn{2}{c}{MRR $\uparrow$} \\
\cmidrule(lr){2-3}\cmidrule(l){4-5}
Embedding & Median & SR-tSNE wins on ($p$) & Median & SR-tSNE wins on ($p$) \\
\midrule
UMAP & .7159 & 7/7 (\textbf{.0156}) & .7949 & 7/7 (\textbf{.0156}) \\
TriMap & .6722 & 7/7 (\textbf{.0156}) & .7611 & 7/7 (\textbf{.0156}) \\
t-SNE & .8035 & 7/7 (\textbf{.0156}) & .8537 & 7/7 (\textbf{.0156}) \\
LocalMAP & .7267 & 7/7 (\textbf{.0156}) & .7985 & 7/7 (\textbf{.0156}) \\
DREAMS & .7707 & 7/7 (\textbf{.0156}) & .8311 & 7/7 (\textbf{.0156}) \\
SR-tSNE & \textbf{.8053} & --- & \textbf{.8547} & --- \\
\bottomrule
\end{tabular}}
\end{table}

Table~\ref{tab:scdatasets} lists the six annotated single-cell datasets and the curated type counts the clustering is scored against.

\begin{table}[!ht]
\centering
\caption{The six annotated single-cell datasets of Table~\ref{tab:umapwins}, where
\emph{Cells} and \emph{Genes} are the dimensions of the raw count matrix,
\emph{Sparsity} the fraction of its entries that are zero and \emph{Types} the number of
curated cell types. The two atlases above $50\,000$ cells are subsampled as
described in Section~\ref{sec:clustering}, and their three columns describe the
subsampled matrix. All six are then reduced to the same 50 principal components
as in Section~\ref{sec:setup}.}
\label{tab:scdatasets}
\setlength{\tabcolsep}{3.6pt}
\adjustbox{max width=\textwidth}{\begin{tabular}{l>{\raggedright\arraybackslash}p{1.9cm}>{\raggedright\arraybackslash}p{3.4cm}rrrr}
\toprule
Dataset & Tissue & Source & Cells & Genes & Sparsity & Types \\
\midrule
Trachea & Trachea & Human Cell Landscape & 9\,949 & 24\,990 & 97.1\% & 14 \\
Liver-MacParland & Liver & MacParland et al., human liver & 8\,444 & 20\,007 & 93.4\% & 20 \\
Kidney & Kidney & Human Cell Landscape & 9\,153 & 25\,601 & 97.5\% & 18 \\
Kidney-Park & Kidney & Park et al., mouse kidney & 5\,002 & 16\,272 & 93.7\% & 16 \\
Brain-Cortex & Cerebral cortex & Precentral gyrus dissection & 50\,212 & 58\,232 & 90.5\% & 35 \\
Renal-Cortex & Renal cortex & Renal cortex benchmarking atlas & 50\,027 & 36\,030 & 95.6\% & 35 \\
\bottomrule
\end{tabular}}
\end{table}

Table~\ref{tab:scgraph} looks inside the graph that Leiden reads. Under SR-tSNE each annotated type stays in fewer pieces and keeps more of its cells in the main piece.

\begin{table}[!ht]
\centering
\caption{The $15$-nearest-neighbor graph of the embedding that Leiden reads
(symmetrized, unweighted), on the six datasets of Table~\ref{tab:umapwins}.
\emph{Pieces} is the number of connected components of the subgraph induced by
an annotated cell type, averaged over types; \emph{largest share} is the
fraction of a type held in its largest component, averaged over types;
\emph{cross-type edges} is the fraction of graph edges joining two types.
$\lambda{=}0$ and SR-tSNE are means over the five warm starts. Bold marks the
best of the four per row.}
\label{tab:scgraph}
\setlength{\tabcolsep}{2.5pt}
\adjustbox{max width=\textwidth}{\begin{tabular}{l cccc cccc ccc}
\toprule
& \multicolumn{4}{c}{Pieces per type $\downarrow$} & \multicolumn{4}{c}{Largest share $\uparrow$} & \multicolumn{3}{c}{Cross-type edges} \\
\cmidrule(lr){2-5}\cmidrule(lr){6-9}\cmidrule(lr){10-12}
Dataset & t-SNE & $\lambda{=}0$ & SR-tSNE & UMAP & t-SNE & $\lambda{=}0$ & SR-tSNE & UMAP & t-SNE & $\lambda{=}0$ & SR-tSNE \\
\midrule
Trachea            & 16.50 & 16.21 & \textbf{15.71} & 19.36 & 0.926 & 0.928 & \textbf{0.934} & 0.893 & 0.078 & 0.078 & 0.079 \\
Liver-MacParland   & 10.40 & 10.22 & \textbf{9.79} & 10.20 & 0.885 & 0.885 & \textbf{0.894} & 0.886 & 0.169 & 0.169 & 0.171 \\
Kidney             & 16.11 & 15.97 & \textbf{15.51} & 16.00 & 0.929 & 0.930 & \textbf{0.935} & 0.934 & 0.123 & 0.123 & 0.125 \\
Kidney-Park        & 4.44 & 4.44 & 4.41 & \textbf{4.31} & 0.868 & 0.868 & \textbf{0.869} & 0.855 & 0.029 & 0.029 & 0.031 \\
Brain-Cortex       & 10.06 & 9.99 & \textbf{9.81} & 10.46 & 0.672 & 0.681 & \textbf{0.705} & 0.674 & 0.020 & 0.021 & 0.021 \\
Renal-Cortex       & 47.40 & 47.46 & \textbf{45.93} & 47.26 & 0.690 & 0.692 & \textbf{0.725} & 0.717 & 0.158 & 0.159 & 0.164 \\
\bottomrule
\end{tabular}}
\end{table}

Table~\ref{tab:fixedres} repeats the clustering with the resolution fixed in advance, so no annotation enters the choice. SR-tSNE keeps the top ARI on all six datasets at both settings.

\begin{table}[!ht]
\centering
\caption{Leiden at a resolution fixed in advance ($1.0$, then $0.5$) for every
method and dataset, so no annotation enters the choice; each row scores one
partition. ARI per dataset with the sign count and two-sided Wilcoxon $p$ of
SR-tSNE against each method over the six datasets (bold where $p<0.05$), then
AMI and FMI averaged over the six datasets and the mean number of clusters
found (the annotations average $23$ types). $\lambda{=}0$ and SR-tSNE are means
over five warm starts. Bold marks the best per row.}
\label{tab:fixedres}
\setlength{\tabcolsep}{2.5pt}
\adjustbox{max width=\textwidth}{\begin{tabular}{l cccccccc}
\toprule
Dataset & t-SNE & PHATE & UMAP & TriMap & LocalMAP & DREAMS & $\lambda{=}0$ & SR-tSNE \\
\midrule
\multicolumn{9}{l}{\textit{Resolution 1.0}, ARI $\uparrow$} \\
\midrule
Trachea            & .2343 & .2100 & .2156 & .2201 & .2082 & .2013 & .2311 & \textbf{.2674} \\
Liver-MacParland   & .3333 & .2988 & .3150 & .3448 & .2893 & .3098 & .3270 & \textbf{.3835} \\
Kidney             & .3131 & .2243 & .2706 & .2774 & .2630 & .2496 & .3029 & \textbf{.3332} \\
Kidney-Park        & .0663 & .0627 & .0603 & .0675 & .0580 & .0604 & .0673 & \textbf{.0865} \\
Brain-Cortex       & .0333 & .0373 & .0407 & .0430 & .0414 & .0379 & .0348 & \textbf{.0565} \\
Renal-Cortex       & .1374 & .0683 & .1433 & .1523 & .1491 & .1335 & .1384 & \textbf{.1885} \\
\cmidrule(lr){1-9}
SR-tSNE wins on & 6/6 & 6/6 & 6/6 & 6/6 & 6/6 & 6/6 & 6/6 &  \\
two-sided $p$ & \textbf{.031} & \textbf{.031} & \textbf{.031} & \textbf{.031} & \textbf{.031} & \textbf{.031} & \textbf{.031} &  \\
Mean AMI over the six (6/6 best) & .5975 & .5243 & .5955 & .5978 & .5940 & .5868 & .5970 & \textbf{.6204} \\
Mean FMI over the six (6/6 best) & .3135 & .2644 & .3039 & .3118 & .2974 & .2933 & .3110 & \textbf{.3474} \\
Mean clusters found (truth 23.0) & 84 & 68 & 74 & 63 & 78 & 90 & 79 & 52 \\
\midrule
\multicolumn{9}{l}{\textit{Resolution 0.5}, ARI $\uparrow$} \\
\midrule
Trachea            & .3626 & .2984 & .3057 & .3423 & .2985 & .2759 & .3448 & \textbf{.3934} \\
Liver-MacParland   & .4456 & .4256 & .3841 & .3933 & .3779 & .3731 & .4476 & \textbf{.5398} \\
Kidney             & .4125 & .2886 & .3559 & .3781 & .3487 & .3297 & .4189 & \textbf{.4829} \\
Kidney-Park        & .1062 & .0934 & .1049 & .0994 & .0829 & .0996 & .1010 & \textbf{.1209} \\
Brain-Cortex       & .0519 & .0546 & .0662 & .0680 & .0599 & .0601 & .0509 & \textbf{.0830} \\
Renal-Cortex       & .1881 & .0942 & .2037 & .2031 & .2025 & .1744 & .1886 & \textbf{.2468} \\
\cmidrule(lr){1-9}
SR-tSNE wins on & 6/6 & 6/6 & 6/6 & 6/6 & 6/6 & 6/6 & 6/6 &  \\
two-sided $p$ & \textbf{.031} & \textbf{.031} & \textbf{.031} & \textbf{.031} & \textbf{.031} & \textbf{.031} & \textbf{.031} &  \\
Mean AMI over the six (6/6 best) & .6325 & .5517 & .6306 & .6305 & .6260 & .6187 & .6323 & \textbf{.6585} \\
Mean FMI over the six (6/6 best) & .3838 & .3219 & .3654 & .3708 & .3549 & .3479 & .3797 & \textbf{.4278} \\
Mean clusters found (truth 23.0) & 61 & 45 & 52 & 44 & 58 & 68 & 58 & 37 \\
\bottomrule
\end{tabular}}
\end{table}

Table~\ref{tab:rnx} scores a criterion the training objective never sees. SR-tSNE takes the top AUC in both the local and the global regime, on both datasets.

\begin{table}[!ht]
\centering
\setlength{\tabcolsep}{3.6pt}
\caption{$R_{NX}$ AUC on a $20$k subsample of MNIST and CIFAR-10. The criterion is
read off the co-ranking matrix and is independent of the training objective.
All methods are scored on the same $20$k subsample (seed 42) from their stored full-size
embeddings; $\lambda{=}0$ is the equal-budget control of Table~\ref{tab:control_main}.
Bold marks the best method per column.}
\label{tab:rnx}
\adjustbox{max width=\textwidth}{\begin{tabular}{lcc cc}
\toprule
& \multicolumn{2}{c}{AUC (full range, $1/K$)} & \multicolumn{2}{c}{AUC ($K \le 100$)} \\
\cmidrule(lr){2-3} \cmidrule(lr){4-5}
Method & MNIST & CIFAR & MNIST & CIFAR \\
\midrule
t-SNE & .3327 & .2252 & .3627 & .1909 \\
UMAP & .2524 & .1557 & .2089 & .0580 \\
PHATE & .1541 & .1268 & .0725 & .0339 \\
TriMap & .2211 & .1457 & .1453 & .0464 \\
PaCMAP & .2269 & .1440 & .1594 & .0439 \\
ISOMAP & .1463 & .1099 & .0392 & .0118 \\
Autoencoder & .2393 & .1489 & .1285 & .0481 \\
ParamPaCMAP & .2188 & .1443 & .1500 & .0422 \\
SUDE & .1392 & .1166 & .0595 & .0325 \\
LocalMAP & .2432 & .1518 & .1962 & .0610 \\
DREAMS & .2952 & .1939 & .2971 & .1424 \\
\midrule
SR-tSNE ($\lambda{=}0$) & .3329 & .2253 & .3631 & .1910 \\
SR-tSNE (ours) & \textbf{.3370} & \textbf{.2278} & \textbf{.3718} & \textbf{.1966} \\
\bottomrule
\end{tabular}}
\end{table}

Table~\ref{tab:datasets_full} gives the source and the data type of each benchmark alongside its dimensions.

\begin{table}[!htbp]
\centering
\setlength{\tabcolsep}{3pt}
\caption{Overview of the eight benchmark datasets and their data types. Here $n$ is the
number of points, $d$ the input dimensionality and \emph{Classes} the number of
annotated labels.}
\label{tab:datasets_full}
\adjustbox{max width=\textwidth}{\begin{tabular}{l>{\raggedright\arraybackslash}p{2.1cm}>{\raggedright\arraybackslash}p{3.5cm}rrr}
\toprule
Dataset & Domain & Data Type & $n$ & $d$ & Classes \\
\midrule
MNIST~\cite{lecun1998gradient} & Computer Vision & Grayscale images (handwritten digits) & 70\,000 & 784 & 10 \\
CIFAR-10~\cite{krizhevsky2009learning} & Computer Vision & Color images (CNN features) & 60\,000 & 512 & 10 \\
Fashion-MNIST~\cite{xiao2017fashion} & Computer Vision & Grayscale images (clothing articles) & 70\,000 & 784 & 10 \\
Speech Commands~\cite{warden2018speech} & Speech / Audio & Spoken keywords (wav2vec\,2.0 features~\cite{baevski2020wav2vec}) & 59\,387 & 768 & 35 \\
Renal-Cortex & Single-cell genomics & scRNA-seq, renal cortex benchmarking atlas & 50\,027 & 36\,030 & 35 \\
Brain-Cortex & Single-cell genomics & Precentral gyrus dissection & 50\,212 & 58\,232 & 35 \\
Covertype~\cite{blackard1999comparative} & Remote sensing & Cartographic and soil attributes (tabular) & 59\,999 & 54 & 7 \\
AG News~\cite{zhang2015character} & Natural language & News headlines (BERT embeddings~\cite{devlin2019bert}) & 60\,000 & 768 & 4 \\
\bottomrule
\end{tabular}}
\end{table}

Figure~\ref{fig:pool30_equity} follows the capital of all 35 names through the quarter, one row per sector group.

\begin{figure}[!htbp]
\centering
\includegraphics[width=\linewidth]{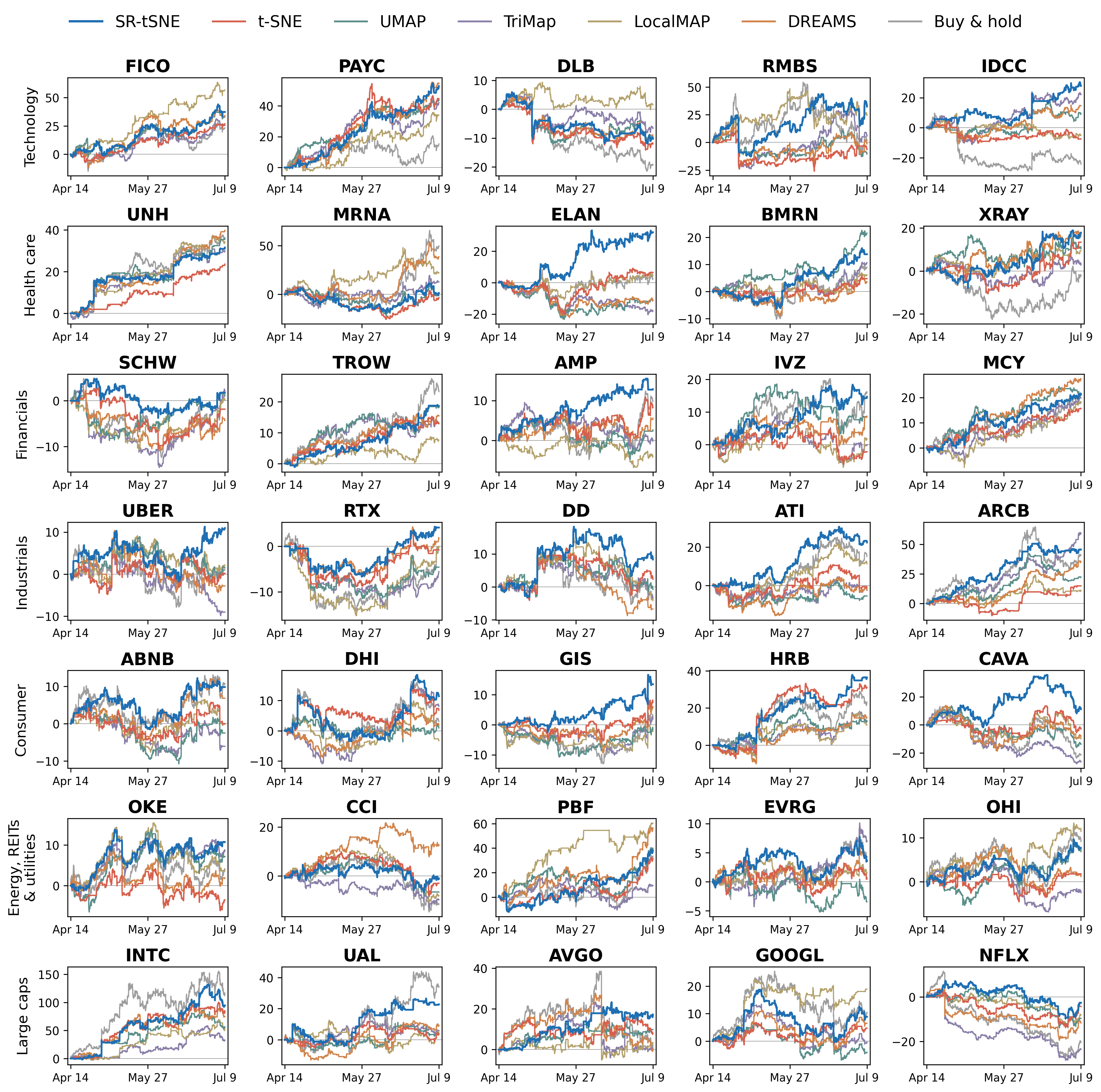}
\caption{Cumulative return of the 35 names at $2$\,bps, one row per group.
Buy and hold is charged the same round-trip cost.}
\label{fig:pool30_equity}
\end{figure}

Figure~\ref{fig:marker_edge} scores the markers themselves, by the move that follows them, so the result does not depend on the execution rule.

\begin{figure}[!htbp]
\centering
\includegraphics[width=\linewidth]{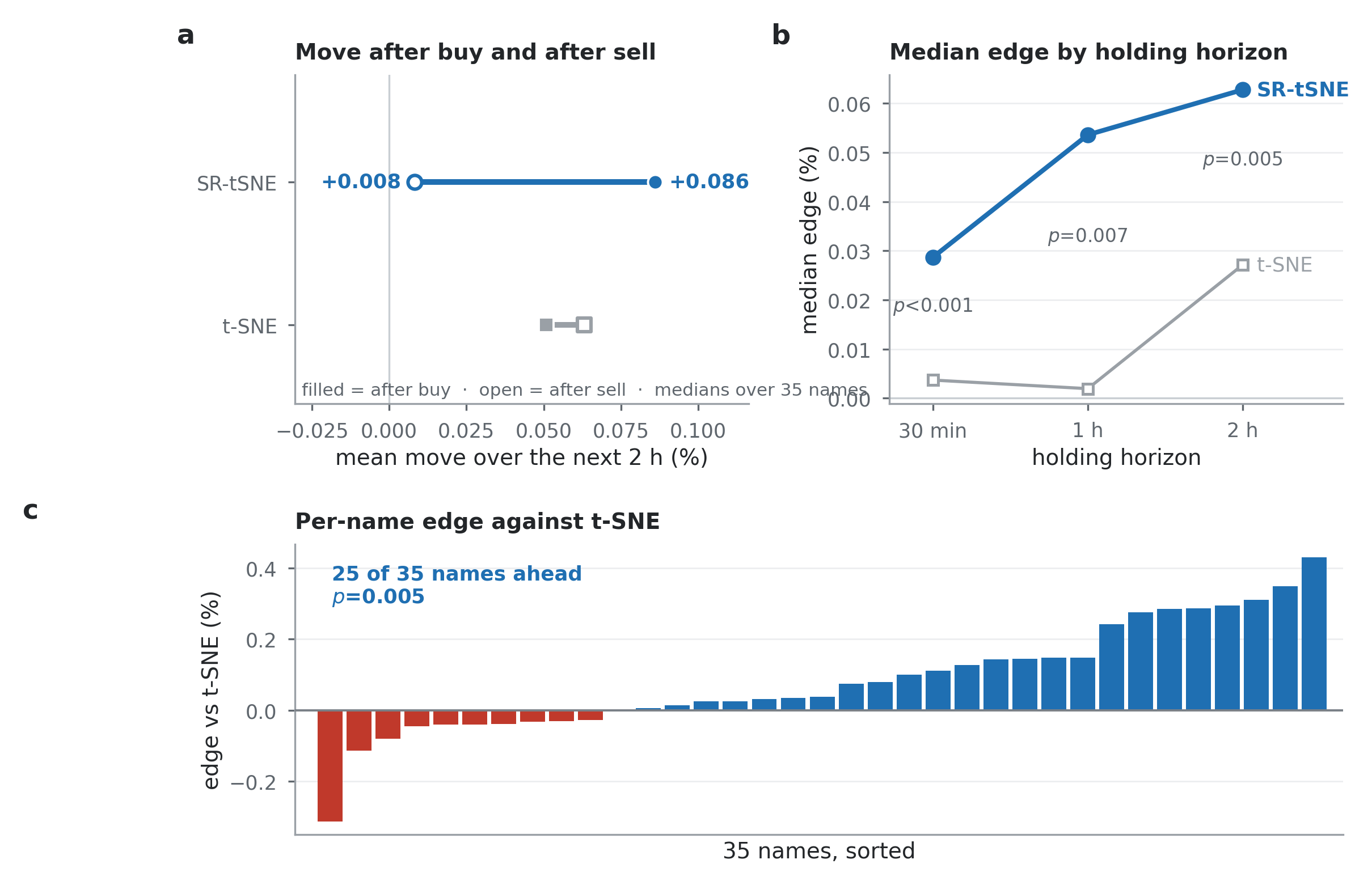}
\caption{What the markers are worth. The \emph{edge} of an embedding on one name
is the mean move of the close over the next $24$ bars after its buy markers
minus the mean move after its sell markers. \textbf{a} gives the two halves
separately, each a median over the $35$ names of Table~\ref{tab:pool35};
\textbf{b} tracks the median edge at three holding horizons, with the paired
Wilcoxon $p$ against t-SNE beside each; \textbf{c} gives the per-name
difference against t-SNE at two hours.}
\label{fig:marker_edge}
\end{figure}

Figure~\ref{fig:sweep_dlp} sweeps neighborhood membership across $k$. SR-tSNE holds the top curve at $k \in \{5, 10\}$ on every benchmark.

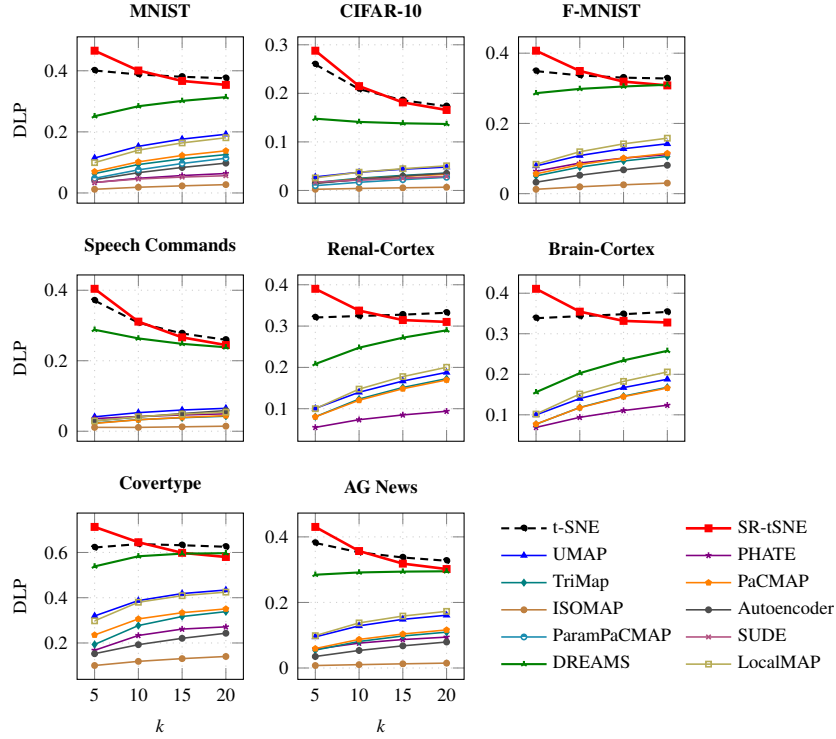
\begin{figure}[!ht]
\centering
\begin{tikzpicture}
\begin{groupplot}[
    group style={group size=3 by 3, horizontal sep=0.72cm, vertical sep=0.95cm},
    scale only axis=true, width=2.20cm, height=2.20cm,
    xmin=0.6, xmax=4.4, xtick={1,2,3,4},
    xticklabels={5,10,15,20},
    ymajorgrids=true, xmajorgrids=true, grid style={dotted, gray!40},
    tick label style={font=\scriptsize}, scaled y ticks=false,
    yticklabel style={/pgf/number format/.cd, fixed, fixed zerofill=false, precision=3, /tikz/.cd},
    label style={font=\scriptsize}, title style={font=\scriptsize\bfseries, yshift=-2pt},
]

\nextgroupplot[title={MNIST}, ylabel={DLP}, xticklabels={}]
\addplot[black, densely dashed, mark=*, mark size=1.0pt, line width=0.8pt] coordinates {(1,0.40083)(2,0.38795)(3,0.38026)(4,0.37565)};
\addplot[red, solid, mark=square*, mark size=1.0pt, line width=1.0pt] coordinates {(1,0.46574)(2,0.40086)(3,0.36674)(4,0.35390)};
\addplot[blue, solid, mark=triangle*, mark size=1.1pt, line width=0.6pt] coordinates {(1,0.11452)(2,0.15263)(3,0.17628)(4,0.19224)};
\addplot[violet, solid, mark=star, mark size=1.2pt, line width=0.6pt] coordinates {(1,0.03408)(2,0.04805)(3,0.05703)(4,0.06369)};
\addplot[teal, solid, mark=diamond*, mark size=1.1pt, line width=0.6pt] coordinates {(1,0.06341)(2,0.09299)(3,0.11166)(4,0.12558)};
\addplot[orange, solid, mark=pentagon*, mark size=1.0pt, line width=0.6pt] coordinates {(1,0.06999)(2,0.10184)(3,0.12278)(4,0.13802)};
\addplot[brown, solid, mark=otimes*, mark size=1.0pt, line width=0.6pt] coordinates {(1,0.01205)(2,0.01873)(3,0.02329)(4,0.02728)};
\addplot[gray!60!black, solid, mark=oplus*, mark size=1.0pt, line width=0.6pt] coordinates {(1,0.04358)(2,0.06642)(3,0.08350)(4,0.09755)};
\addplot[cyan!65!black, solid, mark=halfcircle*, mark size=1.0pt, line width=0.6pt] coordinates {(1,0.04792)(2,0.07566)(3,0.09645)(4,0.11374)};
\addplot[magenta!70!black, solid, mark=x, mark size=1.2pt, line width=0.6pt] coordinates {(1,0.03465)(2,0.04567)(3,0.05199)(4,0.05647)};
\addplot[green!50!black, solid, mark=Mercedes star, mark size=1.2pt, line width=0.8pt] coordinates {(1,0.25130)(2,0.28386)(3,0.30147)(4,0.31375)};
\addplot[yellow!65!black, solid, mark=square, mark size=1.0pt, line width=0.6pt] coordinates {(1,0.09933)(2,0.14008)(3,0.16386)(4,0.18074)};

\nextgroupplot[title={CIFAR-10}, xticklabels={}]
\addplot[black, densely dashed, mark=*, mark size=1.0pt, line width=0.8pt] coordinates {(1,0.25962)(2,0.20872)(3,0.18575)(4,0.17336)};
\addplot[red, solid, mark=square*, mark size=1.0pt, line width=1.0pt] coordinates {(1,0.28782)(2,0.21489)(3,0.18132)(4,0.16575)};
\addplot[blue, solid, mark=triangle*, mark size=1.1pt, line width=0.6pt] coordinates {(1,0.02792)(2,0.03747)(3,0.04375)(4,0.04812)};
\addplot[violet, solid, mark=star, mark size=1.2pt, line width=0.6pt] coordinates {(1,0.01638)(2,0.02177)(3,0.02515)(4,0.02772)};
\addplot[teal, solid, mark=diamond*, mark size=1.1pt, line width=0.6pt] coordinates {(1,0.01638)(2,0.02511)(3,0.03100)(4,0.03578)};
\addplot[orange, solid, mark=pentagon*, mark size=1.0pt, line width=0.6pt] coordinates {(1,0.01537)(2,0.02347)(3,0.02895)(4,0.03328)};
\addplot[brown, solid, mark=otimes*, mark size=1.0pt, line width=0.6pt] coordinates {(1,0.00239)(2,0.00412)(3,0.00545)(4,0.00668)};
\addplot[gray!60!black, solid, mark=oplus*, mark size=1.0pt, line width=0.6pt] coordinates {(1,0.01426)(2,0.02283)(3,0.02957)(4,0.03537)};
\addplot[cyan!65!black, solid, mark=halfcircle*, mark size=1.0pt, line width=0.6pt] coordinates {(1,0.00983)(2,0.01677)(3,0.02223)(4,0.02705)};
\addplot[magenta!70!black, solid, mark=x, mark size=1.2pt, line width=0.6pt] coordinates {(1,0.01693)(2,0.02155)(3,0.02500)(4,0.02878)};
\addplot[green!50!black, solid, mark=Mercedes star, mark size=1.2pt, line width=0.8pt] coordinates {(1,0.14780)(2,0.14116)(3,0.13830)(4,0.13688)};
\addplot[yellow!65!black, solid, mark=square, mark size=1.0pt, line width=0.6pt] coordinates {(1,0.02582)(2,0.03756)(3,0.04483)(4,0.05074)};

\nextgroupplot[title={F-MNIST}, xticklabels={}]
\addplot[black, densely dashed, mark=*, mark size=1.0pt, line width=0.8pt] coordinates {(1,0.34890)(2,0.33675)(3,0.33031)(4,0.32794)};
\addplot[red, solid, mark=square*, mark size=1.0pt, line width=1.0pt] coordinates {(1,0.40706)(2,0.34933)(3,0.31907)(4,0.30824)};
\addplot[blue, solid, mark=triangle*, mark size=1.1pt, line width=0.6pt] coordinates {(1,0.07901)(2,0.10853)(3,0.12796)(4,0.14196)};
\addplot[violet, solid, mark=star, mark size=1.2pt, line width=0.6pt] coordinates {(1,0.06368)(2,0.08717)(3,0.10121)(4,0.11125)};
\addplot[teal, solid, mark=diamond*, mark size=1.1pt, line width=0.6pt] coordinates {(1,0.05054)(2,0.07562)(3,0.09308)(4,0.10621)};
\addplot[orange, solid, mark=pentagon*, mark size=1.0pt, line width=0.6pt] coordinates {(1,0.05637)(2,0.08268)(3,0.10116)(4,0.11477)};
\addplot[brown, solid, mark=otimes*, mark size=1.0pt, line width=0.6pt] coordinates {(1,0.01247)(2,0.01952)(3,0.02525)(4,0.03008)};
\addplot[gray!60!black, solid, mark=oplus*, mark size=1.0pt, line width=0.6pt] coordinates {(1,0.03265)(2,0.05246)(3,0.06785)(4,0.08092)};
\addplot[green!50!black, solid, mark=Mercedes star, mark size=1.2pt, line width=0.8pt] coordinates {(1,0.28616)(2,0.29832)(3,0.30519)(4,0.31003)};
\addplot[yellow!65!black, solid, mark=square, mark size=1.0pt, line width=0.6pt] coordinates {(1,0.08367)(2,0.11932)(3,0.14184)(4,0.15801)};

\nextgroupplot[title={Speech Commands}, ylabel={DLP}, xticklabels={}]
\addplot[black, densely dashed, mark=*, mark size=1.0pt, line width=0.8pt] coordinates {(1,0.37082)(2,0.30784)(3,0.27740)(4,0.25936)};
\addplot[red, solid, mark=square*, mark size=1.0pt, line width=1.0pt] coordinates {(1,0.40410)(2,0.31102)(3,0.26631)(4,0.24452)};
\addplot[blue, solid, mark=triangle*, mark size=1.1pt, line width=0.6pt] coordinates {(1,0.04093)(2,0.05276)(3,0.06019)(4,0.06502)};
\addplot[violet, solid, mark=star, mark size=1.2pt, line width=0.6pt] coordinates {(1,0.03596)(2,0.04180)(3,0.04613)(4,0.04945)};
\addplot[teal, solid, mark=diamond*, mark size=1.1pt, line width=0.6pt] coordinates {(1,0.02330)(2,0.03278)(3,0.03841)(4,0.04259)};
\addplot[orange, solid, mark=pentagon*, mark size=1.0pt, line width=0.6pt] coordinates {(1,0.02220)(2,0.03256)(3,0.03908)(4,0.04415)};
\addplot[brown, solid, mark=otimes*, mark size=1.0pt, line width=0.6pt] coordinates {(1,0.01094)(2,0.01090)(3,0.01255)(4,0.01432)};
\addplot[gray!60!black, solid, mark=oplus*, mark size=1.0pt, line width=0.6pt] coordinates {(1,0.02961)(2,0.04032)(3,0.05004)(4,0.05851)};
\addplot[green!50!black, solid, mark=Mercedes star, mark size=1.2pt, line width=0.8pt] coordinates {(1,0.28794)(2,0.26338)(3,0.24791)(4,0.23819)};
\addplot[yellow!65!black, solid, mark=square, mark size=1.0pt, line width=0.6pt] coordinates {(1,0.02920)(2,0.04124)(3,0.04858)(4,0.05402)};

\nextgroupplot[title={Renal-Cortex}, xticklabels={}]
\addplot[black, densely dashed, mark=*, mark size=1.0pt, line width=0.8pt] coordinates {(1,0.32079)(2,0.32435)(3,0.32764)(4,0.33224)};
\addplot[red, solid, mark=square*, mark size=1.0pt, line width=1.0pt] coordinates {(1,0.39019)(2,0.33747)(3,0.31457)(4,0.31010)};
\addplot[blue, solid, mark=triangle*, mark size=1.1pt, line width=0.6pt] coordinates {(1,0.10069)(2,0.13982)(3,0.16694)(4,0.18794)};
\addplot[violet, solid, mark=star, mark size=1.2pt, line width=0.6pt] coordinates {(1,0.05476)(2,0.07344)(3,0.08494)(4,0.09383)};
\addplot[teal, solid, mark=diamond*, mark size=1.1pt, line width=0.6pt] coordinates {(1,0.08052)(2,0.12334)(3,0.15138)(4,0.17262)};
\addplot[orange, solid, mark=pentagon*, mark size=1.0pt, line width=0.6pt] coordinates {(1,0.08028)(2,0.12051)(3,0.14835)(4,0.16961)};
\addplot[green!50!black, solid, mark=Mercedes star, mark size=1.2pt, line width=0.8pt] coordinates {(1,0.20820)(2,0.24767)(3,0.27193)(4,0.28974)};
\addplot[yellow!65!black, solid, mark=square, mark size=1.0pt, line width=0.6pt] coordinates {(1,0.10008)(2,0.14741)(3,0.17774)(4,0.20028)};

\nextgroupplot[title={Brain-Cortex}, xticklabels={}]
\addplot[black, densely dashed, mark=*, mark size=1.0pt, line width=0.8pt] coordinates {(1,0.33838)(2,0.34339)(3,0.34823)(4,0.35384)};
\addplot[red, solid, mark=square*, mark size=1.0pt, line width=1.0pt] coordinates {(1,0.41085)(2,0.35452)(3,0.33172)(4,0.32776)};
\addplot[blue, solid, mark=triangle*, mark size=1.1pt, line width=0.6pt] coordinates {(1,0.09982)(2,0.14009)(3,0.16717)(4,0.18772)};
\addplot[violet, solid, mark=star, mark size=1.2pt, line width=0.6pt] coordinates {(1,0.06895)(2,0.09378)(3,0.11055)(4,0.12364)};
\addplot[teal, solid, mark=diamond*, mark size=1.1pt, line width=0.6pt] coordinates {(1,0.07675)(2,0.11815)(3,0.14608)(4,0.16780)};
\addplot[orange, solid, mark=pentagon*, mark size=1.0pt, line width=0.6pt] coordinates {(1,0.07715)(2,0.11756)(3,0.14489)(4,0.16602)};
\addplot[green!50!black, solid, mark=Mercedes star, mark size=1.2pt, line width=0.8pt] coordinates {(1,0.15580)(2,0.20279)(3,0.23435)(4,0.25784)};
\addplot[yellow!65!black, solid, mark=square, mark size=1.0pt, line width=0.6pt] coordinates {(1,0.10182)(2,0.15132)(3,0.18247)(4,0.20572)};

\nextgroupplot[title={Covertype}, ylabel={DLP}, xlabel={$k$}]
\addplot[black, densely dashed, mark=*, mark size=1.0pt, line width=0.8pt] coordinates {(1,0.62197)(2,0.63693)(3,0.63193)(4,0.62510)};
\addplot[red, solid, mark=square*, mark size=1.0pt, line width=1.0pt] coordinates {(1,0.71290)(2,0.64518)(3,0.59833)(4,0.58007)};
\addplot[blue, solid, mark=triangle*, mark size=1.1pt, line width=0.6pt] coordinates {(1,0.31940)(2,0.38736)(3,0.41837)(4,0.43410)};
\addplot[violet, solid, mark=star, mark size=1.2pt, line width=0.6pt] coordinates {(1,0.16756)(2,0.23320)(3,0.26142)(4,0.27130)};
\addplot[teal, solid, mark=diamond*, mark size=1.1pt, line width=0.6pt] coordinates {(1,0.19351)(2,0.27701)(3,0.31692)(4,0.33841)};
\addplot[orange, solid, mark=pentagon*, mark size=1.0pt, line width=0.6pt] coordinates {(1,0.23543)(2,0.30627)(3,0.33382)(4,0.35066)};
\addplot[brown, solid, mark=otimes*, mark size=1.0pt, line width=0.6pt] coordinates {(1,0.10049)(2,0.11878)(3,0.13088)(4,0.14032)};
\addplot[gray!60!black, solid, mark=oplus*, mark size=1.0pt, line width=0.6pt] coordinates {(1,0.15272)(2,0.19250)(3,0.22060)(4,0.24306)};
\addplot[green!50!black, solid, mark=Mercedes star, mark size=1.2pt, line width=0.8pt] coordinates {(1,0.53839)(2,0.58316)(3,0.59484)(4,0.59684)};
\addplot[yellow!65!black, solid, mark=square, mark size=1.0pt, line width=0.6pt] coordinates {(1,0.29778)(2,0.37998)(3,0.40892)(4,0.42476)};

\nextgroupplot[title={AG News}, xlabel={$k$}]
\addplot[black, densely dashed, mark=*, mark size=1.0pt, line width=0.8pt] coordinates {(1,0.38146)(2,0.35290)(3,0.33710)(4,0.32718)};
\addplot[red, solid, mark=square*, mark size=1.0pt, line width=1.0pt] coordinates {(1,0.43036)(2,0.35684)(3,0.31842)(4,0.30188)};
\addplot[blue, solid, mark=triangle*, mark size=1.1pt, line width=0.6pt] coordinates {(1,0.09486)(2,0.12861)(3,0.14845)(4,0.16098)};
\addplot[violet, solid, mark=star, mark size=1.2pt, line width=0.6pt] coordinates {(1,0.05748)(2,0.07595)(3,0.08680)(4,0.09420)};
\addplot[teal, solid, mark=diamond*, mark size=1.1pt, line width=0.6pt] coordinates {(1,0.05459)(2,0.08098)(3,0.09769)(4,0.10934)};
\addplot[orange, solid, mark=pentagon*, mark size=1.0pt, line width=0.6pt] coordinates {(1,0.05927)(2,0.08724)(3,0.10372)(4,0.11618)};
\addplot[brown, solid, mark=otimes*, mark size=1.0pt, line width=0.6pt] coordinates {(1,0.00756)(2,0.01021)(3,0.01273)(4,0.01489)};
\addplot[gray!60!black, solid, mark=oplus*, mark size=1.0pt, line width=0.6pt] coordinates {(1,0.03512)(2,0.05348)(3,0.06762)(4,0.07927)};
\addplot[green!50!black, solid, mark=Mercedes star, mark size=1.2pt, line width=0.8pt] coordinates {(1,0.28454)(2,0.29150)(3,0.29401)(4,0.29506)};
\addplot[yellow!65!black, solid, mark=square, mark size=1.0pt, line width=0.6pt] coordinates {(1,0.09868)(2,0.13756)(3,0.15843)(4,0.17280)};

\nextgroupplot[hide axis, ymin=0, ymax=1,
    legend style={draw=none, fill=none, font=\scriptsize, at={(0.5,0.5)}, anchor=center, xshift=9mm,
                  legend cell align=left, legend columns=2, /tikz/every even column/.append style={column sep=3pt}}]
\addlegendimage{black, densely dashed, mark=*, mark size=1.0pt, line width=0.8pt}\addlegendentry{t-SNE}
\addlegendimage{red, solid, mark=square*, mark size=1.0pt, line width=1.0pt}\addlegendentry{SR-tSNE}
\addlegendimage{blue, solid, mark=triangle*, mark size=1.1pt, line width=0.6pt}\addlegendentry{UMAP}
\addlegendimage{violet, solid, mark=star, mark size=1.2pt, line width=0.6pt}\addlegendentry{PHATE}
\addlegendimage{teal, solid, mark=diamond*, mark size=1.1pt, line width=0.6pt}\addlegendentry{TriMap}
\addlegendimage{orange, solid, mark=pentagon*, mark size=1.0pt, line width=0.6pt}\addlegendentry{PaCMAP}
\addlegendimage{brown, solid, mark=otimes*, mark size=1.0pt, line width=0.6pt}\addlegendentry{ISOMAP}
\addlegendimage{gray!60!black, solid, mark=oplus*, mark size=1.0pt, line width=0.6pt}\addlegendentry{Autoencoder}
\addlegendimage{cyan!65!black, solid, mark=halfcircle*, mark size=1.0pt, line width=0.6pt}\addlegendentry{ParamPaCMAP}
\addlegendimage{magenta!70!black, solid, mark=x, mark size=1.2pt, line width=0.6pt}\addlegendentry{SUDE}
\addlegendimage{green!50!black, solid, mark=Mercedes star, mark size=1.2pt, line width=0.8pt}\addlegendentry{DREAMS}
\addlegendimage{yellow!65!black, solid, mark=square, mark size=1.0pt, line width=0.6pt}\addlegendentry{LocalMAP}
\addplot[draw=none, forget plot] coordinates {(1,0)};

\end{groupplot}
\end{tikzpicture}
\caption{DLP versus neighborhood size $k$ on the same eight benchmarks, with the
legend at lower right. Each panel carries the methods for which a sweep across $k$
was stored; the four methods added at $k{=}10$ in Table~\ref{tab:dlp_main} are
scored at that single size and are not swept here. SR-tSNE holds the top curve at
$k \in \{5, 10\}$ on every benchmark.}
\label{fig:sweep_dlp}
\end{figure}

\subsection{Per-name trading results}

Tables~\ref{tab:pool35} and~\ref{tab:pool35_risk} give the per-name results behind
Section~\ref{sec:trading}: net return at $0$ and $2$\,bps, and the annualized Sharpe ratio, maximum drawdown and number
of round trips at $2$\,bps, for the $35$ equities in seven groups of five. Every arm runs through the identical
DBSCAN$+$K scorer of \ref{app:scorer} with the same per-stock marker percentile, so only the embedding differs between
columns. Figure~\ref{fig:pool30_five} summarizes both tables at $2$\,bps.

\begin{table*}[!t]
\centering
\caption{Returns (\%) on the 35 names, five per group across seven groups.
Bold marks the best of the six per name.}
\label{tab:pool35}
\renewcommand{\arraystretch}{0.92}
\setlength{\tabcolsep}{1.6pt}
\adjustbox{max width=\textwidth}{\begin{tabular}{ll cccccc cccccc}
\toprule
& & \multicolumn{6}{c}{0 bps} & \multicolumn{6}{c}{2 bps} \\
\cmidrule(lr){3-8} \cmidrule(lr){9-14}
Group & Stock & t-SNE & UMAP & TriMap & LocalMAP & DREAMS & SR-tSNE & t-SNE & UMAP & TriMap & LocalMAP & DREAMS & SR-tSNE \\
\midrule
 Technology & FICO & +29.6 & +45.5 & +38.8 & \textbf{+62.2} & +39.0 & +44.0 & +26.8 & +33.2 & +25.5 & \textbf{+56.6} & +34.1 & +37.3 \\
  & PAYC & +49.6 & +55.5 & +60.4 & +47.8 & \textbf{+62.4} & +56.6 & +44.6 & +43.6 & +42.0 & +33.9 & \textbf{+54.8} & +52.8 \\
  & DLB & $-$10.7 & $-$2.8 & +3.9 & \textbf{+7.2} & $-$5.3 & $-$8.9 & $-$12.1 & $-$9.6 & $-$6.7 & \textbf{+1.5} & $-$9.0 & $-$10.3 \\
  & RMBS & $-$2.8 & $-$3.6 & +16.6 & +34.7 & +3.7 & \textbf{+38.0} & $-$6.1 & $-$10.1 & +6.1 & +24.8 & $-$1.0 & \textbf{+32.8} \\
  & IDCC & $-$6.2 & +13.3 & \textbf{+33.9} & +1.3 & +19.9 & +29.8 & $-$7.5 & +9.1 & +22.2 & $-$5.1 & +14.6 & \textbf{+28.0} \\
 Health care & UNH & +25.4 & +40.5 & +35.1 & \textbf{+42.8} & +42.7 & +34.2 & +23.6 & +35.7 & +29.9 & +34.4 & \textbf{+39.8} & +31.5 \\
  & MRNA & +0.9 & +0.9 & +17.9 & +27.8 & \textbf{+41.6} & +3.1 & $-$3.8 & $-$1.5 & +12.7 & +22.1 & \textbf{+38.4} & +1.1 \\
  & ELAN & +11.8 & $-$8.7 & $-$12.0 & +14.5 & $-$6.0 & \textbf{+34.9} & +6.4 & $-$11.1 & $-$18.0 & +4.7 & $-$11.5 & \textbf{+32.1} \\
  & BMRN & +9.8 & \textbf{+30.1} & +26.7 & +17.9 & +9.4 & +18.2 & +4.8 & \textbf{+21.6} & +10.7 & +6.1 & +3.7 & +13.8 \\
  & XRAY & +15.5 & +15.2 & +15.6 & +12.5 & +20.6 & \textbf{+22.7} & +13.4 & +10.4 & +2.8 & +9.4 & +16.3 & \textbf{+17.9} \\
 Financials & SCHW & +1.0 & +4.1 & \textbf{+7.6} & +4.5 & $-$1.7 & +5.0 & $-$1.8 & +1.2 & \textbf{+2.2} & +0.3 & $-$4.2 & +1.8 \\
  & TROW & +18.9 & +23.0 & +21.0 & +15.2 & +20.8 & \textbf{+23.1} & +13.0 & +18.1 & +12.9 & +7.6 & +15.4 & \textbf{+18.6} \\
  & AMP & +10.1 & +6.3 & +12.8 & +1.8 & +8.8 & \textbf{+16.3} & +8.4 & +2.6 & +0.1 & $-$3.9 & +5.6 & \textbf{+12.8} \\
  & IVZ & +2.3 & +14.6 & +5.9 & +10.6 & +11.9 & \textbf{+20.0} & $-$2.2 & +8.6 & $-$0.0 & +1.7 & +8.7 & \textbf{+14.7} \\
  & MCY & +20.5 & +29.5 & +30.9 & +26.1 & \textbf{+31.8} & +26.2 & +15.6 & +21.2 & +19.9 & +18.3 & \textbf{+27.3} & +21.3 \\
 Industrials & UBER & +5.0 & +9.1 & $-$5.7 & +9.2 & $-$0.2 & \textbf{+16.0} & +0.2 & +1.6 & $-$9.0 & +2.2 & $-$2.8 & \textbf{+11.0} \\
  & RTX & +1.7 & $-$0.5 & $-$1.9 & \textbf{+6.5} & +3.6 & +6.0 & $-$0.8 & $-$4.6 & $-$5.6 & +0.0 & +1.8 & \textbf{+4.2} \\
  & DD & +4.7 & +2.2 & +0.2 & $-$1.5 & $-$2.6 & \textbf{+11.2} & +2.3 & $-$0.8 & $-$3.7 & $-$4.2 & $-$6.7 & \textbf{+8.3} \\
  & ATI & +2.1 & +0.5 & +10.2 & +17.3 & +3.5 & \textbf{+25.7} & $-$1.2 & $-$5.4 & $-$2.5 & +11.3 & $-$2.3 & \textbf{+22.4} \\
  & ARCB & +16.5 & +27.2 & \textbf{+69.1} & +15.7 & +39.6 & +49.4 & +14.9 & +22.3 & \textbf{+59.4} & +11.1 & +35.5 & +45.6 \\
 Consumer & ABNB & +2.4 & +1.7 & $-$1.2 & +3.6 & +9.4 & \textbf{+12.9} & $-$0.1 & $-$2.5 & $-$6.0 & +0.1 & +6.8 & \textbf{+9.8} \\
  & DHI & +12.4 & +8.9 & \textbf{+19.2} & +2.5 & +6.8 & +14.4 & +6.8 & +0.9 & +8.1 & $-$3.3 & +1.8 & \textbf{+11.4} \\
  & GIS & +9.2 & +6.5 & +5.4 & +0.6 & +12.4 & \textbf{+15.4} & +7.4 & $-$1.6 & +1.9 & $-$2.3 & +5.6 & \textbf{+13.4} \\
  & HRB & +35.9 & +19.0 & +18.5 & +17.5 & +17.6 & \textbf{+38.6} & +31.1 & +12.7 & +13.2 & +11.1 & +14.9 & \textbf{+36.5} \\
  & CAVA & $-$1.3 & $-$9.7 & $-$17.5 & $-$6.8 & $-$2.1 & \textbf{+14.0} & $-$2.9 & $-$13.9 & $-$26.4 & $-$9.9 & $-$8.2 & \textbf{+11.2} \\
 Energy, REITs \& utilities & OKE & +0.3 & +10.7 & \textbf{+15.1} & +9.7 & +6.7 & +14.1 & $-$3.5 & +7.2 & +8.7 & +3.6 & +1.9 & \textbf{+10.7} \\
  & CCI & +0.8 & $-$3.7 & $-$0.5 & $-$3.4 & \textbf{+14.6} & +3.8 & $-$3.0 & $-$6.5 & $-$9.6 & $-$7.9 & \textbf{+12.8} & $-$1.0 \\
  & PBF & +36.9 & +44.5 & +23.9 & \textbf{+65.2} & +61.3 & +42.8 & +31.5 & +39.8 & +9.4 & \textbf{+60.3} & +56.5 & +38.2 \\
  & EVRG & +5.6 & +0.1 & \textbf{+13.0} & +7.4 & +6.3 & +6.8 & +0.6 & $-$3.4 & \textbf{+7.0} & +1.6 & +3.6 & +3.6 \\
  & OHI & +6.5 & +11.1 & +4.6 & \textbf{+14.9} & +5.0 & +10.4 & +1.8 & +6.8 & $-$2.5 & \textbf{+11.4} & +1.2 & +7.7 \\
 Large caps & INTC & +96.9 & +67.9 & +47.9 & +63.3 & +82.5 & \textbf{+104.8} & +84.7 & +56.6 & +32.4 & +53.2 & +75.0 & \textbf{+94.6} \\
  & UAL & +8.1 & +11.9 & +21.2 & +7.2 & +11.2 & \textbf{+25.4} & +3.6 & +4.1 & +6.0 & +4.3 & +9.2 & \textbf{+22.8} \\
  & AVGO & +14.2 & +14.3 & +7.0 & +6.5 & +1.6 & \textbf{+19.2} & +8.0 & +9.8 & +0.6 & +3.3 & $-$0.4 & \textbf{+17.0} \\
  & GOOGL & +9.4 & +0.6 & +6.1 & \textbf{+22.6} & +9.1 & +11.7 & +4.7 & $-$4.0 & $-$0.4 & \textbf{+18.9} & +6.5 & +8.9 \\
  & NFLX & $-$0.8 & $-$2.8 & $-$13.6 & $-$3.3 & $-$8.8 & \textbf{+2.5} & $-$4.2 & $-$9.6 & $-$23.5 & $-$8.0 & $-$11.7 & \textbf{$-$2.6} \\
\midrule
\textit{Median} & & +8.1 & +9.1 & +13.0 & +10.6 & +9.4 & \textbf{+18.2} & +3.6 & +2.6 & +2.8 & +4.3 & +5.6 & \textbf{+13.8} \\
\bottomrule
\end{tabular}}
\end{table*}

\begin{table*}[!t]
\centering
\caption{Risk and activity for the same 35 names at $2$\,bps: annualized Sharpe ratio,
maximum drawdown (\%) and number of round trips. Bold marks the best of the six per name;
for round trips no method is marked.}
\label{tab:pool35_risk}
\setlength{\tabcolsep}{0.7pt}
\adjustbox{max width=\textwidth}{\begin{tabular}{ll cccccc cccccc cccccc}
\toprule
& & \multicolumn{6}{c}{Sharpe} & \multicolumn{6}{c}{Max drawdown (\%)} & \multicolumn{6}{c}{Round trips} \\
\cmidrule(lr){3-8} \cmidrule(lr){9-14} \cmidrule(lr){15-20}
Group & Stock & t-SNE & UMAP & TriMap & LocalMAP & DREAMS & SR-tSNE & t-SNE & UMAP & TriMap & LocalMAP & DREAMS & SR-tSNE & t-SNE & UMAP & TriMap & LocalMAP & DREAMS & SR-tSNE \\
\midrule
 Technology & FICO & +2.90 & +3.19 & +2.52 & \textbf{+5.14} & +3.05 & +3.42 & $-$10.28 & $-$14.97 & $-$14.69 & \textbf{$-$7.54} & $-$15.71 & $-$11.42 & 216 & 880 & 1008 & 352 & 357 & 478 \\
  & PAYC & +4.26 & +4.12 & +4.25 & +3.49 & +5.13 & \textbf{+5.16} & $-$13.62 & \textbf{$-$7.87} & $-$15.87 & $-$11.14 & $-$8.74 & $-$9.49 & 337 & 795 & 1214 & 990 & 483 & 246 \\
  & DLB & $-$1.60 & $-$1.29 & $-$0.82 & \textbf{+0.39} & $-$1.16 & $-$1.34 & $-$16.41 & $-$15.66 & $-$12.90 & \textbf{$-$8.75} & $-$15.59 & $-$17.01 & 129 & 718 & 1007 & 546 & 390 & 147 \\
  & RMBS & $-$0.05 & $-$0.26 & +0.72 & +1.92 & +0.32 & \textbf{+2.08} & $-$32.79 & $-$33.84 & $-$31.67 & \textbf{$-$24.45} & $-$31.17 & $-$30.38 & 343 & 658 & 942 & 761 & 465 & 381 \\
  & IDCC & $-$0.81 & +1.05 & +2.01 & $-$0.43 & +1.47 & \textbf{+2.61} & $-$15.21 & $-$11.61 & $-$17.91 & $-$16.83 & $-$13.32 & \textbf{$-$10.39} & 128 & 380 & 543 & 649 & 449 & 138 \\
 Health care & UNH & +4.52 & +4.78 & +4.18 & +4.59 & \textbf{+5.26} & +4.22 & $-$3.28 & $-$7.06 & $-$3.56 & $-$3.34 & $-$3.59 & \textbf{$-$3.05} & 151 & 341 & 388 & 609 & 205 & 161 \\
  & MRNA & $-$0.10 & +0.17 & +1.22 & +1.74 & \textbf{+2.85} & +0.35 & $-$31.22 & $-$25.84 & \textbf{$-$16.78} & $-$20.65 & $-$19.11 & $-$24.45 & 478 & 242 & 458 & 454 & 224 & 193 \\
  & ELAN & +0.97 & $-$1.31 & $-$2.54 & +0.77 & $-$1.66 & \textbf{+3.56} & $-$21.52 & $-$23.52 & $-$22.74 & $-$15.07 & $-$23.11 & \textbf{$-$10.93} & 497 & 223 & 696 & 299 & 191 & 209 \\
  & BMRN & +0.92 & \textbf{+4.30} & +1.97 & +1.18 & +0.76 & +2.34 & $-$9.17 & \textbf{$-$6.50} & $-$9.29 & $-$7.40 & $-$9.83 & $-$7.93 & 386 & 675 & 1354 & 1022 & 542 & 366 \\
  & XRAY & +1.85 & +1.42 & +0.51 & +1.36 & +1.96 & \textbf{+2.32} & $-$13.53 & $-$12.79 & $-$14.05 & \textbf{$-$10.16} & $-$10.96 & $-$11.31 & 180 & 429 & 1170 & 280 & 360 & 400 \\
 Financials & SCHW & $-$0.36 & +0.35 & \textbf{+0.57} & +0.16 & $-$0.86 & +0.53 & $-$13.21 & $-$10.95 & $-$16.59 & $-$11.69 & $-$13.07 & \textbf{$-$8.31} & 282 & 287 & 517 & 415 & 265 & 318 \\
  & TROW & +3.00 & +4.05 & +2.86 & +1.77 & +3.39 & \textbf{+4.26} & $-$4.08 & $-$5.90 & $-$6.72 & $-$5.58 & $-$3.92 & \textbf{$-$3.63} & 499 & 254 & 388 & 482 & 458 & 318 \\
  & AMP & +1.72 & +0.67 & +0.12 & $-$0.80 & +1.34 & \textbf{+2.94} & $-$7.88 & $-$11.19 & $-$10.94 & $-$8.44 & $-$6.47 & \textbf{$-$5.15} & 155 & 343 & 1198 & 420 & 262 & 305 \\
  & IVZ & $-$0.26 & +1.47 & +0.11 & +0.41 & +1.56 & \textbf{+2.30} & $-$11.26 & $-$13.43 & $-$14.63 & $-$14.00 & $-$9.51 & \textbf{$-$6.24} & 456 & 350 & 556 & 817 & 281 & 453 \\
  & MCY & +2.79 & +3.93 & +3.57 & +3.47 & \textbf{+4.56} & +3.54 & $-$5.98 & $-$5.57 & $-$7.76 & $-$8.65 & \textbf{$-$4.33} & $-$6.08 & 413 & 664 & 821 & 641 & 341 & 397 \\
 Industrials & UBER & +0.15 & +0.38 & $-$1.30 & +0.48 & $-$0.32 & \textbf{+1.88} & \textbf{$-$7.66} & $-$10.76 & $-$14.53 & $-$8.42 & $-$13.69 & $-$11.11 & 462 & 490 & 356 & 668 & 262 & 440 \\
  & RTX & $-$0.09 & $-$1.00 & $-$1.21 & +0.10 & +0.54 & \textbf{+1.06} & $-$10.14 & $-$12.06 & $-$14.08 & $-$15.94 & $-$8.34 & \textbf{$-$7.43} & 220 & 416 & 360 & 630 & 174 & 174 \\
  & DD & +0.52 & +0.01 & $-$0.49 & $-$0.56 & $-$1.11 & \textbf{+1.42} & $-$13.06 & \textbf{$-$10.89} & $-$12.46 & $-$16.55 & $-$18.31 & $-$12.90 & 227 & 291 & 402 & 277 & 254 & 265 \\
  & ATI & +0.03 & $-$0.57 & $-$0.17 & +1.60 & $-$0.15 & \textbf{+2.93} & $-$12.77 & $-$13.78 & $-$11.89 & $-$11.55 & $-$16.55 & \textbf{$-$8.32} & 329 & 254 & 1220 & 524 & 580 & 266 \\
  & ARCB & +2.07 & +2.64 & \textbf{+5.74} & +1.33 & +3.99 & +5.20 & $-$16.55 & $-$17.40 & $-$9.87 & $-$12.63 & $-$10.55 & \textbf{$-$7.47} & 132 & 389 & 590 & 408 & 293 & 257 \\
 Consumer & ABNB & +0.11 & $-$0.29 & $-$0.84 & +0.13 & +1.21 & \textbf{+1.68} & \textbf{$-$9.41} & $-$16.44 & $-$13.87 & $-$11.79 & $-$10.12 & $-$10.22 & 244 & 270 & 501 & 316 & 240 & 276 \\
  & DHI & +1.06 & +0.29 & +1.42 & $-$0.41 & +0.43 & \textbf{+1.62} & $-$9.05 & \textbf{$-$8.30} & $-$10.99 & $-$12.12 & $-$9.79 & $-$13.83 & 513 & 757 & 402 & 469 & 474 & 265 \\
  & GIS & +1.41 & $-$0.26 & +0.47 & $-$0.33 & +1.17 & \textbf{+2.44} & $-$8.22 & $-$11.66 & $-$11.08 & $-$8.87 & $-$9.85 & \textbf{$-$4.80} & 167 & 639 & 338 & 290 & 628 & 176 \\
  & HRB & +3.32 & +1.64 & +1.65 & +1.35 & +1.77 & \textbf{+4.21} & $-$10.53 & $-$11.23 & $-$10.42 & $-$14.67 & $-$10.45 & \textbf{$-$7.31} & 362 & 542 & 460 & 564 & 230 & 156 \\
  & CAVA & $-$0.14 & $-$1.45 & $-$3.13 & $-$0.94 & $-$0.76 & \textbf{+1.29} & \textbf{$-$16.79} & $-$26.21 & $-$34.18 & $-$18.44 & $-$28.67 & $-$20.43 & 171 & 277 & 1139 & 343 & 647 & 249 \\
 Energy, REITs \& utilities & OKE & $-$0.59 & +1.46 & +1.79 & +0.83 & +0.51 & \textbf{+2.11} & $-$10.01 & $-$9.38 & \textbf{$-$7.82} & $-$10.78 & $-$11.11 & $-$8.18 & 308 & 266 & 473 & 567 & 456 & 301 \\
  & CCI & $-$0.43 & $-$1.19 & $-$1.78 & $-$1.39 & \textbf{+2.24} & $-$0.04 & $-$16.38 & $-$16.16 & $-$13.70 & $-$20.95 & \textbf{$-$11.35} & $-$12.06 & 376 & 302 & 895 & 476 & 159 & 467 \\
  & PBF & +2.63 & +3.25 & +1.05 & \textbf{+4.84} & +3.90 & +2.95 & $-$16.37 & $-$19.64 & $-$23.06 & \textbf{$-$12.10} & $-$14.93 & $-$16.12 & 401 & 331 & 1246 & 303 & 302 & 327 \\
  & EVRG & +0.25 & $-$0.91 & \textbf{+2.05} & +0.57 & +1.06 & +1.08 & $-$6.17 & $-$7.68 & $-$5.35 & $-$4.55 & \textbf{$-$4.35} & $-$5.93 & 481 & 242 & 392 & 541 & 252 & 310 \\
  & OHI & +0.58 & +1.85 & $-$0.61 & \textbf{+2.98} & +0.41 & +1.96 & $-$6.55 & $-$7.07 & $-$10.97 & $-$6.76 & $-$10.65 & \textbf{$-$5.70} & 378 & 389 & 343 & 309 & 313 & 252 \\
 Large caps & INTC & +3.61 & +2.69 & +2.27 & +3.53 & +3.29 & \textbf{+3.71} & $-$14.74 & $-$17.66 & $-$17.06 & \textbf{$-$13.60} & $-$15.51 & $-$17.33 & 640 & 294 & 1106 & 638 & 420 & 510 \\
  & UAL & +0.63 & +0.62 & +0.83 & +0.64 & +1.25 & \textbf{+2.46} & \textbf{$-$11.88} & $-$15.78 & $-$13.37 & $-$19.03 & $-$15.90 & $-$12.94 & 250 & 725 & 1340 & 274 & 178 & 213 \\
  & AVGO & +1.16 & +1.48 & +0.28 & +0.54 & +0.19 & \textbf{+2.34} & $-$11.03 & \textbf{$-$7.02} & $-$22.23 & $-$16.73 & $-$22.44 & $-$8.19 & 548 & 400 & 617 & 299 & 199 & 185 \\
  & GOOGL & +1.01 & $-$0.62 & +0.09 & \textbf{+2.85} & +1.25 & +1.37 & \textbf{$-$7.45} & $-$12.69 & $-$15.93 & $-$9.06 & $-$12.87 & $-$16.26 & 172 & 463 & 409 & 308 & 246 & 255 \\
  & NFLX & $-$0.72 & $-$2.02 & $-$3.47 & $-$1.65 & $-$1.52 & \textbf{$-$0.40} & $-$16.14 & $-$17.34 & $-$30.71 & \textbf{$-$16.10} & $-$22.38 & $-$16.25 & 350 & 262 & 1215 & 500 & 323 & 510 \\
\midrule
\textit{Median} & & +0.63 & +0.62 & +0.57 & +0.77 & +1.21 & \textbf{+2.32} & $-$11.26 & $-$12.06 & $-$13.87 & $-$11.79 & $-$11.35 & \textbf{$-$10.22} & 337 & 380 & 590 & 476 & 302 & 266 \\
\bottomrule
\end{tabular}}
\end{table*}

\subsection{The DBSCAN$+$K Marker Scorer}
\label{app:scorer}

The DBSCAN$+$K marker scorer is set out here.
DBSCAN~\cite{ester1996density} is run on the two-dimensional
embedding, and the points it labels as noise are the candidates. Each candidate
$i$ carries two components. The geometric one is
$g_i = d_i / \bar{\delta}_{c(i)}$, where $c(i)$ is the cluster nearest to $i$
and $d_i$ the distance from $i$ to that cluster. The scale $\bar{\delta}_c$ is
the mean pairwise distance inside cluster $c$, estimated on a 50-point sample. A
point sitting far from a tight cluster scores high.

The order component reads rank agreement. Let $\rho_i$ be the Spearman
correlation~\cite{spearman1904} between the ranks of the ten high-dimensional
nearest neighbors of $i$ and the ranks of those same points by embedding
distance. A candidate whose neighborhood survives the projection has $\rho_i$
near one and is unlikely to be a genuine anomaly, so the component is
$1 - \rho_i$. Both are min--max normalized across the candidate set and added
with equal weight,
\begin{equation}
  s_i \;=\; \widetilde{g}_i \;+\; \widetilde{(1 - \rho_i)} ,
\end{equation}
where the tilde denotes that normalization. The highest-scoring candidates in
each direction are retained as markers, long giving a buy and short a sell. We
refer to this scorer as DBSCAN+K.

\subsection{Held-out Validation}
\label{app:heldout}
The configuration rule and the gains in order hold on held-out data. We
add the core Human Lung Cell Atlas~\cite{sikkema2023integrated}, a tissue absent
from the eight benchmarks, subsampled to $60\,000$ cells. Twenty donors are held
out as queries. Gene
selection, scaling and the principal components are fitted on reference cells
only. Queries are projected with the fitted transform. The corpus then runs the
pipeline of Section~\ref{sec:setup} unchanged. At $n = 60\,000$ the size rule
selects $\lambda{=}50$.

SR-tSNE leads on the held-out corpus as well.
Table~\ref{tab:heldout_new} reports the outcome. SR-tSNE is ahead of t-SNE and
UMAP on O-DLP, DLP and nDCG on the query cells, and ahead of t-SNE by $+30\%$ in O-DLP.

\begin{table}[!ht]
\centering
\caption{A held-out dataset: the Human Lung Cell Atlas core, split by donor into
$48\,000$ reference and $12\,000$ query cells, scored on the query cells. The
size rule of Section~\ref{sec:setup} selects $\lambda{=}50$. Bold marks the
best method per column.}
\label{tab:heldout_new}
\adjustbox{max width=\textwidth}{\begin{tabular}{lccc}
\toprule
Method & O-DLP$_{10}$ & DLP$_{10}$ & nDCG$_{10}$ \\
\midrule
UMAP                     & .0183 & .1636 & .1718 \\
t-SNE                    & .0654 & .3754 & .4664 \\
SR-tSNE ($\lambda{=}0$)  & .0674 & .3748 & .4684 \\
SR-tSNE (ours)           & \textbf{.0847} & \textbf{.3931} & \textbf{.5447} \\
\bottomrule
\end{tabular}}
\end{table}

\subsection{Implementation Details}
Table~\ref{tab:notation} collects the notation of the smooth rank loss, its optimization and the analysis of
Section~\ref{sec:theory}, with default values in parentheses. The optimizer of Section~\ref{sec:implementation} is SGD
with Polyak momentum,
\begin{align}
    v^{(t+1)} &\;=\; \mu \cdot v^{(t)} \,+\,
                     \nabla_Y \mathcal{L}_{\text{total}}\bigl(Y^{(t)}\bigr), \\
    Y^{(t+1)} &\;=\; Y^{(t)} \,-\, \eta \cdot v^{(t+1)},
\end{align}
with $\eta = 10$, $\mu = 0.8$ and $T = 1000$.

\label{app:impl}
For large $n$ we store $P$ sparsely, keeping the top $k_P{=}90$ entries
per row. The partition function $Z$ is accumulated in row chunks of
$c{=}10{,}000$ with gradient checkpointing, which cuts peak memory from
$\mathcal{O}(n^2)$ to $\mathcal{O}(c\,n)$; the refinement stage remains
$\mathcal{O}(n^2)$ in time. All gradients come from PyTorch autograd. Every
$T_r = 50$ iterations we recompute each point's intruder set with the same
chunked scan, skipping the point itself and its high-dimensional neighbor
list. We use $m = 10$ intruders throughout.

\begin{table}[h]
\centering
\caption{Notation for the smooth rank loss, its optimization and the analysis
of Section~\ref{sec:theory}. High-dimensional quantities are fixed before
training, while low-dimensional ones are recomputed from the current
embedding. Default values are given in parentheses.}
\label{tab:notation}
\renewcommand{\arraystretch}{1.15}
\adjustbox{max width=\textwidth}{\begin{tabular}{@{}ll@{}}
    \toprule
    \textbf{Symbol} & \textbf{Description} \\
    \midrule
    $\mathcal{N}_i$               & High-dim $k$-NN index list of point $i$ (precomputed, fixed) \\
    $\mathcal{N}_i(j)$                      & Index of the $j$-th high-dim nearest neighbor of point $i$ \\
    $\mathcal{I}_i$, $\mathcal{C}_i$ & Intruder set of point $i$; candidate pool $\mathcal{N}_i \oplus \mathcal{I}_i$ \\
    $m$, $T_r$                    & Intruder count ($10$) and mining interval ($50$ iterations) \\
    $\sigma(\cdot)$               & Sigmoid (logistic) function \\
    $\tau$, $\tau_{\text{start}}$, $\tau_{\text{end}}$ & Temperature parameter and its annealing endpoints \\
    $d_{i,j}$                     & Embedding-space distance from $y_i$ to $y_{C_i(j)}$ \\
    $r_{i,j}^{\text{in}}$       & Input-space rank of candidate $j$ ($= j$ by construction) \\
    $r_{i,j}^{\text{emb}}$        & Soft embedding rank of candidate $j$ with respect to point $i$, Eq.~(\ref{eq:r_low}) \\
    $\rho_i$, $\rho_s$            & Per-point and mean Spearman correlation \\
    $G_i$, $G^{\text{hard}}_i$, $G^{\star}$ & Per-point discounted gain read with the soft ranks, with the integer ranks, and its ideal value, Eqs.~(\ref{eq:gain}) and (\ref{eq:l_hard}) \\
    $\mathcal{L}_{\text{rank}}$   & Rank loss, Eq.~(\ref{eq:rank_gain}); Eq.~(\ref{eq:rank_loss}) when $m{=}0$ \\
    $\mathcal{L}_{\text{total}}$  & Total loss, $\mathrm{KL}(P\|Q) + \lambda \mathcal{L}_{\text{rank}}$ \\
    $\lambda$                     & Regularization weight \\
    $Z$                           & Partition function $\sum_{i\neq j}(1+\|y_i-y_j\|^2)^{-1}$ \\
    $k$, $k_P$, $c$              & Neighbor count ($10$), sparsity ($90$), chunk size ($10{,}000$) \\
    $\eta$, $\mu$, $T$           & Learning rate ($10$), momentum ($0.8$), iterations ($1000$) \\
    \midrule
    $k'$                          & Neighborhood size at which the metrics are read ($10$) \\
    $K$, $\mathrm{disc}(r)$, $\Delta_K$ & $k{+}m$; discount $1/\log_2(1{+}r)$; its smallest step \\
    $\pi_i^{\text{in}}$, $\pi_i^{\text{emb}}$ & Index of the $r$-th nearest neighbor of $i$ in the input and embedding space \\
    $\pi_i^{-1}(j)$ & Discrete embedding rank of candidate $j$ within $\mathcal{C}_i$ \\
    $\mathcal{L}^{\text{hard}}_{\text{rank}}$ & Rank loss evaluated with the discrete ranks, Eq.~(\ref{eq:l_hard}) \\
    $A_i(k')$ & Fraction of the retained neighbors of $i$ that sit at their own rank \\
    $\gamma_i$ & Smallest gap between two candidate squared embedding distances of $i$ \\
    $d^H_{i,j}$ & Input-space distance from $x_i$ to its $j$-th input-space neighbor \\
    $D$, $D'$ & Matrix of squared embedding distances and its single row swap \\
    $\text{DLP}_i$, $\text{O-DLP}_i$ & Per-point membership and exact-rank agreement at size $k'$ \\
    $R_i$, $A_i$                  & Neighbors kept among the $k'$ nearest; their fraction at own rank \\
    $\omega$                      & $k'\,\text{O-DLP}/\text{DLP}$: $1$ for a shuffle, $k'$ for perfect order \\
    $\deg^{H}(j)$                 & In-degree of $j$ in the high-dimensional $k$-NN graph \\
    \bottomrule
\end{tabular}}
\end{table}

\subsection{Hyperparameters}
\label{app:hyperparams}

Table~\ref{tab:hyperparams} collects every setting of the two-stage pipeline, so that a run can be reproduced from this
page alone. Three of them belong to the rank term and the rest are the scikit-learn Barnes--Hut defaults of the warm
start.

\textit{How the three were chosen, and when.} The neighbor count $k$, the intruder count $m$ and the weight $\lambda$
were fixed once and then carried unchanged through every benchmark, application and appendix study of this paper. We
read O-DLP on MNIST over the grid $k, m \in \{10, 30\}$ and $\lambda \in \{5, 50, 100, 200\}$, which selected
$k = m = 10$ and the large-set weight $\lambda = 50$; the small-set weight $\lambda = 5$ was selected on the
$200$-stock pool from which the $35$ trading names of Section~\ref{sec:trading} are drawn. No dataset outside those two
entered the choice, and no per-dataset tuning was performed afterwards. Setting $k = m = 10$ also makes the list the
loss ranks the same list that O-DLP, nDCG and the trading scorer read, so the objective and the criteria see one
neighborhood.

\textit{What the settings mean.} The warm start is a standard Barnes--Hut t-SNE: PCA initialization, perplexity $30$,
learning rate $\max(n/12, 50)$, early exaggeration $12$ and $1{,}000$ iterations. The refinement runs for another
$T = 1{,}000$ iterations at $\eta = 10$ with Polyak momentum $\mu = 0.8$ and no early exaggeration. The value of $\eta$
puts the KL gradient and the rank gradient on the same scale. A pure-KL run under this setting therefore stays close to
its warm start, which is why the $\lambda{=}0$ control of Table~\ref{tab:odlp_main} reproduces t-SNE. The temperature falls linearly from $\tau_{\text{start}} = 5.0$ to $\tau_{\text{end}} = 0.05$, a range in which
Theorem~\ref{thm:soft} puts the soft rank within $\mathcal{O}(e^{-\gamma_i/0.05})$ of the discrete rank at the end of
training. Intruders are re-mined every $T_r = 50$ iterations. For $n \ge 50{,}000$ the affinity matrix keeps
$k_P = 90$ entries per row and the partition function is accumulated in row chunks of $c = 10{,}000$.

\textit{Sensitivity and reproducibility.} The behavior of the three settings away from their defaults is reported in
Section~\ref{sec:sensitivity}: the neighborhood sweep of Figure~\ref{fig:sweep} ($k \in \{5,10,15,20\}$), the
$\lambda{=}0$ control of Table~\ref{tab:odlp_main}, the $m{=}0$ case of \ref{app:m0} and the five perturbed warm
starts of \ref{app:seeds}. All runs use seed $42$ and one NVIDIA A6000; the peer methods run at their published
settings, described in Section~\ref{sec:setup}. The implementation that produces these settings is released
with the paper.

\begin{table}[!ht]
\centering
\setlength{\tabcolsep}{3.6pt}
\caption{Full hyperparameter settings for the SR-tSNE two-stage pipeline. The warm start
and the refinement stage are held fixed across every dataset. Only $\lambda$ follows
dataset size, at $50$ above $50\,000$ points and $5$ below.}
\label{tab:hyperparams}
\adjustbox{max width=\textwidth}{\begin{tabular}{>{\raggedright\arraybackslash}p{3.0cm}>{\raggedright\arraybackslash}p{4.4cm}>{\raggedright\arraybackslash}p{5.0cm}}
\toprule
\textbf{Stage} & \textbf{Parameter} & \textbf{Value} \\
\midrule
\multirow{5}{3.0cm}{Warm start (scikit-learn Barnes--Hut)} & Initialization & PCA \\
 & Perplexity & 30 \\
 & Learning rate & $\max(n/12,\,50)$ \\
 & Early exaggeration & 12 \\
 & Iterations & 1\,000 \\
\midrule
\multirow{7}{3.0cm}{Refinement (PyTorch, exact KL)} & Learning rate $\eta$ & 10 \\
 & Momentum $\mu$ & 0.8 \\
 & Iterations $T$ & 1\,000 \\
 & Early exaggeration & disabled \\
 & $\tau_{\text{start}} \to \tau_{\text{end}}$ & $5.0 \to 0.05$ (linear) \\
 & Sparsity $k_P$ & 90 \\
 & Chunk size $c$ & 10\,000 \\
\midrule
\multirow{5}{3.0cm}{Rank loss} & Loss (all datasets) & smooth rank, $m=10$, $T_r=50$ \\
 & $\lambda$, large datasets ($n \ge 50{,}000$) & 50 \\
 & $\lambda$, other datasets & 5 \\
 & Neighbor count $k$ & 10 \\
\bottomrule
\end{tabular}}
\end{table}

\subsection{The $m{=}0$ Special Case}
\label{app:m0}
With $m{=}0$ no intruders are used. The candidate pool is then the neighbor
list itself, $\mathcal{C}_i = \mathcal{N}_i$. Nothing outside $\mathcal{N}_i$
enters the objective, so the loss no longer defends membership. It scores only
the order inside the neighborhood. In this case the position discount of
Eqs.~(\ref{eq:gain})--(\ref{eq:rank_gain}) can be replaced by the Spearman rank
correlation between the two rank sequences. For point $i$ this is the Pearson
correlation of the centered rank vectors. The usual closed form for Spearman's $\rho$ assumes distinct integer ranks. Soft ranks do not
satisfy that assumption, so we use the general form here. Since
$r_{i,j}^{\text{in}} = j$, the high-dimensional mean is the constant
$\bar{r}^{\text{in}} = (k+1)/2$. The low-dimensional mean is
$\bar{r}_i^{\text{emb}} = \frac{1}{k}\sum_{j=1}^k r_{i,j}^{\text{emb}}$. The
per-point Spearman correlation is:
\begin{equation}\label{eq:spearman_per_point}
    \rho_i \;=\;
    \frac{%
        \displaystyle\sum_{j=1}^{k}
        \Bigl(j - \tfrac{k+1}{2}\Bigr)
        \Bigl(r_{i,j}^{\text{emb}} - \bar{r}_i^{\text{emb}}\Bigr)
    }{%
        \sqrt{\displaystyle\sum_{j=1}^{k}
              \Bigl(j - \tfrac{k+1}{2}\Bigr)^{\!2}}
        \;\cdot\;
        \sqrt{\displaystyle\sum_{j=1}^{k}
              \Bigl(r_{i,j}^{\text{emb}} - \bar{r}_i^{\text{emb}}\Bigr)^{\!2}}
        \;+\; \varepsilon
    },
\end{equation}
where $\varepsilon = 10^{-8}$ prevents division by zero. In the denominator,
the factor built from the high-dimensional ranks,
$\sqrt{\sum_j (j-(k+1)/2)^2}$, is a positive constant set by $k$ alone. The
gradient therefore flows entirely through $r_{i,j}^{\text{emb}}$.

The mean Spearman correlation over all $n$ points is:
\begin{equation}\label{eq:rho_s_def}
    \rho_s \;=\; \frac{1}{n}\sum_{i=1}^{n} \rho_i,
\end{equation}
and the rank loss is then:
\begin{equation}\label{eq:rank_loss}
    \mathcal{L}_{\text{rank}}(Y) \;=\; 1 - \rho_s
    \;=\; \frac{1}{n}\sum_{i=1}^{n}\bigl(1 - \min\!\bigl(\max(\rho_i,\,-1),\,1\bigr)\bigr),
\end{equation}
This lies in $[0, 2]$. It is zero only when every point's high-dimensional
neighbor order is reproduced exactly in the embedding. The
$\min(\max(\cdot,\,-1),\,1)$ projection guards against floating-point
round-off in $\rho_i$.

\clearpage
\subsection{Multi-seed Stability}
\label{app:seeds}
Paired seed statistics. Table~\ref{tab:seeds} reports the paired
advantage of SR-tSNE over the $\lambda{=}0$ control across five warm starts
perturbed with $1\%$ Gaussian jitter.
Every mean is positive, every one of the $30$ metric differences over the ten
paired runs (two datasets, five perturbed starts) is positive, and standard deviations are at most $0.0005$. The jitter moves each
point by about the spacing of its nearest neighbors, so the perturbed start
carries much less local order. On MNIST its O-DLP after the $\lambda{=}0$
refinement is $.056$, against $.078$ for the unperturbed t-SNE. The rank term
brings the same start to $.092$, the level of the unperturbed SR-tSNE run of
Table~\ref{tab:control_main}. The paired advantage is therefore larger than
the gap in that table, and it measures the term's ability to recover order as
well as to add it.

\begin{table}[h]
\centering
\caption{Paired advantage of SR-tSNE over the equal-budget control, mean $\pm$ std over
five warm starts perturbed with $1\%$ Gaussian jitter. Both methods share each
seed's start, so every comparison is paired. All $30$ metric differences over
the ten paired runs are positive.}
\label{tab:seeds}
\adjustbox{max width=\textwidth}{\begin{tabular}{lccc}
\toprule
Dataset & $\Delta$nDCG$_{10}$ & $\Delta$O-DLP$_{10}$ & $\Delta$DLP$_{10}$ \\
\midrule
MNIST & $+.1507 \pm .0005$ & $+.0361 \pm .0004$ & $+.0594 \pm .0002$ \\
CIFAR-10 & $+.0795 \pm .0003$ & $+.0161 \pm .0003$ & $+.0221 \pm .0002$ \\
\bottomrule
\end{tabular}}
\end{table}

\subsection{Warm-start length}
\label{app:warmstart}
The gain does not depend on stopping the warm start at $1{,}000$ iterations. Table~\ref{tab:tsne2000} continues it to
$2{,}000$ and refines from there. SR-tSNE leads that longer run on all eight benchmarks, by a median of $+10\%$, and stands $+6\%$ above its own
$1{,}000$-iteration-start result.

\begin{table}[!ht]
\centering
\caption{Two warm starts. Left: the paper's t-SNE run of $1\,000$ iterations and
SR-tSNE refined from it. Right: a t-SNE run continued to $2\,000$ iterations, the
$\lambda{=}0$ control and SR-tSNE refined from that longer run, and the gain of
SR-tSNE over the longer run. O-DLP at $k{=}10$.}
\label{tab:tsne2000}
\adjustbox{max width=\textwidth}{\begin{tabular}{l cc ccc c}
\toprule
& \multicolumn{2}{c}{Start: $1\,000$ iterations} & \multicolumn{3}{c}{Start: $2\,000$ iterations} & \\
\cmidrule(lr){2-3}\cmidrule(lr){4-6}
Dataset & t-SNE & SR-tSNE & t-SNE & $\lambda{=}0$ & SR-tSNE & Gain (\%) \\
\midrule
MNIST            & .0781 & .0927 & .0890 & .0901 & \textbf{.0979} & +10 \\
CIFAR-10         & .0545 & .0597 & .0624 & .0626 & \textbf{.0641} & +3 \\
Fashion-MNIST    & .0672 & .0792 & .0768 & .0768 & \textbf{.0843} & +10 \\
Speech Commands  & .0809 & .0858 & .0893 & .0900 & \textbf{.0900} & +1 \\
Renal-Cortex     & .0578 & .0741 & .0669 & .0668 & \textbf{.0790} & +18 \\
Brain-Cortex     & .0615 & .0767 & .0697 & .0696 & \textbf{.0808} & +16 \\
Covertype        & .1496 & .1913 & .1679 & .1679 & \textbf{.1998} & +19 \\
AG News          & .0810 & .0905 & .0916 & .0915 & \textbf{.0970} & +6 \\
\bottomrule
\end{tabular}}
\end{table}

\subsection{The form of the rank term}
\label{app:pairwise}
The listwise form is what carries the gain. Table~\ref{tab:pairwise} replaces it with a pairwise logistic surrogate on
the same candidate pool, the same warm start and the same budget. SR-tSNE stands above that surrogate on all eight
benchmarks.

\begin{table}[!ht]
\centering
\caption{The form of the rank term. From the same scikit-learn warm start, with
the same $10$ neighbors and $10$ mined intruders, optimizer and budget, the
refinement runs with a pairwise logistic ordinal loss at the weight selected on
MNIST ($\lambda = 0.05$; larger weights lower every criterion) and with the
listwise discounted gain of SR-tSNE. All criteria at $k{=}10$.}
\label{tab:pairwise}
\setlength{\tabcolsep}{3pt}
\adjustbox{max width=\textwidth}{\begin{tabular}{l ccc ccc ccc}
\toprule
& \multicolumn{3}{c}{O-DLP} & \multicolumn{3}{c}{DLP} & \multicolumn{3}{c}{nDCG} \\
\cmidrule(lr){2-4}\cmidrule(lr){5-7}\cmidrule(lr){8-10}
Dataset & t-SNE & Pairwise & SR-tSNE & t-SNE & Pairwise & SR-tSNE & t-SNE & Pairwise & SR-tSNE \\
\midrule
MNIST            & .0781 & .0764 & \textbf{.0927} & .3880 & .3849 & \textbf{.4009} & .5079 & .5013 & \textbf{.5689} \\
CIFAR-10         & .0545 & .0528 & \textbf{.0597} & .2087 & .2070 & \textbf{.2149} & .3333 & .3275 & \textbf{.3682} \\
Fashion-MNIST    & .0672 & .0646 & \textbf{.0792} & .3368 & .3339 & \textbf{.3493} & .4448 & .4369 & \textbf{.5025} \\
Speech Commands  & .0809 & .0791 & \textbf{.0858} & .3078 & .3065 & \textbf{.3110} & .4652 & .4605 & \textbf{.4993} \\
Renal-Cortex     & .0578 & .0543 & \textbf{.0741} & .3243 & .3198 & \textbf{.3375} & .4147 & .4030 & \textbf{.4890} \\
Brain-Cortex     & .0615 & .0582 & \textbf{.0767} & .3434 & .3397 & \textbf{.3545} & .4380 & .4269 & \textbf{.5121} \\
Covertype        & .1496 & .1470 & \textbf{.1913} & .6369 & .6372 & \textbf{.6452} & .7471 & .7444 & \textbf{.8013} \\
AG News          & .0810 & .0782 & \textbf{.0905} & .3529 & .3516 & \textbf{.3568} & .4821 & .4761 & \textbf{.5260} \\
\bottomrule
\end{tabular}}
\end{table}

\subsection{Per-dataset label-transfer results}
\label{sec:crossdatatab}

Table~\ref{tab:crossdata} gives the per-dataset scores behind the summary of Table~\ref{tab:crossdata_summary}. A
random fifth of each dataset is removed from the retrieval index and used as queries; the embedding is fitted on all
points, so the protocol is transductive and only the index is reduced. Each query is scored from its ten nearest
reference points in the two-dimensional embedding: $\text{P}_1$ asks whether position one already carries the query's
class, and MRR takes the reciprocal of the first position that does. Both criteria therefore read a neighbor's
position, not merely its presence. The last column, $\Delta$, is SR-tSNE minus the equal-budget $\lambda{=}0$ control,
so the control's own score on each row is the SR-tSNE entry minus $\Delta$. The two rows closing each block give the
sign count and the two-sided Wilcoxon signed-rank test over the seven datasets.

\begin{center}
\footnotesize
\setlength{\tabcolsep}{1.9pt}
\begin{longtable}{lrrcccccc r}
\caption{Label transfer across seven vision datasets, one paired observation each, with
$\Delta$ taken against the equal-budget control ($\lambda{=}0$), whose own score on each
row is SR-tSNE minus $\Delta$. Every query is scored
from its ten nearest reference points in the two-dimensional embedding, on the same
held-out fifth of each dataset. Bold marks the best score in a row. The two rows below each block
give the sign count and the two-sided Wilcoxon signed-rank test over the seven datasets,
where bold marks $p < 0.05$.}
\label{tab:crossdata} \\
\toprule
Dataset & $n$ & $\lambda$ & UMAP & TriMap & t-SNE & LocalMAP & DREAMS & SR-tSNE & $\Delta$ \\
\midrule
\endfirsthead
\multicolumn{10}{l}{\footnotesize\textit{Table \ref{tab:crossdata} continued}} \\
\toprule
Dataset & $n$ & $\lambda$ & UMAP & TriMap & t-SNE & LocalMAP & DREAMS & SR-tSNE & $\Delta$ \\
\midrule
\endhead
\midrule
\multicolumn{10}{r}{\footnotesize\textit{continued on next page}} \\
\bottomrule
\endfoot
\bottomrule
\endlastfoot
\multicolumn{10}{l}{\textit{P}$_1$: is the first neighbor returned of the query's class} \\
MNIST            & 70\,000 & 50 & .8811 & .8351 & .9288 & .9164 & .9019 & \textbf{.9354} & $+0.0063$ \\
Fashion-MNIST    & 70\,000 & 50 & .7159 & .6877 & .8296 & .7471 & .8173 & \textbf{.8382} & $+0.0086$ \\
CIFAR-10         & 60\,000 & 50 & .7218 & .6722 & .8035 & .7267 & .7707 & \textbf{.8053} & $+0.0016$ \\
Tiny-ImageNet    & 70\,000 & 50 & .2479 & .1689 & .4561 & .2706 & .4044 & \textbf{.4584} & $+0.0023$ \\
EMNIST-Balanced  & 69\,983 & 50 & .4608 & .3174 & .6728 & .5692 & .6090 & \textbf{.6888} & $+0.0164$ \\
CIFAR-100        & 60\,000 & 50 & .3280 & .2438 & .5466 & .3356 & .4842 & \textbf{.5522} & $+0.0058$ \\
USPS             & 9\,298 & 5 & .9301 & .9091 & .9602 & .9381 & .9478 & \textbf{.9634} & $+0.0032$ \\
\cmidrule(l){4-10}
\multicolumn{3}{r}{SR-tSNE wins on} & 7/7 & 7/7 & 7/7 & 7/7 & 7/7 & & \\
\multicolumn{3}{r}{two-sided $p$}    & \textbf{.0156} & \textbf{.0156} & \textbf{.0156} & \textbf{.0156} & \textbf{.0156} & & \\
\midrule
\multicolumn{10}{l}{MRR: reciprocal rank of the first correct neighbor} \\
MNIST            & 70\,000 & 50 & .9151 & .8835 & .9459 & .9394 & .9299 & \textbf{.9507} & $+0.0046$ \\
Fashion-MNIST    & 70\,000 & 50 & .8038 & .7812 & .8802 & .8262 & .8721 & \textbf{.8858} & $+0.0057$ \\
CIFAR-10         & 60\,000 & 50 & .7949 & .7611 & .8537 & .7985 & .8311 & \textbf{.8547} & $+0.0009$ \\
Tiny-ImageNet    & 70\,000 & 50 & .3401 & .2632 & .5291 & .3651 & .4868 & \textbf{.5313} & $+0.0021$ \\
EMNIST-Balanced  & 69\,983 & 50 & .5891 & .4581 & .7522 & .6796 & .7079 & \textbf{.7619} & $+0.0099$ \\
CIFAR-100        & 60\,000 & 50 & .4321 & .3521 & .6215 & .4420 & .5749 & \textbf{.6247} & $+0.0034$ \\
USPS             & 9\,298 & 5 & .9503 & .9325 & .9677 & .9524 & .9594 & \textbf{.9693} & $+0.0017$ \\
\cmidrule(l){4-10}
\multicolumn{3}{r}{SR-tSNE wins on} & 7/7 & 7/7 & 7/7 & 7/7 & 7/7 & & \\
\multicolumn{3}{r}{two-sided $p$}    & \textbf{.0156} & \textbf{.0156} & \textbf{.0156} & \textbf{.0156} & \textbf{.0156} & & \\
\end{longtable}
\end{center}

\subsection{Runtime}
\label{app:runtime}

Table~\ref{tab:runtime} times embedding construction alone, on one NVIDIA A6000 with exclusive access; neighbor search
and the evaluation of the criteria are excluded, and every method is timed on the same host. The last two rows differ
by the rank term only, so their difference is the cost of the term itself: $1.8\%$ of the refinement stage on MNIST and $2.1\%$ on CIFAR-10,
which is where the cost of the method sits. The refinement in this table evaluates the exact partition function, and
Section~\ref{sec:discussion} gives the $O(n \log n)$ route that keeps the rank term unchanged.

\begin{table}[!ht]
\centering
\caption{Wall-clock runtime in minutes on one NVIDIA A6000, timing only embedding
construction. The two refinement methods differ only in the rank term.}
\label{tab:runtime}
\adjustbox{max width=\textwidth}{\begin{tabular}{lcc}
\toprule
Stage / method & MNIST & CIFAR-10 \\
\midrule
UMAP & 0.40 & 0.06 \\
PaCMAP & 0.53 & 0.21 \\
TriMap & 0.63 & 0.41 \\
t-SNE (warm start) & 3.3 & 3.1 \\
\midrule
$+$ refinement, $\lambda{=}0$ & 25.69 & 18.93 \\
$+$ refinement, SR-tSNE & 26.16 & 19.33 \\
\bottomrule
\end{tabular}}
\end{table}

\FloatBarrier
\subsection{Proofs}
\label{app:proofs}
This section proves the four theorems of Section~\ref{sec:theory}, together with the two lemmas, the two propositions
and the corollary they rest on. Results labelled with an S, namely Lemmas~S1 and~S2, Propositions~S1 and~S2 and Corollary~S1, are stated and proved here. Every other number refers to the main text.

Fix a point $i$. Its candidate pool $\mathcal{C}_i = \mathcal{N}_i \oplus \mathcal{I}_i$ has length
$K = k+m$. Candidate $j \le k$ is the $j$-th true neighbor and carries
relevance $\mathrm{rel}_j = k-j+1$; candidates $j > k$ are intruders and carry
$\mathrm{rel}_j = 0$. Write $R_i$ for the retained neighbors of $i$, the true neighbors that the embedding returns
among its $k'$ nearest points, and $\mathrm{pos}_i(j)$ for the position at which the embedding returns the $j$-th true
neighbor. The soft and discrete gains of point $i$ are $G^{\text{soft}}_i$ and $G^{\text{hard}}_i$, the first read with
the soft ranks of Eq.~(\ref{eq:r_low}) and the second with the integer ranks. The embedding orders $\mathcal{C}_i$ by distance to $y_i$; call
that ordering $\pi_i$, which orders the candidate pool and is distinct from the neighbor maps
$\pi_i^{\text{in}}$ and $\pi_i^{\text{emb}}$ of Eq.~(\ref{eq:odlp}), so that $\pi_i(r)$ is the candidate at position $r$
and $\pi_i^{-1}(j)$ the position of candidate $j$. Ties among intruders are
immaterial because their relevance is equal; ties elsewhere are broken
arbitrarily. The hard per-point gain and loss are
\begin{equation}
G^{\text{hard}}_i(\pi_i) = \sum_{r=1}^{K} \mathrm{rel}_{\pi_i(r)}\,\mathrm{disc}(r),
\qquad
\ell_i = 1 - \frac{G^{\text{hard}}_i(\pi_i)}{G^{\star}},
\qquad
\mathcal{L}^{\text{hard}}_{\text{rank}} = \frac{1}{n}\sum_i \ell_i,
\end{equation}
with $\mathrm{disc}(r) = 1/\log_2(1+r)$ and $G^{\star}$ as in
Eq.~(\ref{eq:ideal_gain}). This $G^{\text{hard}}_i$ is the discrete gain of Eq.~(\ref{eq:l_hard}) re-indexed by
position; intruders carry relevance zero, so the two sums agree. Two facts about the discount are used repeatedly:
it is strictly decreasing, and its steps
$\mathrm{disc}(r) - \mathrm{disc}(r+1)$ are strictly decreasing in $r$,
because $1/\log_2(1+r)$ is convex. Hence
$\mathrm{disc}(r) - \mathrm{disc}(r') \ge \Delta_K$ for all $1 \le r < r' \le K$.

\paragraph{A sorting lemma}
Call a pair of candidates $(a,b)$ a \emph{strict inversion} of $\pi_i$ if
$\mathrm{rel}_a > \mathrm{rel}_b$ and $\pi_i^{-1}(a) > \pi_i^{-1}(b)$. Let
$\mathrm{inv}(\pi_i)$ be their number. The ideal ordering, neighbors
$1,\dots,k$ in order followed by the intruders, has none.

\noindent\emph{This lemma says that every wrongly ordered pair costs the gain at least one
smallest step of the discount.}
\begin{lemma}
\label{lem:sort}
$G^{\star} - G^{\text{hard}}_i(\pi_i) \;\ge\; \mathrm{inv}(\pi_i)\,\Delta_K$.
\end{lemma}

\begin{pf}
If $\mathrm{inv}(\pi_i) > 0$ there is an adjacent position pair $(r, r+1)$
with $\mathrm{rel}_{\pi_i(r)} < \mathrm{rel}_{\pi_i(r+1)}$; otherwise
relevance would be non-increasing along the list and no strict inversion
could exist. Swapping the two candidates raises the gain by
\begin{equation}
\bigl(\mathrm{rel}_{\pi_i(r+1)} - \mathrm{rel}_{\pi_i(r)}\bigr)
\bigl(\mathrm{disc}(r) - \mathrm{disc}(r+1)\bigr) \;\ge\; 1 \cdot \Delta_K,
\end{equation}
since relevances are integers, and it lowers $\mathrm{inv}$ by exactly one:
the swapped pair stops being an inversion and no other pair changes its
relative order. Repeating until no strict inversion remains takes exactly
$\mathrm{inv}(\pi_i)$ swaps and ends at an ordering whose gain is $G^{\star}$,
because relevance is then non-increasing along the list and, intruders being
interchangeable, that ordering is the ideal one. Summing the gains gives the
claim. \qed
\end{pf}

\paragraph{Proof of Theorem~\ref{thm:control}}
Let $k' \le \min(k,m)$ and let $T_i$ be the $k'$ points nearest to $y_i$
among all $n-1$ others. At most $k'$ of them are non-neighbors, and any
non-neighbor in $T_i$ is one of the $k'$ nearest non-neighbors, hence one of
the $m \ge k'$ nearest, hence in $\mathcal{I}_i$. So $T_i \subset \mathcal{C}_i$, and the position
of each element of $T_i$ in $\pi_i$ equals its position among all $n-1$
points. O-DLP at $k'$ therefore reads the first $k'$ positions of $\pi_i$:
\begin{equation}
1 - \text{O-DLP}_i(k') \;=\; \frac{1}{k'}\,\#\{\, r \le k' : \pi_i(r) \ne r \,\}
\;=:\; \frac{\mathrm{miss}_i}{k'} .
\end{equation}
Call a candidate \emph{displaced} if it does not sit where the ideal ordering
puts it, that is, a neighbor $j$ with $\pi_i^{-1}(j) \ne j$ or an intruder at
a position $\le k$. Every miss is a position holding a displaced candidate, so
$\mathrm{miss}_i \le \#\mathrm{displaced}$. Each swap in the proof of
Lemma~\ref{lem:sort} moves two candidates and the process ends with none
displaced, so $\#\mathrm{displaced} \le 2\,\mathrm{inv}(\pi_i)$. With
Lemma~\ref{lem:sort},
\begin{equation}
1 - \text{O-DLP}_i(k') \;\le\; \frac{2\,\mathrm{inv}(\pi_i)}{k'}
\;\le\; \frac{2\,(G^{\star} - G^{\text{hard}}_i)}{k'\,\Delta_K}
\;=\; \frac{2\,G^{\star}}{k'\,\Delta_K}\,\ell_i .
\end{equation}
Averaging over $i$ gives Eq.~(\ref{eq:control}).

For the equivalence: if $\mathcal{L}^{\text{hard}}_{\text{rank}} = 0$ then
every $\ell_i = 0$, so by Lemma~\ref{lem:sort} every $\pi_i$ has no strict
inversion, so the $k$ neighbors of $i$ occupy positions $1,\dots,k$ in order
and every intruder comes after them. A non-candidate is farther from $y_i$
than every intruder, hence farther than every neighbor. So the $k$ nearest
points to $y_i$ are its $k$ neighbors in high-dimensional order, and
$\text{O-DLP}_i(k') = 1$ for every $k' \le k$. Conversely, if
$\text{O-DLP}(k) = 1$ then for every $i$ the $k$ nearest points are the
neighbors in order, so every intruder, being a non-neighbor, is farther than
all of them; $\pi_i$ is the ideal ordering and $\ell_i = 0$. \qed

\paragraph{Proof of Theorem~\ref{thm:soft}}
Eq.~(\ref{eq:r_low}) writes the soft rank of candidate $j$ as
$1 + \sum_{l \ne j} \sigma\bigl((d^2_{i,j} - d^2_{i,l})/\tau\bigr)$, and the
hard rank as $1 + \sum_{l \ne j} \mathbf{1}[d^2_{i,l} < d^2_{i,j}]$. Each
summand differs from its indicator by $\sigma(-|d^2_{i,j} - d^2_{i,l}|/\tau)
\le e^{-\gamma_i/\tau}$, and there are $K-1$ summands. That is the first
claim. For the second, $\mathrm{disc}$ has derivative
$-\ln 2 / \bigl((1+r)\ln^2(1+r)\bigr)$, whose absolute value on $[1,K]$ is
largest at $r=1$, where it equals $1/(2\ln 2)$. Hence
\begin{equation}
|G^{\text{soft}}_i - G^{\text{hard}}_i|
\;\le\; \sum_{j=1}^{k} \mathrm{rel}_j\,\frac{1}{2\ln 2}\,(K-1)\,e^{-\gamma_i/\tau}
\;=\; \frac{k(k+1)}{2}\cdot\frac{K-1}{2\ln 2}\,e^{-\gamma_i/\tau},
\end{equation}
and dividing by $G^{\star}$ gives the bound on $|\ell^{\text{soft}}_i - \ell^{\text{hard}}_i|$. \qed

\paragraph{Proof of Theorem~\ref{thm:pref}}
Under the row relaxation each objective is evaluated by its formula on the
array $D$. The row swap changes only the two entries $d^2_{i,j}$ and
$d^2_{i,l}$, so only the terms that read those entries change.

\textit{(a)} Up to a constant in $Y$,
$\mathrm{KL} = \sum_{a \ne b} p_{ab}\ln(1 + d^2_{a,b}) + \ln Z$ with
$Z = \sum_{a \ne b}(1+d^2_{a,b})^{-1}$ (\ref{app:impl}). $Z$ is a symmetric
function of the entries of $D$ and is unchanged by the swap. Of the first sum,
only the terms $(i,j)$ and $(i,l)$ change, and
\begin{align}
\mathrm{KL}(D) - \mathrm{KL}(D')
&= p_{ij}\ln(1+d^2_{i,j}) + p_{il}\ln(1+d^2_{i,l}) \nonumber\\
&\quad - p_{ij}\ln(1+d^2_{i,l}) - p_{il}\ln(1+d^2_{i,j}) \nonumber\\
&= (p_{ij} - p_{il})\bigl[\ln(1+d^2_{i,j}) - \ln(1+d^2_{i,l})\bigr].
\end{align}
The inversion $d^2_{i,l} < d^2_{i,j}$ makes the bracket positive.

\textit{(b)} The UMAP objective is
$\mathrm{CE} = -\sum_{a \ne b}\bigl[v_{ab}\ln w_{ab} + (1-v_{ab})\ln(1-w_{ab})\bigr]$
with $w_{ab} = w(d^2_{a,b})$. Write each term as
$-v_{ab}\,\mathrm{logit}(w_{ab}) - \ln(1-w_{ab})$. The second part depends on
$D$ only through the multiset of the two swapped entries and cancels in the
difference. The first part gives
$(v_{ij} - v_{il})\bigl[\mathrm{logit}(w_{il}) - \mathrm{logit}(w_{ij})\bigr]$,
and for $w(x) = (1+ax^b)^{-1}$ one has $\mathrm{logit}(w(x)) = -\ln a - b\ln x$,
so the bracket equals $b\,[\ln d^2_{i,j} - \ln d^2_{i,l}] > 0$.

\textit{(c)} The row swap changes only $\ell_i$, and in $\pi_i$ it exchanges
the candidates at positions $r < r'$, where $l$ sits at $r$ and $j$ at $r'$.
The gain changes by
$(\mathrm{rel}_j - \mathrm{rel}_l)[\mathrm{disc}(r) - \mathrm{disc}(r')]
= (l-j)[\mathrm{disc}(r) - \mathrm{disc}(r')] \ge \Delta_K$, and
$n[\mathcal{L}^{\text{hard}}_{\text{rank}}(D) - \mathcal{L}^{\text{hard}}_{\text{rank}}(D')]
= \ell_i(D) - \ell_i(D')$ equals that change divided by $G^{\star}$. \qed

Here $p_{j|i} = \exp(-\|x_i - x_j\|_2^2 / 2\sigma_i^2)/Z_i$ is the conditional affinity of $j$ to $i$, with $\sigma_i$
the perplexity-calibrated bandwidth of Section~\ref{sec:method} and $Z_i$ the normalizer that makes $p_{\cdot|i}$ sum
to one.

\noindent\emph{This lemma says that the affinities t-SNE and UMAP read can tie, or run against
the distances, so the request in parts (a) and (b) can vanish or point the wrong way.}
\begin{lemma}[When the affinities do not follow the distances]
\label{lem:affinity}
\textup{(a)} In t-SNE, $p_{ij} < p_{il}$ with $d^H_{i,j} < d^H_{i,l}$ if and only if
\begin{equation}
p_{i|l} - p_{i|j} \;>\; p_{j|i} - p_{l|i} \;>\; 0 ,
\end{equation}
This occurs when $j$ sits in a region dense enough that its bandwidth
$\sigma_j$ makes $p_{i|j}$ small, while $l$ sits in a sparse region with a
large $\sigma_l$. The symmetrization $p_{ij} = (p_{j|i} + p_{i|j})/2n$
does not preserve the order of the conditionals $p_{\cdot|i}$.
\textup{(b)} In UMAP, $v_{j|i} = 1$ for the nearest neighbor $j$ of every
point $i$, and the fuzzy union gives $v_{ij} = 1$ whenever either conditional
equals one. If $j$ is the nearest neighbor of $i$ and $i$ is the nearest
neighbor of $l$, then $v_{ij} = v_{il} = 1$ and the cross-entropy is
indifferent to the order of $j$ and $l$ around $i$.
\end{lemma}

\paragraph{Proof of Lemma~\ref{lem:affinity}}
\textit{(a)} $p_{ij} < p_{il}$ reads $p_{j|i} + p_{i|j} < p_{l|i} + p_{i|l}$,
which rearranges to the displayed inequality. Its right-hand side is positive
because $p_{\cdot|i}$ is a decreasing function of $d^H_{i,\cdot}$ at the
common bandwidth $\sigma_i$. The left-hand side is
$\exp(-d^{H\,2}_{l,i}/2\sigma_l^2)/Z_l - \exp(-d^{H\,2}_{j,i}/2\sigma_j^2)/Z_j$,
which can be made positive by taking $\sigma_j$ small relative to
$d^H_{j,i}$, so that the second term vanishes, while $\sigma_l$ is large
relative to $d^H_{l,i}$ and $Z_l$ is moderate. Perplexity calibration sets
$\sigma_j$ small exactly when $j$ has many close neighbors.
\textit{(b)} UMAP sets $v_{j|i} = \exp\bigl(-(d^H_{i,j} - \varrho_i)/\sigma_i\bigr)$
with the local-connectivity offset $\varrho_i = d^H_{i,j}$ for the nearest neighbor $j$, so $v_{j|i} = 1$. The
fuzzy union $v_{ab} = v_{b|a} + v_{a|b} - v_{b|a}v_{a|b}$ equals one whenever
either conditional does. Under the stated hypothesis $v_{j|i} = 1$ and
$v_{i|l} = 1$, so $v_{ij} = v_{il} = 1$, and Eq.~(\ref{eq:ce_pref}) vanishes. \qed

Remark (where a loss stops asking for order). Suppose an objective
is invariant under every permutation of the distances from $i$ to its $k'$
returned points, and the optimizer's output inherits that symmetry. Then the
returned arrangement is exchangeable, part (ii) applies, and
$\mathbb{E}[\omega_i] = 1$. Theorem~\ref{thm:pref}(a)--(b) say when the
invariance holds for t-SNE and UMAP: exactly when the affinities of the
returned points tie, because the row-swap difference then vanishes for every
transposition and hence for every permutation. The remark describes the limit
in which the two losses stop asking for order; Section~\ref{sec:results}
measures where each method sits relative to that baseline.

\paragraph{Proof of Theorem~\ref{thm:blind}}
\textit{(i)} By definition $\text{O-DLP}_i(k) = \#\{j \in R_i : \text{pos}_i(j) = j\}/k
= (|R_i|/k)\cdot A_i = \text{DLP}_i A_i$, since only retained neighbors can be
hits. Averaging gives the first claim, and
$k\,\text{O-DLP}/\text{DLP} = \sum_i \text{DLP}_i\,(kA_i) / \sum_i \text{DLP}_i$
is the stated weighted mean; $A_i \in [0,1]$ gives the range.

\textit{(ii)} Write $k'$ for the size at which the metric is read, let
$q = |R_i|$, and let $X_j = \mathbf{1}[\text{pos}_i(j) = j]$ for $j \in R_i$.
Under a uniform arrangement of the $k'$ returned points, each returned point
is equally likely to occupy any of the $k'$ positions, so
$\mathbb{E}X_j = 1/k'$ and $\mathbb{E}[X_jX_l] = 1/(k'(k'-1))$ for $j \ne l$.
Hence $\mathbb{E}[\#\text{hits}] = q/k'$, which is $\mathbb{E}[A_i] = 1/k'$, and
\begin{equation}
\mathrm{Var}(\#\text{hits}) = \frac{q}{k'}\Bigl(1-\frac{1}{k'}\Bigr)
 + q(q-1)\Bigl[\frac{1}{k'(k'-1)} - \frac{1}{k'^2}\Bigr]
 = \frac{q}{k'}\Bigl(1-\frac{1}{k'}\Bigr) + \frac{q(q-1)}{k'^2(k'-1)}
 \;\le\; \frac{q}{k'},
\end{equation}
the last step because $q - 1 \le k' - 1$. Since $k'\,\text{O-DLP}_i = \#\text{hits}$
and $q/k' = \text{DLP}_i$, this is $\mathrm{Var}(k'\,\text{O-DLP}_i) \le \text{DLP}_i$.
With independent arrangements across points,
$\mathrm{Var}(\text{O-DLP}) = n^{-2}\sum_i \mathrm{Var}(\text{O-DLP}_i)
\le n^{-2} k'^{-2} \sum_i \text{DLP}_i = \text{DLP}/(n k'^2)$, so the standard
deviation of O-DLP is at most $k'^{-1}\sqrt{\text{DLP}/n}$. Dividing by
$\text{DLP}/k'$ gives Eq.~(\ref{eq:omega_res}).

\textit{(iii)} Since $k' \le \min(k,m)$, every one of the $k'$ points the
embedding returns for $i$ is a candidate in $\mathcal{C}_i$: a returned neighbor lies
among the $k$ neighbors, and a returned non-neighbor is one of the $k' \le m$
nearest non-neighbors. So the first $k'$ entries of $\pi_i$ are exactly the
returned points, in embedded order. A retained neighbor $j \in R_i$ is a miss
when $\text{pos}_i(j) \ne j$. If $\text{pos}_i(j) < j$, the $\text{pos}_i(j) - 1
\le j - 2$ entries ahead of $j$ cannot hold all $j - 1$ neighbors of higher
relevance, so one of them follows $j$ in $\pi_i$: a strict inversion involving
$j$. If $\text{pos}_i(j) > j$, at least $j$ entries precede $j$, and at most
$j - 1$ of them have relevance above $\mathrm{rel}_j$; the remaining one is a
neighbor of lower rank or an intruder of relevance zero, and $(j, \cdot)$ is a
strict inversion involving $j$. Every miss therefore belongs to at least one
strict inversion of $\pi_i$, and a strict inversion contains at most two
retained neighbors, so $\#\text{misses} \le 2\,\mathrm{inv}(\pi_i)$ and
$\text{DLP}_i(k') - \text{O-DLP}_i(k') = \#\text{misses}/k' \le
2\,\mathrm{inv}(\pi_i)/k'$. Lemma~\ref{lem:sort} gives
$\mathrm{inv}(\pi_i) \le G^{\star}\ell_i/\Delta_K$; averaging over $i$ gives
Eq.~(\ref{eq:deficit}). \qed

\paragraph{What the plane can hold}
The following bound makes the crowding remark of Section~\ref{sec:results}
quantitative. Let $\deg^{H}(j)$ be the number of points that count $j$ among
their $k$ high-dimensional neighbors, the in-degree of $j$ in the
high-dimensional $k$-NN digraph.

\noindent\emph{This proposition says that in the plane a point can be among the $k$ nearest
neighbors of at most $7k$ others, so neighbor edges into high-dimensional hubs beyond that
count cannot be kept by any embedding.}
\begin{proposition}[Crowding bound for order]
\label{prop:crowd}
For any embedding of $n$ points in $\mathbb{R}^2$ with distinct pairwise
distances,
\begin{equation}
\text{O-DLP}(k) \;\le\; \text{DLP}(k) \;\le\;
\frac{1}{nk}\sum_{j=1}^{n} \min\bigl(\deg^{H}(j),\; 7k\bigr).
\label{eq:crowd}
\end{equation}
\end{proposition}
\begin{pf}
The first inequality is Eq.~(\ref{eq:odlp_le_dlp}). For the second, write
$\text{DLP}(k) = \frac{1}{nk}\sum_i |\mathcal{N}_i \cap \hat{\mathcal{N}}_i|$ with $\hat{\mathcal{N}}_i$ the
$k$ nearest embedded points, and count pairs by their second element:
$\sum_i |\mathcal{N}_i \cap \hat{\mathcal{N}}_i| = \sum_j \#\{i : j \in \mathcal{N}_i,\; j \in \hat{\mathcal{N}}_i\}
\le \sum_j \min(\deg^{H}(j), \deg^{\text{low}}(j))$, where $\deg^{\text{low}}(j)$
is the in-degree of $j$ in the embedded $k$-NN digraph. Stone's
lemma~\cite{stone1977consistent} bounds that in-degree: partition the plane
around $y_j$ into seven cones of angle $2\pi/7 < \pi/3$; if two points $y_a,
y_b$ lie in the same cone with $\|y_a - y_j\| < \|y_b - y_j\|$, the angle at
$y_j$ is below $\pi/3$ and so $\|y_b - y_a\| < \|y_b - y_j\|$, hence $y_a$ is
closer to $y_b$ than $y_j$ is; only the $k$ points of a cone nearest to $y_j$
can therefore have $y_j$ among their $k$ nearest, and
$\deg^{\text{low}}(j) \le 7k$. \qed
\end{pf}

The bound bites only through the hubs of the high-dimensional graph
\cite{radovanovic2010hubs}: every neighbor edge that points into a node with
$\deg^{H} > 7k$ beyond the first $7k$ is lost in any planar embedding, whatever
the loss. Table~\ref{tab:crowd} evaluates the bound on the eight benchmarks at
$k = 10$. The hub tail is thin everywhere: the bound is $1.000$ on MNIST and
Covertype, where no point exceeds in-degree $70$, and at least $0.968$ on the
rest, with CIFAR-10 the most hub-heavy at $3.2\%$ of edges. The measured DLP
of the best method lies between $0.18$ and $0.65$ (Table~\ref{tab:dlp_main}),
so the in-degree ceiling of the plane is far from what limits membership, and
the order that survives is set by the loss and the layout it reaches.

\begin{table}[!ht]
\centering
\caption{The crowding bound of Proposition~\ref{prop:crowd} at $k = 10$: the
largest high-dimensional in-degree, the number of points with in-degree above
$7k = 70$, the share of neighbor edges that no planar embedding can keep, and
the right-hand side of Eq.~(\ref{eq:crowd}).}
\label{tab:crowd}
\adjustbox{max width=\textwidth}{\begin{tabular}{l r r r r r}
\toprule
Dataset & $n$ & Max in-degree & Points above $70$ & Unkeepable edges & Bound \\
\midrule
MNIST & 70\,000 & 53 & 0 & 0.00\% & 1.000 \\
CIFAR-10 & 60\,000 & 277 & 660 & 3.24\% & 0.968 \\
Fashion-MNIST & 70\,000 & 188 & 215 & 0.69\% & 0.993 \\
Speech Commands & 59\,387 & 110 & 51 & 0.10\% & 0.999 \\
Renal-Cortex & 50\,027 & 126 & 88 & 0.26\% & 0.997 \\
Brain-Cortex & 50\,212 & 142 & 55 & 0.16\% & 0.998 \\
Covertype & 59\,999 & 30 & 0 & 0.00\% & 1.000 \\
AG News & 60\,000 & 179 & 236 & 0.79\% & 0.992 \\
\bottomrule
\end{tabular}}
\end{table}

\paragraph{The trading scorer}
The scorer of \ref{app:scorer} computes, for each bar $i$ that DBSCAN leaves
outside every cluster, the Spearman correlation $\rho_i$ between the ranks
$1,\dots,k'$ of $i$'s $k'$ nearest input-space neighbors and the ranks of those
same bars by distance to $y_i$ in the embedding, with $k' = 10$. The order half
of the marker score is $1-\rho_i$.

\noindent\emph{This proposition says that the order half of the marker score is bounded by
the same per-bar rank loss that SR-tSNE minimizes.}
\begin{proposition}[The scorer reads a controlled quantity]
\label{prop:scorer}
Let $k' \le k$ and let $\ell_i$ be the hard per-point rank loss of
Section~\ref{sec:theory}. Then, for every bar $i$,
\begin{equation}
1 - \rho_i \;\le\; \frac{12}{k'(k'+1)}\cdot\frac{G^{\star}}{\Delta_K}\;\ell_i .
\label{eq:prop_scorer}
\end{equation}
\end{proposition}

\begin{pf}
Write $d_j = |j - r_j|$ for the displacement of the $j$-th input-space neighbor,
$r_j$ being its rank among the $k'$ neighbors by embedded distance, and let
$\mathrm{inv}_i(k')$ be the number of pairs among the $k'$ neighbors whose order the
embedding inverts. Spearman's coefficient on two permutations of $\{1,\dots,k'\}$ is
$\rho_i = 1 - 6\sum_j d_j^2 / \bigl(k'(k'^2-1)\bigr)$. The Diaconis--Graham
inequality~\cite{diaconis1977spearman} gives $\sum_j d_j \le 2\,\mathrm{inv}_i(k')$, and each
$d_j \le k'-1$, so $\sum_j d_j^2 \le (k'-1)\sum_j d_j \le 2(k'-1)\,\mathrm{inv}_i(k')$. Hence
$1 - \rho_i \le 12\,\mathrm{inv}_i(k') / \bigl(k'(k'+1)\bigr)$. Every pair of the first $k'$
neighbors that the embedding inverts is a strict inversion of $\pi_i$, because
the two carry distinct relevances $k-j+1$ and the order among candidates is the
order by embedded distance; so $\mathrm{inv}_i(k') \le \mathrm{inv}(\pi_i)$. Lemma~\ref{lem:sort}
gives $\mathrm{inv}(\pi_i) \le (G^{\star} - G^{\text{hard}}_i)/\Delta_K = G^{\star}\ell_i/\Delta_K$.
\qed
\end{pf}

The constant is uniform: it is fixed by $k$, $m$ and $k'$, and it does not grow with $n$. At the setting used
throughout this paper, $k = m = k' = 10$, it is $8.8 \times 10^{2}$. The proposition identifies which objective drives
the quantity down: under SR-tSNE what remains of $1-\rho_i$ after training is
the order the plane could not restore, whereas under t-SNE it also carries the
order the objective never asked for. The min--max normalization of
\ref{app:scorer} then ranks bars by that residual.

\noindent\emph{This corollary says that a finite weight on the hard rank term makes the whole
loss prefer the correct order of every pair.}
\begin{corollary}[A finite weight suffices]
\label{cor:lambda}
Let the embedding have squared diameter at most $R^2$. The total loss
$\mathrm{KL} + \lambda \mathcal{L}^{\text{hard}}_{\text{rank}}$, read on the
distance arrays of Theorem~\ref{thm:pref}, strictly prefers the correct order for every
inverted pair of every point whenever
\begin{equation}
\lambda \;>\; \lambda^{\star}
\;=\; \frac{n\,G^{\star}\,\ln(1+R^2)}{\Delta_K}
      \max_{i,\;j<l}\,(p_{il} - p_{ij})^{+} .
\end{equation}
Since $p_{ab} \le 1/n$ for every pair, $\lambda^{\star} \le G^{\star} \ln(1+R^2)/\Delta_K$,
a bound that does not grow with $n$.
\end{corollary}

\paragraph{Proof of Corollary~\ref{cor:lambda}}
For an inverted pair, Theorem~\ref{thm:pref}(a) bounds the KL preference
from below by $-(p_{il} - p_{ij})^{+}\ln(1+R^2)$, since
$0 < \ln(1+d^2_{i,j}) - \ln(1+d^2_{i,l}) \le \ln(1+R^2)$. Theorem~\ref{thm:pref}(c)
bounds the rank preference from below by $\lambda\Delta_K/(nG^{\star})$. The
sum is positive when $\lambda$ exceeds $\lambda^{\star}$. Finally
$p_{ab} = (p_{b|a} + p_{a|b})/2n \le 1/n$ because each conditional
distribution sums to one. The statement is about the hard rank loss; the soft
loss of Eqs.~(\ref{eq:r_low})--(\ref{eq:ideal_gain}) converges to it as $\tau \to 0$ (Theorem~\ref{thm:soft}), so for any
pair with a fixed positive distance gap the soft preference converges to the
hard one. \qed


\end{document}